\documentclass{article}

\usepackage{microtype}
\usepackage{graphicx}
\usepackage{subfigure}
\usepackage{booktabs} 

\usepackage{hyperref}

\usepackage[accepted]{icml2024}

\usepackage{amsmath}
\usepackage{amssymb}
\usepackage{mathtools}
\usepackage{amsthm}

\usepackage[capitalize,noabbrev]{cleveref}

\usepackage{tikz}
\usetikzlibrary{graphs,graphs.standard,arrows}

\usepackage{enumitem}
\usepackage{bbding}
\usepackage{bm}

\newtheorem{Theorem}{Theorem}
\newtheorem{Proposition}{Proposition}
\newtheorem{Lemma}{Lemma}

\newtheorem{Definition}{Definition}

\newtheorem{Assumption}{Assumption}
\newtheorem{Example-set}{Example}

\newtheorem{Remark}{Remark}

\renewcommand{\algorithmicrequire}{ \textbf{Input:}} 
\renewcommand{\algorithmicensure}{ \textbf{Output:}} 

\usepackage{centernot}

\newcommand{\CI}{\mathrel{\perp\mspace{-10mu}\perp}}
\newcommand{\nCI}{\centernot{\CI}}

\usetikzlibrary{fit,matrix,positioning, 
decorations.pathreplacing,calc,decorations,arrows,shadows,patterns}

\usetikzlibrary{arrows,fit,positioning,shapes,matrix}

\usepackage{tikz,pgfplots}

\newenvironment{nospaceflalign*}
 {\setlength{\abovedisplayskip}{1pt}\setlength{\belowdisplayskip}{1pt}%
  \csname flalign*\endcsname}
 {\csname endflalign*\endcsname\ignorespacesafterend}

 \newenvironment{nospaceflalign}
 {\setlength{\abovedisplayskip}{1pt}\setlength{\belowdisplayskip}{1pt}%
  \csname flalign\endcsname}
 {\csname endflalign\endcsname\ignorespacesafterend}

\newdimen\arrowsize
\pgfarrowsdeclare{arcsq}{arcsq}
{
	\arrowsize=0.2pt
	\advance\arrowsize by .5\pgflinewidth
	\pgfarrowsleftextend{-4\arrowsize-.5\pgflinewidth}
	\pgfarrowsrightextend{.5\pgflinewidth}
}
{
	\arrowsize=0.8pt
	\advance\arrowsize by .5\pgflinewidth
	\pgfsetdash{}{0pt} 
	\pgfsetroundjoin   
	\pgfsetroundcap    
	\pgfpathmoveto{\pgfpoint{0\arrowsize}{0\arrowsize}}
	\pgfpatharc{-90}{-140}{4\arrowsize}
	\pgfusepathqstroke
	\pgfpathmoveto{\pgfpointorigin}
	\pgfpatharc{90}{140}{4\arrowsize}
	\pgfusepathqstroke
}

\usepackage[textsize=tiny]{todonotes}

\icmltitlerunning{Automating the Selection of Proxy Variables of Unmeasured Confounders}

\begin{document}

\twocolumn[
\icmltitle{Automating the Selection of Proxy Variables of Unmeasured Confounders}




\begin{icmlauthorlist}
\icmlauthor{Feng Xie}{aaa}
\icmlauthor{Zhengming Chen}{ccc,ddd}
\icmlauthor{Shanshan Luo}{aaa}
\icmlauthor{Wang Miao}{bbb}
\icmlauthor{Ruichu Cai}{ccc}
\icmlauthor{Zhi Geng}{aaa}
\end{icmlauthorlist}

\icmlaffiliation{aaa}{Department of Applied Statistics, Beijing Technology and Business University, Beijing, China}
\icmlaffiliation{bbb}{Department of Probability and Statistics, Peking University, Beijing, China}
\icmlaffiliation{ccc}{School of Computer Science, Guangdong University of Technology, Guangzhou 510006, China}
\icmlaffiliation{ddd}{Machine Learning Department, Mohamed bin Zayed University of Artificial Intelligence, Abu Dhabi, UAE}


\icmlcorrespondingauthor{Shanshan Luo}{shanshanluo@btbu.edu.cn}

\icmlkeywords{Machine Learning, ICML}

\vskip 0.3in
]



\printAffiliationsAndNotice{}  
%
\begin{abstract}
Recently, interest has grown in the use of proxy variables of unobserved confounding for inferring the causal effect in the presence of unmeasured confounders from observational data. 
One difficulty inhibiting the practical use is finding valid proxy variables of unobserved confounding to a target causal effect of interest. These proxy variables are typically justified by background knowledge.
In this paper, we investigate the estimation of causal effects among multiple treatments and a single outcome, all of which are affected by unmeasured confounders, within a linear causal model, without prior knowledge of the validity of proxy variables.
To be more specific, we first extend the existing proxy variable estimator, originally addressing a single unmeasured confounder, to accommodate scenarios where multiple unmeasured confounders exist between the treatments and the outcome.
Subsequently, we present two different sets of precise identifiability conditions for selecting valid proxy variables of unmeasured confounders, based on the second-order statistics and higher-order statistics of the data, respectively.
Moreover, we propose two data-driven methods for the selection of proxy variables and for the unbiased estimation of causal effects. Theoretical analysis demonstrates the correctness of our proposed algorithms. 
Experimental results on both synthetic and real-world data show the effectiveness of the proposed approach.
\end{abstract}

\section{Introduction}
Estimating the causal effect from observational data is a fundamental problem in various fields of scientific research, including social sciences~\citep{pearl2009causality,spirtes2000causation}, economics~\citep{imbens2015causal}, public health~\citep{hernan2006estimating}, and
machine learning~\citep{spirtes2010introduction,peters2017elements,fernandez2022causal}. Within the framework of causal graphical models, covariate adjustment, such as the use of the back-door criterion, emerges as a powerful and primary tool for estimating causal effects from observational data \citep{pearl2009causality,van2019separators}.
However, although this method has been used in a range of fields, it should be noted that biased causal effects can arise when unmeasured confounders are present and the covariate adjustment set does not exist in the system~\citep{pearl2009causality,rotnitzky2020efficient,cheng2022toward}.

The method of instrumental variables is a general approach used to estimate the causal effect of interest in the presence of unobserved confounders \citep{pearl2009causality,wright1928tariff,goldberger1972structural,bowden1990instrumental}. This method has been extensively studied in practical sciences, including economics~\citep{imbens2015causal,Imbens2014IV}, sociology \citep{pearl2009causality,spirtes2000causation} and epidemiology~\citep{hernan2006instruments,baiocchi2014instrumental}.
In practice, it can be quite challenging to identify a valid instrumental variable \citep{pearl1995testability,kuroki2005instrumental,kang2016instrumental,silva2017learning,gunsilius2021nontestability,xie2022testability,cheng2023discovering}.
Sometimes, in the system of interest, an instrumental variable may not even exist.

Recently, the proximal causal learning method, also referred to as negative control, has emerged as an alternative strategy to address unmeasured confounders and estimate the unbiased causal effects of interest 
\citep{kuroki2014measurement,miao2016identifiability,de2017proxy,miao2018proxy,miao2018confounding,wang2019jasa,shi2020multiply,tchetgen2020introduction,singh2020kernel,wang2021proxy,mastouri2021proximal,xu2021deep,shpitser2023proximal}.
This method allows us to infer the causal effect of interest by observing suitable proxy variables for unmeasured confounding, with these proxy variables often being termed Negative Controls (NCs). NCs are readily applicable in various domains \citep{lipsitch2010negative,sofer2016negative}.
For instance, one study of the causal effect of the flu shot ($X_k$) on influenza-related hospitalization ($Y$), where there exists unmeasured health-seeking behavior ($U$) \citep{shi2020selective}.
The proximal causal learning method operates on the following principles: (i) find a variable e.g., a person's annual wellness visit history ($Z$), that is influenced by confounder $U$ and has no direct effect on the outcome $Y$, referred to as the Negative Control Exposure (NCE); (ii) find another variable, e.g., a person's injury/trauma hospitalization ($W$), that is influenced by confounder $U$ and is not causally affected by the treatment $X_k$, referred to as the Negative Control Outcome (NCO); and (iii) use these two proxy variables to estimate the causal effect of flu shot on influenza-related hospitalization. Figure \ref{Fig-proximal-example} illustrates the corresponding causal graph that satisfies the above conditions respectively, with further details in Section \ref{Sub-sec-confounderproxy}.
However, although these methods have been used in a range of fields, the valid proxy variables are typically justified by background knowledge in those works.
Thus, it is vital to develop statistical methods for selecting proxy variables of unmeasured confounding from observational data.

\begin{figure}[htp]
	\begin{center}
	\begin{tikzpicture}[scale=1.0, line width=0.5pt, inner sep=0.2mm, shorten >=.1pt, shorten <=.1pt]
		\draw (3, 1.0) node(U) [circle, fill=gray!60, minimum size=0.5cm,draw] {{\footnotesize\,$U$\,}};
		\draw (1.5, 1.0) node(Z) [] {{\footnotesize\,$Z$\,}};
		\draw (4.5, 1.0) node(W) [] {{\footnotesize\,$W$\,}};
		\draw (2.25, 0) node(X) [] {{\footnotesize\,{$X_k$}\,}};
		\draw (3.75, 0) node(Y) [] {{\footnotesize\,{$Y$}\,}};
            %
		\draw (2.25, -0.3) node(Xkk) [] {{\footnotesize\,flu shot\,}};
            \draw (3.75, -0.3) node(YY) [] {{\footnotesize\,influenza\,}};
            \draw (3.75, -0.6) node(YYY) [] {{\footnotesize\, hospitalization\,}};
            \draw (0.8, 0.7) node(ZZZ) [] {{\footnotesize\,annual wellness\,}};
            \draw (0.8, 0.4) node(ZZZZ) [] {{\footnotesize\,visit history\,}};
            \draw (5.2, 0.7) node(WWW) [] {{\footnotesize\,injury/trauma\,}};
            \draw (5.2, 0.4) node(WWWW) [] {{\footnotesize\,hospitalization\,}};
		%
		\draw (0.8, 1.0) node(ZZ) [] {{\footnotesize\,(NCE)\,}};
		\draw (5.2, 1.0) node(ZZ) [] {{\footnotesize\,(NCO)\,}};
		\draw[-arcsq] (U) -- (Z) node[pos=0.5,sloped,above] {}; 
		\draw[-arcsq] (U) -- (W)node[pos=0.5,sloped,above] {}; 
		\draw[-arcsq] (U) -- (X) node[pos=0.5,sloped,above] {};
		\draw[-arcsq] (U) -- (Y) node[pos=0.5,sloped,above] {};
		\draw[-arcsq] (X) -- (Y) node[pos=0.5,sloped,above] {};
	\end{tikzpicture}\\
        \vspace{-4mm}
	\caption{A typical confounder proxy causal diagram.  $Z$ and $W$ are NCE and NCO of unmeasured confounder $U$ for the causal relationship $X_k \to Y$.}
	\vspace{-5mm}
	\label{Fig-proximal-example}
	\end{center}
\end{figure}
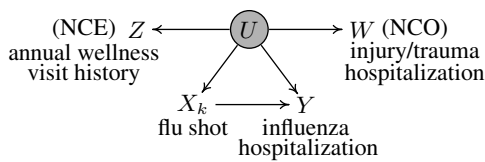

Recently, \citet{kummerfeld2022data} established sufficient conditions for selecting valid NCE and NCO of one unmeasured confounder in a linear causal model,  making valuable contributions to the field. 
However, their work is limited to single-treatment settings, where potential proxy variables cannot directly affect both the treatment and the outcome of interest. In reality, multiple-treatment scenarios exist, where unmeasured confounders influence both treatments and the outcome.
For instance, in gene expression studies, there exist multiple gene expressions may affect the trait of a human of interest (e.g., body weight) \citep{miao2022identifying}.  Besides, their work only considers a particular class of proxy variables of an unmeasured confounder, where those proxy variables are independent of both the treatment and outcome conditional on the unmeasured confounder while our work does not restrict this condition.
In this paper, we tackle the challenge of proxy variable identification in a more complex scenario, where the proxy variables can have effects on the outcome, and multiple unmeasured confounders may exist in the system. 
Specifically, we make the following contributions:
\begin{itemize}[leftmargin=15pt,itemsep=0pt,topsep=0pt,parsep=0pt]
\item [1.] We extend the existing proxy variable estimator that deals with a single unmeasured confounder, as previously discussed by \citet{kuroki2014measurement}, to accommodate scenarios where multiple unmeasured confounders exist between treatments and the outcome.
\item [2.] We present two different sets of precise identifiability conditions for selecting proxy variables of unmeasured confounders, based on the second-order statistics \footnote{Second-order statistics means the second-order moments (like covariances or correlations).} and higher-order statistics \footnote{Higher-order statistics means beyond the second-order moments in statistics, e.g., skewness, kurtosis, etc. of the data.}, respectively.
\item [3.] We propose two efficient algorithms for selecting proxy variables of unmeasured confounders. The first algorithm leverages the rank-deficiency properties of covariance matrices, while the second algorithm takes advantage of the non-Gaussianity of the data. Both algorithms consistently estimate the desired causal effect and come with theoretical proofs that establish their correctness.
\item [4.] We demonstrate the efficacy of the proposed algorithms on both synthetic and real-world data. 
\end{itemize}

\vspace{-2mm}
\section{Preliminaries}
\vspace{-1mm}
\subsection{Notations}
Our work is in the framework of causal graphical models \citep{pearl2009causality,spirtes2000causation}. In a directed acyclic graph (DAG) $\mathcal{G}$, a \textbf{path} is a sequence of nodes $\{X_1, \dots, X_r\}$ such that $X_i$ and $X_{i+1}$ are adjacent in $\mathcal{G}$, where $1 \le i<r$. A \textbf{collider} on a path $\{X_1,...X_p\}$ is a node $X_i$ , $1< i < p$, such that $X_{i-1}$ and $X_{i+1}$ are parents of $X_{i}$. A \textbf{trek} between $X_i$ and $X_j$ is a path that does not contain any colliders in $\mathcal{G}$. A \textbf{source} in a trek is a unique node such that no arrows point to it. 
We use the ordered pair of directed paths $(P_1,P_2)$ denotes a trek in $\mathcal{G}$ from $X_i$ to $X_j$, where $P_1$ has sink $X_i$, $P_2$ has sink $X_j$, and both $P_1$ and $P_2$ have the same source.\footnote{A sink of a graph $\mathcal{G}$ is any node that is not a parent of any other node.} Other commonly used concepts in graphical models, such as d-separation, can be found in standard sources \citep{pearl1988probabilistic,pearl2009causality,spirtes2000causation}.

We denote vectors and matrices by boldface letters. The $(i,j)$ entry of matrix $\mathbf{M}$ is denoted by $\mathbf{M}_{i,j}$. The notation $|\mathbf{A}|$ denotes the cardinality of set $\mathbf{A}$. The 
notation $\boldsymbol{\Sigma}_{\mathbf{A},\mathbf{B}}$ denotes the cross-covariance matrix of set $\mathbf{A}$ (rows) and $\mathbf{B}$ (columns). The notation $\mathrm{rk}(\mathbf{C})$ denotes the rank of matrix $\mathbf{C}$, e.g., $\mathrm{rk}(\boldsymbol{\Sigma}_{\mathbf{A},\mathbf{B}})$ denotes the rank of cross-covariance matrix of set $\mathbf{A}$ and $\mathbf{B}$. 
The determinant of a matrix $\mathbf{A}$ is denoted $\mathrm{det}(\mathbf{A})$. 
We use the notation $\mathbf{A} \CI \mathbf{B} | \mathbf{C}$ for 
``$\mathbf{A}$ is independent of $\mathbf{B}$ given $\mathbf{C}$”, and $\mathbf{A} \nCI \mathbf{B} | \mathbf{C}$ for the negation of the same sentence \citep{dawid1979conditional}.

\vspace{-2mm}
\subsection{Proximal Causal Learning}\label{Sub-sec-confounderproxy}
The proximal causal learning approach offers a new strategy for inferring the causal effect of interest in the presence of unmeasured confounders \citep{kuroki2014measurement,de2017proxy,miao2018proxy,wang2019jasa,shi2020multiply,tchetgen2020introduction}. Specifically, suppose that $X_k$ is the treatment, $Y$ is the outcome, and $\mathbf{U}$ represents the set of unmeasured confounders between $X_k$ and $Y$. 
The theory around the proximal causal learning approach says that the target causal effect of $X_k$ on $Y$ can be identified when two sets of proxy variables, $\mathbf{Z}$ and $\mathbf{W}$, are available for the unmeasured confounder $\mathbf{U}$. In such cases, the proxy set $\mathbf{Z}$, referred to as the Negative Control Exposure (NCE), does not causally affect the primary outcome $Y$, and another proper proxy set $\mathbf{W}$, called the Negative Control Outcome (NCO), is not causally affected by the treatment $X_k$.
The graphical condition for NCE and NCO relative to a target causal effect of $X_k$ on $Y$ is described in Definition \ref{defi-proximal-criteria}, and an illustrative example is provided accordingly.
\begin{Definition}[\textbf{NCE and NCO} \citep{miao2018proxy,shi2020selective}]\label{defi-proximal-criteria}
Given a target causal effect of $X_k$  on $ Y$ in the case where $\mathbf{U}$ are the set of unmeasured confounding between $X_k$ and $Y$, sets $\mathbf{Z}$ and $\mathbf{W}$ are the valid NCE and NCO respectively if the following conditions hold:
\begin{itemize}[leftmargin=15pt,itemsep=0pt,topsep=0pt,parsep=0pt]
    \item [1.] $\mathbf{Z}$ is independent of $Y$ conditional on $(\mathbf{U},X_k)$, i.e., $\mathbf{Z} \CI {Y} | (\mathbf{U},X_k)$, and
    \item [2.] $\mathbf{W}$ is independent of $(X_k,\mathbf{Z})$ conditional on $\mathbf{U}$, i.e., $\mathbf{W} \CI (X_k, \mathbf{Z}) | \mathbf{U}$.
\end{itemize}
\end{Definition}
\vspace{-3mm}
\begin{figure}[htp]
	\begin{center}
		\begin{tikzpicture}[scale=1.0, line width=0.5pt, inner sep=0.2mm, shorten >=.1pt, shorten <=.1pt]
			%
			\draw (3.0,1.6) node(cross) [] {{\footnotesize\,\color{red}{\XSolidBrush}\,}};
			\draw[-,dashed, line width =1.0pt] (Z) edge[bend right=-25] (W);
			\draw (2.0, -0.3) node(cross) [] {{\footnotesize\,\color{red}{\XSolidBrush}\,}};
			\draw[-, dashed, line width =1.0pt] (Z) .. controls (1.8, -0.5) .. (Y);
			\draw (4.0, -0.3) node(cross) [] {{\footnotesize\,\color{red}{\XSolidBrush}\,}};
			\draw[-, dashed, line width =1.0pt] (W) .. controls (4.2, -0.5) .. (X);
			%
			\draw (3, 1.0) node(U) [circle, fill=gray!60, minimum size=0.5cm,draw] {{\footnotesize\,$\mathbf{U}$\,}};
			\draw (1.5, 1.0) node(Z) [] {{\footnotesize\,$\mathbf{Z}$\,}};
			\draw (4.5, 1.0) node(W) [] {{\footnotesize\,$\mathbf{W}$\,}};
			\draw (2.25, 0) node(X) [] {{\footnotesize\,{$X_k$}\,}};
			\draw (3.75, 0) node(Y) [] {{\footnotesize\,{$Y$}\,}};
			%
			\draw (0.8, 1.0) node(ZZ) [] {{\footnotesize\,(NCE)\,}};
			\draw (5.2, 1.0) node(ZZ) [] {{\footnotesize\,(NCO)\,}};
			\draw[-arcsq] (U) -- (Z) node[pos=0.5,sloped,above] {}; 
			\draw[-arcsq] (U) -- (W)node[pos=0.5,sloped,above] {}; 
			\draw[-arcsq] (U) -- (X) node[pos=0.5,sloped,above] {};
			\draw[-arcsq] (U) -- (Y) node[pos=0.5,sloped,above] {};
			\draw[-arcsq] (X) -- (Y) node[pos=0.5,sloped,above] {};
			\draw[-] (Z) -- (X) node[pos=0.5,sloped,above] {};
			\draw[-arcsq] (W) -- (Y) node[pos=0.5,sloped,above] {};
		\end{tikzpicture}\\
            \vspace{-3mm}
		\caption{Diagram of one possible violation of NCE and NCO assumptions. Dashed lines represent active paths.  The symbol "{\color{red}{\XSolidBrush}}" indicates that the current active paths should not exist here.}
		\vspace{-5mm}
		\label{Fig-proximal-example-violation}
	\end{center}
\end{figure}
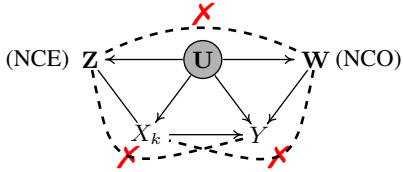

For the rest of the paper, we will call the above two conditions the graphical criteria for proxy variables validity, or simply \textbf{\emph{proximal criteria}}.
Figure \ref{Fig-proximal-example-violation} is an illustration of the NCE and NCO conditions and one potential violation of the NCE and NCO conditions. Notice that the validity of both proxy variables NCE and NCO are mutually dependent on each other. Therefore, when we say that proxy variables are valid for a causal relationship, we mean that both NCE and NCO are valid simultaneously.
\begin{Definition}[\textbf{Connected (Disconnected) NCE and NCO}]
    Assume $\mathbf{Z}$ and $\mathbf{W}$ are NCE and NCO of unmeasured confounders $\mathbf{U}$ for the causal relationship $X_k \to Y$. We refer to the set $\mathbf{Z}$ as { Connected (Disconnected) NCE if $\mathbf{Z} \nCI X_k | \mathbf{U}$ ($\mathbf{Z} \CI X_k | \mathbf{U}$).} Similarly, we refer to the set $\mathbf{W}$ as Connected (Disconnected) NCO if $\mathbf{W} \nCI Y | \mathbf{U}$ ($\mathbf{W} \CI Y | \mathbf{U}$).
\end{Definition}
\begin{Definition}[\textbf{Quadruple-disconnected NC}]
	Assume $\mathbf{Z}$ and $\mathbf{W}$ are NCE and NCO of unmeasured confounder $\mathbf{U}$ for the causal relationship $X_k \to Y$. We say a variable $Q$ is a Quadruple-disconnected NC if  $Q \CI X_k | \mathbf{U}$, $Q \CI Y | \mathbf{U}$, $Q \CI \mathbf{Z} | \mathbf{U}$, and $Q \CI \mathbf{W} | \mathbf{U}$. 
\end{Definition}

\begin{Example-set}
    { Consider the causal relationship $X_2\to Y$ in Figure \ref{Fig-model-example}. $X_1$ and $X_6$ are valid disconnected NCE and disconnected NCO relative to $X_2\to Y$, respectively. Because $X_3 \CI X_2 | \mathbf{U}$, $X_3 \CI Y | \mathbf{U}$, $X_3 \CI X_1 | \mathbf{U}$, and $X_3 \CI X_6 | \mathbf{U}$, $X_3$ can serve as a Quadruple-disconnected NC.}
\end{Example-set}

\begin{Proposition}[\textbf{Proxy Variables Estimator}~\citep{kuroki2014measurement}]\label{Pro-Proxy-Estimator}
Assume the system is a linear causal model, i.e., all variables are continuous and the causal relationships among variables are linear. Further, assume that there exist one unmeasured confounder $U$ that affects both treatment $X_k$ and outcome $Y$, and that $Z$ and $W$ are NCE and NCO {of confounder ${U}$}, e.g.,  the causal graph in Figure \ref{Fig-proximal-example}, the unbiased estimator for the causal effect $\beta_{X_k \to Y}$ of $X_k$ on $Y$ is as follows,
\begin{nospaceflalign}\label{Eq-single-proximal-inference}
\beta_{X_k \to Y} & = \frac{\sigma_{X_kY}\sigma_{WZ}-\sigma_{X_kW}\sigma_{YZ}}{\sigma_{X_kX_k}\sigma_{WZ}-\sigma_{X_kW}\sigma_{X_kZ}}
\end{nospaceflalign} 
where $\sigma_{X_kY}$ is the covariance between $X_k$ and $Y$, etc. 
\end{Proposition}
It is worth noting that this standard estimator is only applicable in the case of a single $U$. For scenarios involving multiple confounders $\mathbf{U}$, please refer to the extended estimator introduced in Section \ref{Section-extended-estimator}.

\vspace{-2mm}
\subsection{Model Definition}\label{Sub-Section-Model-Definition}

In this paper, let $\mathbf{X}=\{X_1,\ldots, X_p\}^\intercal$   denote a vector of $p$-dimensional treatments, $Y$  denote an outcome, and ${\bf{U}}= \{U_1,\ldots,U_q\}^\intercal$  denote a vector of $q$-dimensional unmeasured confounders. Analogous to \citet{wang2019jasa,ogburn2019comment,damour2019comment}, we consider the case that $\bf{U}$ affects both treatments $\mathbf{X}$ and outcome $Y$.   Without loss of generality, we assume that all variables have a zero mean. We here restrict our attention to a linear acyclic causal model,
\vspace{-2mm}
\begin{equation}\label{Eq-model-defi}
\begin{aligned}
\mathbf{X} &= \mathbf{B} \mathbf{X} + \mathbf{C} \mathbf{U} + \varepsilon_{\mathbf{X}}, \quad  c_{ij} \neq 0, \\
Y &= \bm{\beta}^{\intercal} \mathbf{X}  + \bm{\delta}^\intercal \mathbf{U} + \varepsilon_{{Y}}, \quad  {\bm{\delta}}_i \neq 0, 
\end{aligned}
\end{equation}
\vspace{-1mm}
where $\bm{\beta} $ is the column vector that signifies the causal effects of interest. The noise terms in $\varepsilon_{\mathbf{X}}$ and $\mathbf{U}$ are independent of each other, $\varepsilon_{{Y}}$ is independent of $\mathbf{X}$ and $\mathbf{U}$. We assume that the generating process is recursive. That is to say, the causal relationships among variables can be represented by a DAG \citep{pearl2009causality,spirtes2000causation}.

In contrast to the single-treatment model studied in \citet{kummerfeld2022data}, where potential variables cannot directly influence the treatment and outcome of interest, our work explores a more general scenario. In our study, the system of interest accommodates multiple unmeasured confounders and multiple treatments, where latent  variables can have a direct effect on both the treatment and the outcome of interest.
Figure \ref{Fig-model-example} provides a simple graph that satisfies our model while violating the model in \citet{kummerfeld2022data}, where the variables $X_i, i=1, \dots,6$ are potential treatments, $Y$ is the outcome, and $U$ is an unmeasured confounder.
\begin{figure}[htp]
	\begin{center}
		\begin{tikzpicture}[scale=1.0, line width=0.5pt, inner sep=0.2mm, shorten >=.1pt, shorten <=.1pt]
			\draw (3, 1.5) node(U) [circle, fill=gray!60, minimum size=0.5cm,draw] {{\footnotesize\,${U}$\,}};
			\draw (1, 0) node(X1) [] {{\footnotesize\,$X_1$\,}};
			\draw (2, 0) node(X2) [] {{\footnotesize\,$X_2$\,}};
			\draw (3, 0) node(X3) [] {{\footnotesize\,{$X_3$}\,}};
			\draw (4, 0) node(X4) [] {{\footnotesize\,{$X_4$}\,}};
			\draw (5, 0) node(X5) [] {{\footnotesize\,{$X_5$}\,}};
			\draw (6, 0) node(X6) [] {{\footnotesize\,{$X_6$}\,}};
			\draw (7, 0) node(Y) [circle, draw]  {{\footnotesize\,{$Y$}\,}};
			\draw[-arcsq] (U) -- (X1) node[pos=0.5,sloped,above] {}; 
			\draw[-arcsq] (U) -- (X2)node[pos=0.5,sloped,above] {}; 
			\draw[-arcsq] (U) -- (X3) node[pos=0.5,sloped,above] {};
			\draw[-arcsq] (U) -- (X4) node[pos=0.5,sloped,above] {};
			\draw[-arcsq] (U) -- (X5) node[pos=0.5,sloped,above] {};
			\draw[-arcsq] (U) -- (X6) node[pos=0.5,sloped,above] {};
			\draw[-arcsq] (X1) -- (X2) node[pos=0.5,sloped,above] {};
			\draw[-arcsq] (X4) -- (X5) node[pos=0.5,sloped,above] {};
			%
			\draw[-arcsq] (U) -- (Y) node[pos=0.5,sloped,above] {}; 
			\draw [-arcsq] (X2) edge[bend right=20] (Y);
			\draw[-arcsq] (X5) -- (X6) node[pos=0.5,sloped,above] {};
			\draw[-arcsq] (X6) -- (Y) node[pos=0.5,sloped,above] {};
		\end{tikzpicture}
            \vspace{-3mm}
		\caption{A simple causal graph involving 6 potential treatments and one outcome.}
		\vspace{-5mm}
		\label{Fig-model-example}
	\end{center}
\end{figure}
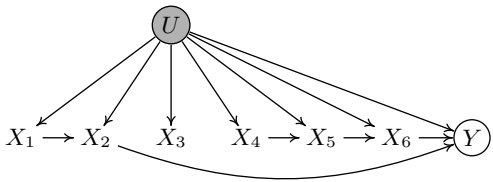

For convenience, we further assume that the number of unmeasured confounders is known. When the number of confounders is unknown, consistent estimation of the number of confounders has been well-established by \citet{bai2002determining} under factor models. In practical applications, one may directly use open-source software, to perform a significance test to determine whether the number of confounders in a factor model is sufficient to capture the full dimensionality of the dataset, as stated in \citet{miao2022identifying}. Notice that the model of Equation \ref{Eq-model-defi} assumes that $\bf{U}$ affects both treatments $\mathbf{X}$ and outcome $Y$, i.e., all entries of $\mathbf{C}$ are non-zero, as the studied in \citet{kummerfeld2022data}.
In Appendix \ref{Appendix-Discussion}, we explored a scenario in which, after applying certain necessary preprocessing steps, our theory remains applicable, even when certain entries of $\mathbf{C}$ are zero.

\textbf{Goal}: The goal of this paper is to identify sets NCE and NCO of unmeasured confounders that satisfy proximal criteria for a given casual relationship $X_k \to Y, X_i \in \mathbf{X}$ and estimate the total causal effect of treatment $X_i$ on $Y$ simultaneously.

\begin{Remark}
In the problem described above, conventional constraint-based causal discovery methods that account for unmeasured confounders, such as the FCI (Fast Causal Inference) algorithm \citep{spirtes1995causal,zhang2008completeness} or its variants, like the RFCI algorithm (Really Fast Causal Inference) \citep{colombo2012learning}, result in a fully connected causal graph. This occurs because unmeasured confounders $\mathbf{U}$ affect both the treatments $X$ and $Y$.
As a result, it becomes challenging to identify valid Negative Control Exposure (NCE) and Negative Control Outcome (NCO)   for unmeasured confounders from the resulting (fully connected) graph.
\end{Remark}

\vspace{-3mm}
\section{Extended Proxy Variables Estimator with Multiple Unmeasured Confounders}\label{Section-extended-estimator}

In this section, we will extend the existing proxy variable estimator with a single unmeasured confounder (Proposition \ref{Pro-Proxy-Estimator}) to handle the case when there exist multiple unmeasured confounders between treatment and outcome. Specifically, we build upon the work of \citet{kuroki2014measurement} on the proxy variable estimator with an unmeasured confounder and extend it to include multiple unmeasured confounders in the case of a linear causal model. To improve readability, we defer all proofs to Appendix \ref{Appendix-Section-Proofs}.

\begin{Proposition}[\textbf{Extended Proxy Variables Estimator}]\label{Pro-Multi-Proxy-Estimator}
  Assume the system is a linear causal model, i.e., all variables are continuous and the causal relationships among variables are linear, and assume there exist $q$ unmeasured confounders, denoted by $\mathbf{U}$, that affect both treatment $X_k$ and outcome $Y$. Let $\mathbf{Z}$ with $|\mathbf{Z}|=q$ and $\mathbf{W}$ with $|\mathbf{W}|=q$ be two valid NCE and NCO of $\mathbf{U}$ respectively. Thus, the unbiased estimator for the total causal effect $\beta_{X_k \to Y}$ of $X_k$ on $Y$ is as follows,
\begin{nospaceflalign}\label{Eq-multiple-proximal-inference}
\beta_{X_k \to Y} & = \frac{\mathrm{det}(\boldsymbol{\Sigma}_{\{X_k \cup \mathbf{Z}\}, \{Y \cup \mathbf{W}\}})}{\mathrm{det}(\boldsymbol{\Sigma}_{\{X_k \cup \mathbf{Z}\}, \{X_k \cup \mathbf{W}\}})}.
\end{nospaceflalign}  
\end{Proposition}
Proposition \ref{Pro-Multi-Proxy-Estimator} asserts that, given two valid $q$-dimensional NCE and $q$-dimensional NCO for the causal relationship $X_k \to Y$ when there exist $q$-dimensional unmeasured confounders, then the total causal effect of $X_k$ on $Y$ can be consistently estimated using Eq. \ref{Eq-multiple-proximal-inference}. Note that if the dimension of NCE and NCO is less than $q$, the estimated ${\beta}_{X_k \to Y}$ will be biased (see an example described in Appendix. \ref{Appendix-Extended-Proxy}).

\begin{Remark}
    If $q=1$, i.e., there exists only one unmeasured confounder between $X_k$ and $Y$, the estimator in Proposition 2 is equal to the estimator in Proposition 1.
\end{Remark}

\begin{Remark}\label{remark-necessary-conditions}
    According to Proposition 2, for a given causal relationship $X_k \to Y$, the necessary conditions for the NCE $\mathbf{Z}$ and NCO $\mathbf{W}$ to be valid are that $\boldsymbol{\Sigma}_{\{X_k, \mathbf{Z}\},  \{Y,\mathbf{W}\}}$ and $\boldsymbol{\Sigma}_{\{X_k, \mathbf{Z}\},  \{X_k,\mathbf{W}\}}$ both are full rank, i.e., $\mathrm{rk}(\boldsymbol{\Sigma}_{\{X_k, \mathbf{Z}\},  \{Y,\mathbf{W}\}}) = q+1$,
$\mathrm{rk}(\boldsymbol{\Sigma}_{\{X_k, \mathbf{Z}\},  \{X_k,\mathbf{W}\}}) = q+1$.
\end{Remark}

\vspace{-3mm}
\section{Identifiability with Second-Order Statistics}
In this section, we first investigate the  identifiability of the proxy variables in the model described in Eq. \ref{Eq-model-defi} with second-order statistics. Then, we provide a data-driven method for selecting valid proxy variables (i.e., NCE and NCO) of each treatment $X_k$ on outcome $Y$ and obtaining its corresponding unbiased causal effect of $X_k$ on $Y$ simultaneously.
All proofs are included in Appendix \ref{Appendix-Section-Proofs}.

\vspace{-2mm}
\subsection{Identification of Proxy Variables with Second-Order Statistics}\label{Subsection-Second-Order-Statistics}

In this section, we investigate the identifiability of proxy variables using second-order statistics. Before giving our main results, we first introduce the concept of rank constraints (which is an extension of the famous Tetrad constraints presented in \citet{spearman1928pearson}), which is an essential constraint that leverages the second-order statistics derived from the data \citep{Sullivant-T-separation,spirtes2013calculation-t-separation}.

\begin{figure*}[t!]
	\centering
	\begin{tikzpicture}[scale=1.0, line width=0.5pt, inner sep=0.2mm, shorten >=.1pt, shorten <=.1pt]
		\draw (3, 1.0) node(U) [circle, fill=gray!60, minimum size=0.5cm,draw] {{\footnotesize\,$U$\,}};
		\draw (1.5, 1.0) node(Z) [] {{\footnotesize\,$Z$\,}};
		\draw (4.5, 1.0) node(W) [] {{\footnotesize\,$W$\,}};
		\draw (2.25, 0) node(X) [] {{\footnotesize\,{$X_k$}\,}};
		\draw (3.75, 0) node(Y) [] {{\footnotesize\,{$Y$}\,}};
		%
		%
		\draw[-arcsq] (U) -- (Z) node[pos=0.5,sloped,above] {{\scriptsize\,$a$\,}}; 
		\draw[-arcsq] (U) -- (W)node[pos=0.5,sloped,above] {{\scriptsize\,$d$\,}}; 
		\draw[-arcsq] (U) -- (X) node[pos=0.5,sloped,above] {{\scriptsize\,$b$\,}};
		\draw[-arcsq] (U) -- (Y) node[pos=0.5,sloped,above] {{\scriptsize\,$c$\,}};
		\draw[-arcsq] (X) -- (Y) node[pos=0.5,sloped,above] {{\scriptsize\,$\beta$\,}};
    	\draw[-arcsq] (Z) -- (X) node[pos=0.5,sloped,above] {{\scriptsize\,$e$\,}};
		\draw[-arcsq] (W) -- (Y) node[pos=0.5,sloped,above] {{\scriptsize\,$f$\,}};
		\draw (3,-0.5) node(label-ii) [] {{\footnotesize\,(a)\,}};
	\end{tikzpicture}~~~~~~
	\begin{tikzpicture}[scale=1.0, line width=0.5pt, inner sep=0.2mm, shorten >=.1pt, shorten <=.1pt]
	\draw (3, 1.0) node(U) [circle, fill=gray!60, minimum size=0.5cm,draw] {{\footnotesize\,$U$\,}};
	\draw (1.5, 1.0) node(Z) [] {{\footnotesize\,$Z$\,}};
	\draw (4.5, 1.0) node(W) [] {{\footnotesize\,$W$\,}};
	\draw (2.25, 0) node(X) [] {{\footnotesize\,{$X_k$}\,}};
	\draw (3.75, 0) node(Y) [] {{\footnotesize\,{$Y$}\,}};
	%
	%
	\draw[-arcsq] (U) -- (Z) node[pos=0.5,sloped,above] {{\scriptsize\,$a$\,}}; 
	\draw[-arcsq] (U) -- (W)node[pos=0.5,sloped,above] {{\scriptsize\,$d$\,}}; 
	\draw[-arcsq] (U) -- (X) node[pos=0.3,sloped,above] {{\scriptsize\,$b$\,}};
	\draw[-arcsq] (U) -- (Y) node[pos=0.5,sloped,above] {{\scriptsize\,$c$\,}};
	\draw[-arcsq] (X) -- (Y) node[pos=0.5,sloped,above] {{\scriptsize\,$\beta$\,}}; 
        \draw[-arcsq] (Z) -- (X) node[pos=0.5,sloped,above] {{\scriptsize\,$e$\,}}; 
	\draw[-arcsq, color=red] (Z) -- (Y) node[pos=0.3,sloped,above] {{\scriptsize\,$g$\,}}; 
	\draw[-arcsq] (W) -- (Y) node[pos=0.5,sloped,above] {{\scriptsize\,$f$\,}}; 
	\draw (3,-0.5) node(label-ii) [] {{\footnotesize\,(b)\,}};
	\end{tikzpicture}~~~~~~
    \begin{tikzpicture}[scale=1.0, line width=0.5pt, inner sep=0.2mm, shorten >=.1pt, shorten <=.1pt]
	\draw (3, 1.0) node(U) [circle, fill=gray!60, minimum size=0.5cm,draw] {{\footnotesize\,$U$\,}};
	\draw (1.5, 1.0) node(Z) [] {{\footnotesize\,$Z$\,}};
	\draw (4.5, 1.0) node(W) [] {{\footnotesize\,$W$\,}};
	\draw (2.25, 0) node(X) [] {{\footnotesize\,{$X_k$}\,}};
	\draw (3.75, 0) node(Y) [] {{\footnotesize\,{$Y$}\,}};
	%
	%
	\draw[-arcsq] (U) -- (Z) node[pos=0.5,sloped,above] {{\scriptsize\,$a$\,}}; 
	\draw[-arcsq] (U) -- (W)node[pos=0.5,sloped,above] {{\scriptsize\,$d$\,}}; 
	\draw[-arcsq] (U) -- (X) node[pos=0.5,sloped,above] {{\scriptsize\,$b$\,}};
	\draw[-arcsq] (U) -- (Y) node[pos=0.3,sloped,above] {{\scriptsize\,$c$\,}};
	\draw[-arcsq] (X) -- (Y) node[pos=0.5,sloped,above] {{\scriptsize\,$\beta$\,}}; 
	\draw[-arcsq] (Z) -- (X) node[pos=0.5,sloped,above] {{\scriptsize\,$e$\,}}; 
        \draw[-arcsq] (W) -- (Y) node[pos=0.5,sloped,above] {{\scriptsize\,$f$\,}}; 
	\draw[-arcsq, color=red] (W) -- (X) node[pos=0.3,sloped,above] {{\scriptsize\,$g$\,}}; 
	\draw (3,-0.5) node(label-ii) [] {{\footnotesize\,(c)\,}};
\end{tikzpicture}
    \vspace{-4mm}
	\caption{A linear causal  model with any of the graphical structures above entails all possible rank constraints in the marginal covariance matrix of $\{X_k, Y, Z, W\}$. }
	\label{Fig-equivalent-proximal-example}
 \vspace{-5mm}
\end{figure*}
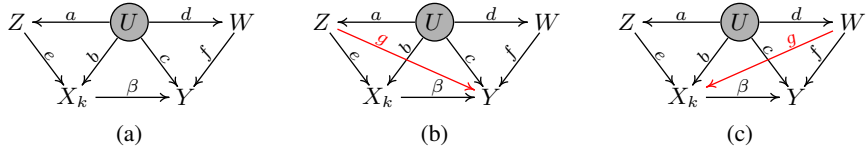

\begin{Definition}[\textbf{Rank Constraint}]
Suppose all variables follow the linear acyclic causal model. Let $\mathbf{A}$ and $\mathbf{B}$ be two sets of random variables.
A rank constraint in the submatrix of the covariance matrix $\mathrm{rk}(\boldsymbol{\Sigma}_{\mathbf{A}, \mathbf{B}})$ is any constraint of the type $\mathrm{rk}(\boldsymbol{\Sigma}_{\mathbf{A}, \mathbf{B}}) \leq r$, where $r$ is some constant.
\end{Definition}
It is noteworthy that if one uses rank constraints to structural constraints with unobserved variables, the rank-faithful assumption is necessary \citep{spirtes2013calculation-t-separation}. 

\begin{Definition}[rank-faithfulness]\label{Ass_rankfaithful}
Let a probability distribution $P$ be rank-faithful to a DAG $\mathcal{G}$ if every rank constraint on a sub-covariance matrix that holds in $P$ is entailed by every linear structural equation model with respect to $\mathcal{G}$.
\vspace{-2mm}
\end{Definition}
The rank-faithfulness assumption allows us to use the rank-deficiency constraints to impose structural constraints with unobserved variables.
Intuitively speaking, the set of values of free parameters for which $\mathrm{rk}(\boldsymbol{\Sigma}_{\mathbf{A}, \mathbf{B}}) \leq r$ has a Lebesgue measure of 0. Note that this assumption does not restrict the data distribution, making it a distribution-free assumption.
Furthermore, the practicality of rank-faithfulness has been demonstrated through simulation results and applications in \citet{kummerfeld2016causal,xie2022identification,huang2022latent}, as well as in our paper.
For a further discussion of rank-faithfulness assumptions, please refer to Section 4 in \citet{spirtes2013calculation-t-separation} for more details.

\textbf{\emph{A Motivating Example:}} Before showing the theoretical results, we give a simple example to illustrate {the basic idea}. Consider the causal diagram in Figure \ref{Fig-model-example}. We observe a causal relationship $X_2 \to Y$. In this case, $X_1$ and $X_6$ serve as valid NCE and NCO, respectively, for this causal relationship.
The cross-covariance matrix $\boldsymbol{\Sigma}_{\{X_2, X_3, X_1\}, \{X_2, Y, X_6\}} $ is singular, that is,
\begin{nospaceflalign}
    \mathrm{det}(\boldsymbol{\Sigma}_{\{X_2, X_3, X_1\},  \{X_2,Y, X_6\}})=0. 
\label{Eq_Rank1-intuition}
\end{nospaceflalign}

By Eq.\ref{Eq_Rank1-intuition}, we quickly know that $\mathrm{rk}(\boldsymbol{\Sigma}_{\{X_2, X_3, X_1\},  \{X_2,Y, X_6\}}) \leq 2$.
We introduce an edge $X_1 \to Y$ in the graph of Figure \ref{Fig-model-example}, causing $X_1$ and $X_6$ to become invalid NCE and NCO, respectively, concerning the causal relationship $X_2 \to Y$.
 The cross-covariance matrix $\boldsymbol{\Sigma}_{\{X_2, X_3, X_1\}, \{X_2, Y, X_6\}}$ will no longer have a vanishing determinant, and instead, 
\begin{nospaceflalign}
\mathrm{det}(\boldsymbol{\Sigma}_{\{X_2, X_3, X_1\},  \{X_2,Y, X_6\}}) \neq 0.
\label{Eq_Rank2-intuition}
\end{nospaceflalign}

That is to say, $\boldsymbol{\Sigma}_{\{X_2, X_3, X_1\},  \{X_2,Y, X_6\}}$ is full rank.
Assuming the distribution is rank-faithful to the graph, the above facts show that lack of edge $X_1 \to Y$, i.e., the variable of NCE does not causally affect the primary outcome, has a testable implication.

We now investigate the conditions under which the valid NCE and NCO of unmeasured confounder $\mathbf{U}$ relative to a causal relationship $X_k \to Y$ can be identified in terms of rank constraints. To estimate the causal effect of $X_k \in \mathbf{X}$ on $Y$ in the system, according to Proposition \ref{Pro-Multi-Proxy-Estimator}, the minimal condition is as follows,
\begin{Assumption}\label{Ass-minimal}
    For a given causal relationship $X_k \to Y$ in the system, there exist at least $q$ variables in $\mathbf{X}$ that qualify as NCE and $q$ variables in $\mathbf{X}$ that qualify as NCO.
\end{Assumption}
Assumption \ref{Ass-minimal} is a very natural condition that one expects to hold. This is because if Assumption \ref{Ass-minimal} fails, i.e., there are no valid sets of NCE and NCO for the causal relationship in the system, then we can not estimate the unbiased causal effect of interest using the {extended proxy variables estimator}.
Unfortunately, Assumption \ref{Ass-minimal} is an \emph{insufficient} condition for identifying the sets of NCE and NCO in terms of rank constraints. An illustrative example is given below. 

\begin{Example-set}[Counterexample]
Consider the causal diagrams shown in Figure \ref{Fig-equivalent-proximal-example}.  Assume that the data are generated from a linear causal model and rank-faithfulness holds. We find that all possible rank constraints are full-rank (no rank-deficiency) in the marginal matrix of $\{X_k, Y, Z, W\}$, e.g., $\mathrm{rk}(\boldsymbol{\Sigma}_{\{X_k, Y\},  \{Z, W\}}) = 2$ in three subgraphs. However, according to proximal criteria, we know that only in subgraph (a), $Z$ and $W$ are NCE and NCO of unmeasured confounder $U$ for the causal relationship $X_k \to Y$, while they are not in other subgraphs (b) and (c).
The above facts imply that one can not identify valid NCE and NCO using rank constraints under Assumption \ref{Ass-minimal}.
\end{Example-set}

We next give two sufficient conditions that render the sets of NCE and NCO of the unmeasured confounders $\mathbf{U}$ relative to a causal relationship $X_k \to Y$ identifiable, respectively.

\begin{Assumption}\label{Ass_nc_rank1}
    For a given causal relationship $X_k \to Y$ in the system, the following conditions hold: i) there exist at least $q$ variables in $\mathbf{X}$ that qualify as NCE and $q$ variables in $\mathbf{X}$ that qualify as NCO, and ii) there exist at least one Quadruple-disconnected NC relative to  $X_k \to Y$.
\end{Assumption}
Assumption \ref{Ass_nc_rank1} says that apart from satisfying the minimum number of NCE and NCO, the system also requires at least one additional Quadruple-disconnected NC. 
{Assumption \ref{Ass_nc_rank1} is much milder than the assumptions considered in \citet{kummerfeld2022data}, since for a causal relationship $X_k \to Y$ in the presence of one unmeasured confounder, we only need one Quadruple-disconnected NC, one unrestricted NCE, and one unrestricted NCO, while they need three Quadruple-disconnected NCs.}
Roughly speaking, the existence of a quadruple disconnected negative control, termed as $Q$, simplifies the condition testing in the proximal criteria. That is to say, when verifying the two criteria of proximal criteria, that is, when finding two Rank-Deficiency Constraints, it is only necessary to add this variable to the rows and columns of the cross-covariance matrix. For example, the variable $X_3$ in Example 3 is a quadruple disconnected negative control.

\begin{Lemma}\label{Lemma-rule1}
Assume that the input data $\mathbf{X}$ and $Y$ strictly follow the Equation \ref{Eq-model-defi} and the rank-faithfulness holds. If Assumption \ref{Ass_nc_rank1} holds, then the underlying NCE and NCO relative to the causal relationship $X_k \to Y$ can be identified by using the following rule.
\begin{itemize}
\vspace{-4mm}
\item[$\mathcal{R}1$.] Let $\mathbf{A}$ and $\mathbf{B}$ be two disjoint subsets of $\mathbf{X}$, where $|\mathbf{A}|=q$ and $|\mathbf{B}|=q$. Furthermore, let $Q$ be a variable in $\{\mathbf{X} \setminus \{\mathbf{A} \cup \mathbf{B} \cup X_k\}\}$. If 1) $\mathrm{rk}(\boldsymbol{\Sigma}_{\{X_k, Q, \mathbf{A}\}, \{X_k,Y,\mathbf{B}\}}) \leq q+1$, and 2) $\mathrm{rk}(\boldsymbol{\Sigma}_{\{X_k, \mathbf{A}\},  \{Q, \mathbf{B}\}}) \leq q$, then $\mathbf{A}$ and $\mathbf{B}$ are valid NCE and NCO relative to $X_k \to Y$ respectively.
\end{itemize}
\vspace{-3mm}
\end{Lemma}

\begin{Example-set}[$\mathcal{R}1$]
Let's consider the causal diagram shown in Figure \ref{Fig-model-example}. We consider the causal relationship $X_2 \to Y$ (Assumption \ref{Ass_nc_rank1} holds for $X_2 \to Y$). Let $\mathbf{A}=\{X_1\}$, $\mathbf{B}=\{X_6\}$, and $Q = X_3$. We check $\mathcal{R}1$ and obtain that 
	1) $\mathrm{rk}(\boldsymbol{\Sigma}_{\{X_2, X_3, X_1\},  \{X_2,Y, X_6\}}) \leq 2$, and  2) $\mathrm{rk}(\boldsymbol{\Sigma}_{\{X_2, X_1\},  \{X_3, X_6\}}) \leq 1$. These facts imply that $X_1$ and $X_6$ are valid NCE and NCO relative to $X_2 \to Y$, respectively.
\end{Example-set}

We next introduce another sufficient condition when there is no proper Quadruple-disconnected NC mentioned in Assumption \ref{Ass_nc_rank1} in the system.

\begin{Assumption}\label{Ass_nc_rank2}
    For a given causal relationship $X_k \to Y$ in the system, there exist at least $q+1$ variables in $\mathbf{X}$ that qualify as NCE and $q+1$ variables in $\mathbf{X}$ that qualify as NCO.
\end{Assumption}

{Assumption \ref{Ass_nc_rank2} states that apart from satisfying the minimum number of NCE and NCO, i.e., $q$ NCE and $q$ NCO, the system also requires at least one additional NCE and one additional NCO.}

\begin{Lemma}\label{Lemma-rule2}
Assume that the input data $\mathbf{X}$ and $Y$ strictly follow the Equation \ref{Eq-model-defi} and the rank-faithfulness holds. If Assumption \ref{Ass_nc_rank2} holds, then the underlying NCE and NCO relative to the causal relationship $X_k \to Y$ can be identified by using the following rule.
\begin{itemize}
\vspace{-4mm}
\item[$\mathcal{R}2$.] Let $\mathbf{A}$ and $\mathbf{B}$ be two disjoint subsets of $\mathbf{X}$, where $|\mathbf{A}|=q+1$ and $|\mathbf{B}|=q+1$. If 1) $\mathrm{rk}(\boldsymbol{\Sigma}_{\{X_k, \mathbf{A}\}, \{X_k,Y,\mathbf{B}\}}) \leq q+1$, and 2) $\mathrm{rk}(\boldsymbol{\Sigma}_{\{X_k, \mathbf{A}\},  \mathbf{B}}) \leq q$, then $\mathbf{A}$ and $\mathbf{B}$ are valid NCE and NCO relative to $X_k \to Y$ respectively.
\end{itemize}
\vspace{-2mm}
\end{Lemma}

\begin{Example-set}[$\mathcal{R}2$]\label{Example-rank-rule2}
Continue to consider the causal diagram shown in Figure \ref{Fig-model-example}. We now consider the causal relationship $X_6 \to Y$. Let $\mathbf{A}=\{X_4,X_5\}$, and $\mathbf{B}=\{X_1,X_2\}$. We check $\mathcal{R}2$ and obtain that 
	1) $\mathrm{rk}(\boldsymbol{\Sigma}_{\{X_6, X_4, X_5\},  \{X_6,Y, X_1,X_2\}}) \leq 2$, and 2) $\mathrm{rk}(\boldsymbol{\Sigma}_{\{X_6, X_4, X_5\},  \{X_1, X_2\}}) \leq 1$. These facts imply that $\{X_4,X_5\}$ and $\{X_1,X_2\}$ are valid NCE and NCO relative to $X_6\to Y$ respectively.
\end{Example-set}

Building upon Lemmas \ref{Lemma-rule1} and \ref{Lemma-rule2},
we provide graphical conditions that are sufficient for the identifiability of NCE and NCO in terms of rank constraints.

\begin{Theorem}[Identifiability of NCE and NCO with Rank Constraints]\label{Theo-NCE-Rank}
Assume that the input data $\mathbf{X}$ and $Y$ strictly follow the Equation \ref{Eq-model-defi} and the rank-faithfulness holds.  
Then the underlying NCE and NCO relative to the causal relationship $X_k \to Y$ can be identified if Assumption \ref{Ass_nc_rank1} or \ref{Ass_nc_rank2} is satisfied.
\end{Theorem}

\vspace{-4mm}
\subsection{Algorithm}\label{Section-Algorithm-proxy-rank}
In this section, we will leverage the above theoretical results and propose a data-driven method called Proxy-Rank to estimate the total causal effects of treatment $X_k \in \mathbf{X}$ on outcome $Y$:

\vspace{-3mm}
\rule{8.2cm}{1.2pt}
Proxy-Rank algorithm
\begin{description}[noitemsep,topsep=0pt,leftmargin=25pt]
    \item[1.] Given a $p$-dimensional treatments $\mathbf{X}$, outcome $Y$, the number of unmeasured confounders $q$. Initialize the sets of NCE, and NCO, causal effect, as $\mathcal{NCE}$, $\mathcal{NCO}$, and $\mathcal{C}$, respectively, with an empty set, i.e., $\mathcal{NCE} \coloneqq \emptyset$, $\mathcal{NCO} \coloneqq \emptyset$, and $\mathcal{C} \coloneqq \emptyset$.
    \item[2.]  Find valid NCE and NCO of unmeasured confounder $\mathbf{U}$ relative to per causal relationship $X_k \to Y$ according to Lemmas \ref{Lemma-rule1} and \ref{Lemma-rule2}.
    \item[3.] Estimate the corresponding unbiased causal effect by Proposition \ref{Pro-Multi-Proxy-Estimator} given NCE and NCO for per causal relationship $X_k \to Y$. Otherwise, output a value (NA) indicating the lack of knowledge to obtain the unbiased causal effect.
    \vspace{-2.0mm}
\end{description}
\vspace{-2mm}
\rule{8.2cm}{1.2pt}

The specific details of algorithm execution are provided in the Appendix \ref{Appendix-Algorithm-rank}.

We now show the correctness of the proposed algorithm. That is to say, for the causal relationships of interest, our algorithm outputs the true proxy variables and the unbiased estimation of causal effects.

\begin{Theorem}[Correctness]\label{Theorem-Corr-Rank}
Assume that the data $Y$ and $\mathbf{X}$ strictly follow the Equation \ref{Eq-model-defi} and the rank-faithfulness holds. Given infinite samples, the Proxy-Rank algorithm outputs the true causal effect $\mathcal{C}$ correctly.
\end{Theorem}
For more discussion on the consistency result and convergence rate of the above theorem, please refer to Appendix \ref{Appendix-Consistency-The2}.

We finally analyze the complexity of the Proxy-Rank algorithm. Let $q$ be the number of latent confounders, and $p$ be the number of treatments. There are two dominant parts. One dominant part is to check $\mathcal{R}_1$ of Lemma \ref{Lemma-rule1} with worst-case complexity is $\mathcal{O}\left(\frac{p!}{q! \cdot q! \cdot (p-2q-1)!}\right)$. The other dominant part is to check $\mathcal{R}_2$ of Lemma \ref{Lemma-rule2} with worst-case complexity is also $\mathcal{O}\left(\frac{p!}{q! \cdot q! \cdot (p-2q-1)!}\right)$. Hence, the worst-case complexity of the Proxy-Rank algorithm is $\mathcal{O}\left(\frac{p!}{q! \cdot q! \cdot (p-2q-1)!}\right)$.

\vspace{-2mm}
\section{Identifiability with Higher-Order Statistics}

In the above section, we have shown that the proxy variables can be identified with the help of rank constraints of the covariance matrix (second-order statistics) under some mild assumptions. However, if Assumption \ref{Ass_nc_rank1} or Assumption \ref{Ass_nc_rank2} is violated, e.g., the dimension of the sets of NCE and NCO are exactly $q$, the underlying NCE and NCO are not guaranteed to be identified with second-order statistics. To tackle this issue, below we show that we can benefit from higher-order statistics of the noise terms. 
Then, we provide another data-driven method with higher-order statistics for selecting valid proxy variables for a given causal relationship $X_k \to Y$ and obtaining its unbiased causal effect of $X_k$ on $Y$.

\vspace{-2mm}
\subsection{Identification of Proxy Variables with Higher-Order Statistics}\label{Subsection-Higher-Order-Statistics}

We assume the model of interest is a linear causal model with non-Gaussian error terms (also known as Linear, Non-Gaussian, Acyclic Model, shortly LiNGAM) \citep{shimizu2006linear}. Specifically, the assumption is as follows,

\begin{Assumption}\label{Ass-non-Gaussianity}
	The noise terms of variables follow non-Gaussian distributions.
\end{Assumption}
\vspace{-2mm}

Assumption \ref{Ass-non-Gaussianity} states the non-Gaussianity of data, which is expected to be ubiquitous, due to Cram\'{e}r Decomposition Theorem \citep{Cramer62}, as stated in \citet{spirtes2016causal}. Within the framework of this assumption, a significant body of research has already been initiated \citep{hyvarinen2010estimation, wang2020high,salehkaleybar2020learning,zhao2022learning}. For further reference, we recommend consulting the work of \citet{shimizu2022statistical}.

We next introduce an important constraint, Generalized Independent Noise (GIN) (which is an extension of the familiar Independent Noise (IN) constraint presented in \citet{shimizu2011directlingam}), which is an essential
constraint that exploits the non-Gaussianity (higher-order statistics) from data \citep{xie2020GIN,cai2019triad,xie2023generalized}.

\begin{Definition}[GIN Condition]\label{Definition-GIN}
Suppose all variables follow the linear non-Gaussian acyclic causal model. Let $\mathcal{Y}$, $\mathcal{Z}$ be two sets of random variables.  We say that $(\mathcal{Z},\mathcal{Y})$ follows the GIN condition if and only if $\omega^\intercal \mathcal{Y} \CI \mathcal{Z}$, where $\omega$ satisfies $\omega^\intercal \mathbb{E}(\mathcal{Y}\mathcal{Z}^\intercal) = \mathbf{0}$ and $\omega \neq \mathbf{0}$. 
\end{Definition}
By Darmois–Skitovich theorem \citep{darmois1953analyse,skitovitch1953property}\footnote{
Assume that ${V_1}$ and ${V_2}$ are  linear combinations of independent noise terms ${\varepsilon_i}(i = 1,...,n)$. If ${V_1}$ and ${V_2}$ are statistically independent, there are no common non-Gaussian noise terms between ${V_1}$ and ${V_2}$. See Theorem \ref{Theorem-DS} of Appendix. \ref{Appendix-Section-Proofs}.
}, GIN (that is linear transformation $\omega^\intercal \mathcal{Y} \CI \mathcal{Z}$) implies that $\omega^\intercal \mathcal{Y}$ shares no common non-Gaussian exogenous noise components with $\mathcal{Z}$. 

\textbf{\emph{A {Motivating} Example:}} To illustrate the intuitions behind it, we will begin by providing a straightforward example before presenting the theoretical results.
Let's consider the causal relationship $X_k \to Y$ shown in Figure \ref{Fig-equivalent-proximal-example}. Assume that the data are generated from a linear causal model with non-Gaussian error terms.
In the subgraph (a), $Z$ and $W$ are the valid NCE and NCO for the causal relationship $X_k \to Y$. We have that  $(\{X_k, Z\}, \{X_k, Y, W\})$ follows the GIN constraint, as explained below. The causal models of latent variables is $U =\varepsilon_{U}$, and 
$\{X_k,Y,W\}$ and $\{X_k,Z\}$ can then be represented as
\begin{nospaceflalign}\nonumber
        \underbrace{\left[\begin{matrix}
{{X_k}}\\
{{Y}}\\
{{W}}
\end{matrix}\right]}_{\mathcal{Y}} & =  {\left[\begin{matrix}
1 & {0}\\
\beta & {c+fd}\\
0 & {d}
\end{matrix}\right]}\left[\begin{matrix}
{{{{X_k}}}}\\
{{{{U}}}}
\end{matrix}\right] + \underbrace{\left[\begin{matrix}
{{0}}\\
{{f\varepsilon_{W}+\varepsilon_{{Y}}}}\\
{{\varepsilon_{{W}}}}
\end{matrix}\right]}_{\mathcal{E}_{\mathcal{Y}}},\\ \nonumber
\underbrace{\left[\begin{matrix}
{{X_k}}\\
{{Z}}
\end{matrix}\right]}_{\mathcal{Z}} &=  {\left[\begin{matrix}
1 &  0\\
0 &  a
\end{matrix}\right]}\left[\begin{matrix}
{{X_k}}\\
{{U}}
\end{matrix}\right] + \underbrace{\left[\begin{matrix}
{{0}}\\
{\varepsilon_{{Z}}}
\end{matrix}\right]}_{\mathcal{E}_{\mathcal{Z}}}.
\end{nospaceflalign}

According to the above equations, 
$\omega^\intercal \mathbb{E}(\mathcal{Y}\mathcal{Z}^{\intercal}) = \mathbf{0}$
$\Rightarrow$ $\omega = (d\beta, -d, c + df)^{\intercal}$. Then we can see $\omega^\intercal \mathcal{Y}= \omega^\intercal \mathcal{E}_{\mathcal{Y}}= c\varepsilon_{W} - d\varepsilon_{Y}$. By Darmois–Skitovich theorem, $\omega^\intercal \mathcal{Y}$ is independent of $\mathcal{Z}$ because there is no common non-Gaussian noise terms between $c\varepsilon_{W} - d\varepsilon_{Y}$ and $\mathcal{Z}$ (including noise terms $\varepsilon_{U}$ and $\varepsilon_{Z}$).
That is to say, $(\{X_k, Z\}, \{X_k, Y, W\})$ follows the GIN constraint.

Next, we discuss the subgraph (b), where $Z$ and $W$ are the invalid NCE and NCO, respectively, concerning the causal relationship $X_k \to Y$. $\{X_k,Y,W\}$ and $\{X_k,Z\}$ can then be represented as
\begin{nospaceflalign}\nonumber
    \underbrace{\left[\begin{matrix}
{{X_k}}\\
{{Y}}\\
{{W}}
\end{matrix}\right]}_{\mathcal{Y}} & =  {\left[\begin{matrix}
1 & {0}\\
\beta & {c+fd + ag}\\
0 & {d}
\end{matrix}\right]}\left[\begin{matrix}
{{{{X_k}}}}\\
{{{{U}}}}
\end{matrix}\right] + \underbrace{\left[\begin{matrix}
{{0}}\\
{{f\varepsilon_{W}+ a\varepsilon_{Z}+ \varepsilon_{{Y}}}}\\
{{\varepsilon_{{W}}}}
\end{matrix}\right]}_{\mathcal{E}_{\mathcal{Y}}},\\ \nonumber
\underbrace{\left[\begin{matrix}
{{X_k}}\\
{{Z}}
\end{matrix}\right]}_{\mathcal{Z}} &=  {\left[\begin{matrix}
1 &  0\\
0 &  a
\end{matrix}\right]}\left[\begin{matrix}
{{X_k}}\\
{{U}}
\end{matrix}\right] + \underbrace{\left[\begin{matrix}
{{0}}\\
{\varepsilon_{{Z}}}
\end{matrix}\right]}_{\mathcal{E}_{\mathcal{Z}}}.
\end{nospaceflalign}

We have $\omega^\intercal \mathcal{Y}$ is \emph{dependent} of $\mathcal{Z}$ because there exists common non-Gaussian noise terms $\varepsilon_Z$ between $\omega^\intercal \mathcal{Y}$ and $\mathcal{Z}$, no matter $\omega^\intercal \mathbb{E}(\mathcal{Y}\mathcal{Z}^{\intercal}) = \mathbf{0}$ or not.
That is to say, $(\{X_k, Z\}, \{X_k, Y, W\})$ violates the GIN constraint.
Assuming the distribution is rank-faithful to the graph, the above facts show that the lack of edge $Z \to Y$, i.e., the variable of NCE does not causally affect the primary outcome, has a testable implication with the help of non-Gaussianity. For further details regarding the above example, please refer to the Appendix. \ref{Appendix-Higher-Order}.

We now demonstrate that if valid NCE and NCO relative to a causal relationship $X_k \to Y$ exist (Assumption \ref{Ass-minimal} holds), we can identify them using GIN constraints.

\begin{Lemma}\label{Lemma-rule3}
Assume that the input data $\mathbf{X}$ and $Y$ strictly follow the Equation \ref{Eq-model-defi} and the rank-faithfulness holds. If Assumptions \ref{Ass-minimal} and \ref{Ass-non-Gaussianity} hold, then the underlying NCE and NCO relative to the causal relationship $X_k \to Y$ can be identified by using the following rule.
\begin{itemize}
\vspace{-4mm}
	\item[$\mathcal{R}3$.] Let $\mathbf{A}$ and $\mathbf{B}$ be two disjoint subsets of $\mathbf{X}$, where $|\mathbf{A}|=q$ and $|\mathbf{B}|=q$. If 1) $(\{X_k, \mathbf{A}\}, \{X_k, Y, \mathbf{B}\})$ follows the GIN constraint, and 2) $(\mathbf{B}, \{X_k, \mathbf{A}\})$ follows the GIN constraint, then $\mathbf{A}$ and $\mathbf{B}$ are valid NCE and NCO relative to $X_k \to Y$ respectively.
\end{itemize}
\vspace{-3mm}
\end{Lemma}

Note that the result of Lemma \ref{Lemma-rule3} does not strictly require adherence to Assumption 3-that is, not all noise variables need to follow a non-Gaussian distribution. For further discussion, please see Appendix \ref{Appendix-Section-Assumption-Lemma3}.

\begin{Example-set}[$\mathcal{R}3$]
	Consider the causal diagram shown in Figure \ref{Fig-model-example}. We consider the causal relationship $X_2 \to Y$. Assume that the data are generated from a linear non-Gaussian acyclic causal model. Let $\mathbf{A}=\{X_1\}$, and $\mathbf{B}=\{X_6\}$. We check $\mathcal{R}3$ and obtain that $(\{X_2, X_1\}, \{X_2, Y, X_6\})$ follows the GIN constraint, and 2) that $(X_6, \{X_2, X_1\})$ follows the GIN constraint. These facts imply that $X_1$ and $X_6$ are valid NCE and NCO relative to $X_2 \to Y$ respectively.
\end{Example-set}

Based on Lemma \ref{Lemma-rule3}, we present the identifiability of NCE and NCO in terms of GIN constraints.

\begin{Theorem}[Identifiability of NCE and NCO with GIN Constraints]\label{Theo-NCE-GIN}
Assume that the input data $\mathbf{X}$ and $Y$ strictly follow the Equation \ref{Eq-model-defi} and the rank-faithfulness holds. Furthermore, assume Assumption \ref{Ass-non-Gaussianity} holds. 
Then the underlying NCE and NCO relative to the causal relationship $X_k \to Y$ can be identified if Assumption \ref{Ass-minimal} holds.
\end{Theorem}

\vspace{-2mm}
\subsection{Algorithm}\label{Section-Algorithm-GIN}
In this section, we will leverage the above theoretical results and propose another data-driven method called Proxy-GIN to estimate the total causal effects of treatment $X_k \in \mathbf{X}$ on outcome $Y$:
\vspace{-1.5mm}
\rule{8.2cm}{1.2pt}
Proxy-GIN algorithm
\begin{description}[noitemsep,topsep=0pt,leftmargin=25pt]
    \item[1.] Given a $p$-dimensional treatments $\mathbf{X}$, outcome $Y$, the number of unmeasured confounders $q$. Initialize the sets of NCE, and NCO, causal effect, as $\mathcal{NCE}$, $\mathcal{NCO}$, and $\mathcal{C}$, respectively, with an empty set, i.e., $\mathcal{NCE} \coloneqq \emptyset$, $\mathcal{NCO} \coloneqq \emptyset$, and $\mathcal{C} \coloneqq \emptyset$.
    \item[2.]  Find valid NCE and NCO of unmeasured confounder $\mathbf{U}$ relative to per causal relationship $X_k \to Y$ according to Lemmas \ref{Lemma-rule3}.
    \item[3.] Estimate the corresponding unbiased causal effect by Proposition \ref{Pro-Multi-Proxy-Estimator} given NCE and NCO for per causal relationship $X_k \to Y$. Otherwise, output a value (NA) indicating the lack of valid NCE and NCO for this causal relationship $X_k \to Y$ to obtain the unbiased causal effect.
    \vspace{-2mm}
\end{description}
\vspace{-2mm}
\rule{8.2cm}{1.2pt}

The specific details of algorithm execution are provided in the Appendix \ref{Appendix-Algorithm-GIN}.

We now show that, in the large sample limit, for the causal relationships of interest, our algorithm outputs the true proxy variables and the unbiased estimation of causal effect.

\begin{Theorem}[Correctness]\label{Theorem-Corre-GIN}
Assume that the input data $\mathbf{X}$ and $Y$ strictly follow the Equation \ref{Eq-model-defi} and the rank-faithfulness holds. Furthermore, assume that Assumption \ref{Ass-non-Gaussianity} holds. Given infinite samples, the Proxy-GIN algorithm outputs the true causal effect $\mathcal{C}$ correctly.
\end{Theorem}
\vspace{-2mm}

We finally analyze the complexity of the Proxy-GIN algorithm. Let $q$ be the number of latent confounders, and $p$ be the number of treatments. The dominant part is to check $\mathcal{R}_3$ of Lemma \ref{Lemma-rule3} with worst-case complexity is $\mathcal{O}\left(\frac{p!}{(q+1)! \cdot (q+1)! \cdot (p-2q-3)!}\right)$. Hence, the worst case complexity of the Proxy-GIN algorithm is $\mathcal{O}\left(\frac{p!}{(q+1)! \cdot (q+1)! \cdot (p-2q-3)!}\right)$.

\begin{figure*}[t]
	\setlength{\abovecaptionskip}{0pt}
	\setlength{\belowcaptionskip}{-6pt}
	\begin{center}
		\includegraphics[height=0.20\textwidth,width=0.9\textwidth]{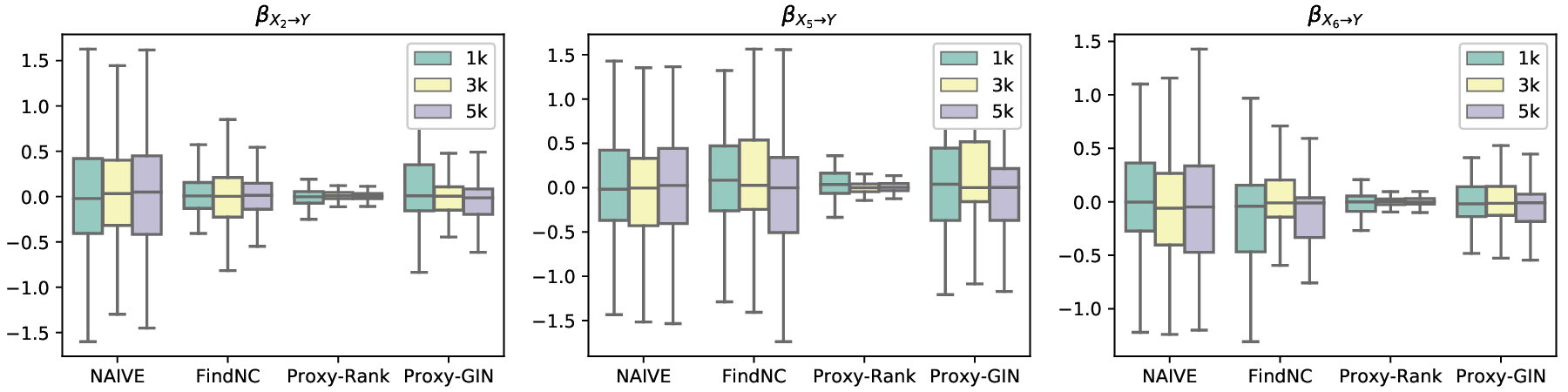}
		\vspace{-2mm}
		\caption{Performance of NAIVE, FindNC, Proxy-Rank, and Proxy-GIN on the Gaussian case.}
		\label{Fig-simulation-Gaussian-results} 
	\end{center}
	\vspace{-3mm}
\end{figure*}
\begin{figure*}[htp]
	\setlength{\abovecaptionskip}{0pt}
	\setlength{\belowcaptionskip}{-6pt}
	\begin{center}
		\includegraphics[height=0.20\textwidth,width=0.9\textwidth]{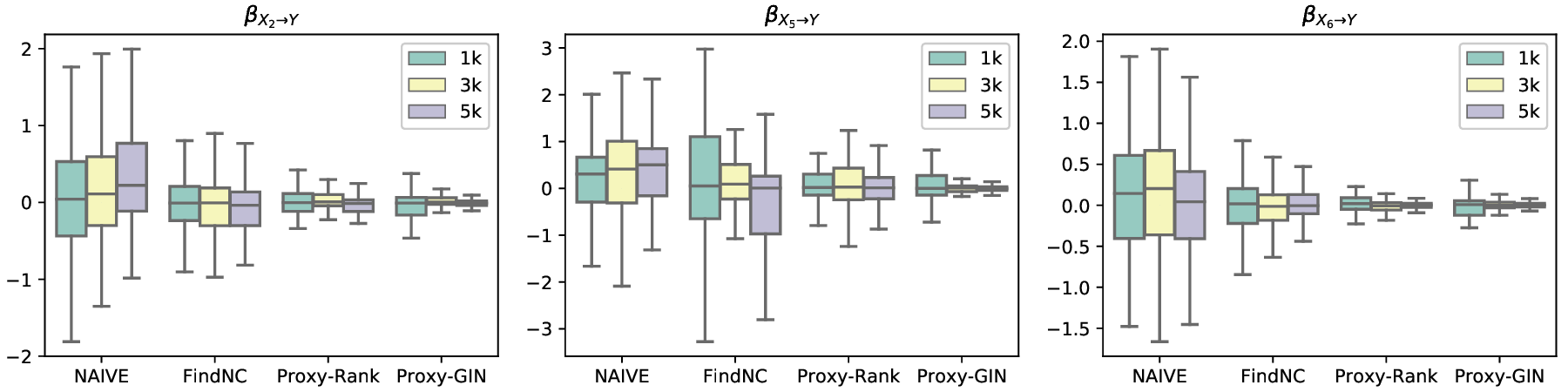}
		\vspace{-2mm}
		\caption{Performance of NAIVE, FindNC, Proxy-Rank, and Proxy-GIN on the Non-Gaussian case.}
		\label{Fig-simulation-non-Gaussian-results} 
	\end{center}
	\vspace{-4mm}
\end{figure*}

\vspace{-2mm}
\section{Experimental Results on Synthetic Data}
In this section, we evaluate the performance of the proposed methods in estimating causal effects from synthetic data. We here consider the following two typical settings:
    \textbf{Gaussian case:} The data are generated according to the causal graph in Figure \ref{Fig-model-example}, with the noise terms being generated from standard normal distributions; 
    \textbf{Non-Gaussian case}: The data are generated according to the graph obtained by removing variable $X_3$ from Figure \ref{Fig-model-example}, with the noise terms being generated from standard exponential distributions.
In three cases, the connected coefficient $\beta_{k}$ is sampled from a uniform distribution between $[-1,1]$.
Note that in the \emph{Gaussian case}, either Assumption \ref{Ass_nc_rank1} or \ref{Ass_nc_rank2} holds for the causal relationships $X_2 \to Y$, $X_5 \to Y$, and $X_6 \to Y$. However, in the \emph{non-Gaussian case}, both Assumption \ref{Ass_nc_rank1} and \ref{Ass_nc_rank2} are violated for the causal relationships $X_2 \to Y$ and $X_5 \to Y$, but Assumption \ref{Ass-minimal} holds. As a result, we will focus on these three causal relationships in this context.

The methods we compare against are: 1) NAIVE, the least-squares regression coefficient of $Y$ on $X_k \in \mathbf{X}$; 2) FindNC, the algorithm 1 of in~\cite{kummerfeld2022data} + standard confounder proxy estimator; 3) Proxy-Rank, our method, using T. W. Anderson's canonical correlation-based rank test  \citep{Anderson1984} to evaluate the Rank constraint, and 4) Proxy-GIN, our method, using HSIC-based independence test~\citep{zhang2018large} to evaluate the GIN constraint, due to the non-Gaussianity of the data. For the sake of comparison, if the algorithm fails to find the valid NCE and NCO for a causal relationship, we will randomly select variables as NCE and NCO to estimate the causal effect of interest.
Each experiment was repeated 100 times with randomly generated data and the results were averaged.  The sample size is selected from $\{1,000 (1k), 3,000 (3k), 5,000 (5k)\}$. The source code is in the Supplementary file.

Figures \ref{Fig-simulation-Gaussian-results} $\sim$ \ref{Fig-simulation-non-Gaussian-results} summarize the bias of the estimators of each parameter. As expected, our proposed Proxy-Rank algorithm almost outperforms other methods
(with little bias for all causal effects) in all two cases (except for $\beta_{X_2 \to Y}$ and $\beta_{X_5 \to Y}$ in the Non-Gaussian case), with all sample sizes.
The reason that $\beta_{X_2 \to Y}$ and $\beta_{X_5 \to Y}$ cannot be consistently estimated by the Proxy-Rank algorithm is that, in the Non-Gaussian case, there are no valid NCE and NCO for the causal relationship $X_2 \to Y$ and $X_5 \to Y$ in the ground-truth graph.
The Proxy-GIN algorithm outperforms the NAIVE algorithm and the FindNC algorithm in non-Gaussian cases, with all sample sizes, which verifies the correctness of the Proxy-GIN algorithm.
The the NAIVE algorithm and FindNC algorithm are expected to perform poorly, since there exists unmeasured confounder $U$ and the FindNC algorithm needs three Quadruple-disconnected NCs for per causal relationships. More experimental results are provided in Appendix \ref{Appendix-Section-Simulation}.

\vspace{-5mm}
\section{Experimental Results on Real-world Data}
In this section, we apply the proposed methods to analyze the causal effects of gene expressions on the body weight of F2 mice using the mouse obesity dataset as described by \citet{wang2006genetic}. The dataset we used comprises 17 gene expressions that are known to potentially influence mouse weight, as reported by \citet{lin2015regularization}. Additionally, it includes body weight as the outcome variable and data collected from 227 mice.
As discussed in \citet{miao2022identifying}, gene expression studies like this one may encounter unmeasured confounding issues stemming from batch effects or unobserved phenotypes. 

Following the analysis conducted by Miao et al., we assume that there is only one latent variable underlying the common influence, and the data generation mechanism adheres to a linear causal model. We observed that the majority of our findings align with those presented by \citet{miao2022identifying}. For instance, the gene expressions \emph{Gstm2}, \emph{Sirpa}, and \emph{2010002N04Rik} exhibit positive and significant effects on body weight, whereas the gene expression \emph{Dscam} demonstrates a negative impact on body weight. Detailed results and analysis are included in Appendix \ref{Appendix-Application}.

\vspace{-2mm}
\section{Conclusion}
This paper focuses on the identifiability conditions for selecting proxy variables for unmeasured confounders in observational data.
Initially, we introduce an extended proxy variable estimator to handle multiple unmeasured confounders between treatments and outcomes. Subsequently, we provide two specific identifiability conditions based on second-order and higher-order statistics. 
Additionally, the paper proposes two efficient algorithms for selecting proxy variables, utilizing Rank-deficiency and GIN properties, with their effectiveness substantiated by experimental results.

\section*{Acknowledgements}
FX would like to acknowledge the support by the Natural Science Foundation of China (62306019) and the Disciplinary funding of Beijing Technology and Business University (STKY202302). FX would like to thank Kun Zhang for his valuable insights and discussions during the initial phase of the research. ZM's research was supported by the China Scholarship Council (CSC). We appreciate the comments from anonymous reviewers, which greatly helped to improve the paper.

\section*{Impact Statement}
Learning causal effects is essential throughout the data-driven sciences and has attracted much attention.
Our research focuses on estimating causal effects where proxy variables exist to help adjust for unmeasured confounders in fields like social sciences, economics, public health, and neuroscience.
We assess the impact of our work in the context of these fields.
However, the applicability of existing methods is often limited in practice, as the validity of proxy variables typically relies on background knowledge.
Notable merits of our work include providing a practical method for selecting proxy variables and estimating unbiased causal effects from purely observational data.
\nocite{langley00}

\bibliography{example_paper}

\begin{thebibliography}{85}
\providecommand{\natexlab}[1]{#1}
\providecommand{\url}[1]{\texttt{#1}}
\expandafter\ifx\csname urlstyle\endcsname\relax
  \providecommand{\doi}[1]{doi: #1}\else
  \providecommand{\doi}{doi: \begingroup \urlstyle{rm}\Url}\fi

\bibitem[Anderson(1984)]{Anderson1984}
Anderson, T.~W.
\newblock \emph{An Introduction to Multivariate Statistical Analysis. 2nd ed}.
\newblock John Wiley \& Sons, 1984.

\bibitem[Bai \& Ng(2002)Bai and Ng]{bai2002determining}
Bai, J. and Ng, S.
\newblock Determining the number of factors in approximate factor models.
\newblock \emph{Econometrica}, 70\penalty0 (1):\penalty0 191--221, 2002.

\bibitem[Baiocchi et~al.(2014)Baiocchi, Cheng, and
  Small]{baiocchi2014instrumental}
Baiocchi, M., Cheng, J., and Small, D.~S.
\newblock Instrumental variable methods for causal inference.
\newblock \emph{Statistics in Medicine}, 33:\penalty0 2297--2340, 2014.

\bibitem[Bollen(1989)]{bollen1989structural}
Bollen, K.~A.
\newblock Structural equations with latent variables wiley.
\newblock \emph{New York}, 1989.

\bibitem[Bowden \& Turkington(1990)Bowden and
  Turkington]{bowden1990instrumental}
Bowden, R.~J. and Turkington, D.~A.
\newblock \emph{Instrumental variables}.
\newblock Number~8. Cambridge university press, 1990.

\bibitem[Cai et~al.(2019)Cai, Xie, Glymour, Hao, and Zhang]{cai2019triad}
Cai, R., Xie, F., Glymour, C., Hao, Z., and Zhang, K.
\newblock Triad constraints for learning causal structure of latent variables.
\newblock In \emph{Advances in Neural Information Processing Systems}, pp.\
  12863--12872, 2019.

\bibitem[Cheng et~al.(2022)Cheng, Li, Liu, Yu, Le, and Liu]{cheng2022toward}
Cheng, D., Li, J., Liu, L., Yu, K., Le, T.~D., and Liu, J.
\newblock Toward unique and unbiased causal effect estimation from data with
  hidden variables.
\newblock \emph{IEEE Transactions on Neural Networks and Learning Systems},
  2022.

\bibitem[Cheng et~al.(2023)Cheng, Li, Liu, Yu, Le, and
  Liu]{cheng2023discovering}
Cheng, D., Li, J., Liu, L., Yu, K., Le, T.~D., and Liu, J.
\newblock Discovering ancestral instrumental variables for causal inference
  from observational data.
\newblock \emph{IEEE Transactions on Neural Networks and Learning Systems},
  2023.

\bibitem[Colombo et~al.(2012)Colombo, Maathuis, Kalisch, and
  Richardson]{colombo2012learning}
Colombo, D., Maathuis, M.~H., Kalisch, M., and Richardson, T.~S.
\newblock Learning high-dimensional directed acyclic graphs with latent and
  selection variables.
\newblock \emph{The Annals of Statistics}, pp.\  294--321, 2012.

\bibitem[Cram\'{e}r(1962)]{Cramer62}
Cram\'{e}r, H.
\newblock \emph{Random variables and probability distributions}.
\newblock Cambridge University Press, Cambridge, 2nd edition, 1962.

\bibitem[D'Amour(2019b)]{damour2019comment}
D'Amour, A.
\newblock Comment: Reflections on the deconfounder.
\newblock \emph{Journal of the American Statistical Association}, 114:\penalty0
  1597--1601, 2019b.

\bibitem[Darmois(1953)]{darmois1953analyse}
Darmois, G.
\newblock Analyse g{\'e}n{\'e}rale des liaisons stochastiques: etude
  particuli{\`e}re de l'analyse factorielle lin{\'e}aire.
\newblock \emph{Revue de l'Institut international de statistique}, pp.\  2--8,
  1953.

\bibitem[Dawid(1979)]{dawid1979conditional}
Dawid, A.~P.
\newblock Conditional independence in statistical theory.
\newblock \emph{Journal of the Royal Statistical Society: Series B
  (Methodological)}, 41\penalty0 (1):\penalty0 1--15, 1979.

\bibitem[de~Luna et~al.(2017)de~Luna, Fowler, and Johansson]{de2017proxy}
de~Luna, X., Fowler, P., and Johansson, P.
\newblock Proxy variables and nonparametric identification of causal effects.
\newblock \emph{Economics Letters}, 150:\penalty0 152--154, 2017.

\bibitem[Draisma et~al.(2013)Draisma, Sullivant, and
  Talaska]{draisma2013positivity}
Draisma, J., Sullivant, S., and Talaska, K.
\newblock Positivity for gaussian graphical models.
\newblock \emph{Advances in Applied Mathematics}, 50\penalty0 (5):\penalty0
  661--674, 2013.

\bibitem[Drton et~al.(2020)Drton, Robeva, and Weihs]{drton2020nested}
Drton, M., Robeva, E., and Weihs, L.
\newblock Nested covariance determinants and restricted trek separation in
  gaussian graphical models.
\newblock \emph{Bernoulli}, 26\penalty0 (4), 2020.

\bibitem[Fern{\'a}ndez-Lor{\'\i}a \& Provost(2022)Fern{\'a}ndez-Lor{\'\i}a and
  Provost]{fernandez2022causal}
Fern{\'a}ndez-Lor{\'\i}a, C. and Provost, F.
\newblock Causal classification: Treatment effect estimation vs. outcome
  prediction.
\newblock \emph{The Journal of Machine Learning Research}, 23\penalty0
  (1):\penalty0 2573--2607, 2022.

\bibitem[Goldberger(1972)]{goldberger1972structural}
Goldberger, A.~S.
\newblock Structural equation methods in the social sciences.
\newblock \emph{Econometrica: Journal of the Econometric Society}, pp.\
  979--1001, 1972.

\bibitem[Gunsilius(2021)]{gunsilius2021nontestability}
Gunsilius, F.~F.
\newblock Nontestability of instrument validity under continuous treatments.
\newblock \emph{Biometrika}, 108\penalty0 (4):\penalty0 989--995, 2021.

\bibitem[Hern{\'a}n \& Robins(2006{\natexlab{a}})Hern{\'a}n and
  Robins]{hernan2006estimating}
Hern{\'a}n, M.~A. and Robins, J.~M.
\newblock Estimating causal effects from epidemiological data.
\newblock \emph{Journal of Epidemiology \& Community Health}, 60\penalty0
  (7):\penalty0 578--586, 2006{\natexlab{a}}.

\bibitem[Hern{\'a}n \& Robins(2006{\natexlab{b}})Hern{\'a}n and
  Robins]{hernan2006instruments}
Hern{\'a}n, M.~A. and Robins, J.~M.
\newblock Instruments for causal inference: an epidemiologist's dream?
\newblock \emph{Epidemiology}, 17\penalty0 (4):\penalty0 360--372,
  2006{\natexlab{b}}.

\bibitem[Hoyer et~al.(2009)Hoyer, Janzing, Mooij, Peters, and
  Sch{\"o}lkopf]{hoyer2009ANM}
Hoyer, P.~O., Janzing, D., Mooij, J.~M., Peters, J., and Sch{\"o}lkopf, B.
\newblock Nonlinear causal discovery with additive noise models.
\newblock In \emph{Advances in neural information processing systems}, pp.\
  689--696, 2009.

\bibitem[Huang et~al.(2022)Huang, Low, Xie, Glymour, and
  Zhang]{huang2022latent}
Huang, B., Low, C. J.~H., Xie, F., Glymour, C., and Zhang, K.
\newblock Latent hierarchical causal structure discovery with rank constraints.
\newblock \emph{Advances in Neural Information Processing Systems},
  35:\penalty0 5549--5561, 2022.

\bibitem[Hyv{\"a}rinen et~al.(2004)Hyv{\"a}rinen, Karhunen, and
  Oja]{hyvarinen2004independent}
Hyv{\"a}rinen, A., Karhunen, J., and Oja, E.
\newblock \emph{Independent component analysis}, volume~46.
\newblock John Wiley \& Sons, 2004.

\bibitem[Hyv{\"a}rinen et~al.(2010)Hyv{\"a}rinen, Zhang, Shimizu, and
  Hoyer]{hyvarinen2010estimation}
Hyv{\"a}rinen, A., Zhang, K., Shimizu, S., and Hoyer, P.~O.
\newblock Estimation of a structural vector autoregression model using
  non-gaussianity.
\newblock \emph{Journal of Machine Learning Research}, 11\penalty0 (5), 2010.

\bibitem[Imbens(2014)]{Imbens2014IV}
Imbens, G.~W.
\newblock Instrumental variables: An econometrician's perspective.
\newblock \emph{Statistical Science}, 29\penalty0 (3):\penalty0 323--358, 2014.

\bibitem[Imbens \& Rubin(2015)Imbens and Rubin]{imbens2015causal}
Imbens, G.~W. and Rubin, D.~B.
\newblock Causal inference for statistics, social, and biomedical sciences: An
  introduction.
\newblock \emph{Cambridge University Press}, 2015.

\bibitem[Kang et~al.(2016)Kang, Zhang, Cai, and Small]{kang2016instrumental}
Kang, H., Zhang, A., Cai, T.~T., and Small, D.~S.
\newblock Instrumental variables estimation with some invalid instruments and
  its application to \uppercase{M}endelian randomization.
\newblock \emph{Journal of the American statistical Association}, 111:\penalty0
  132--144, 2016.

\bibitem[Kummerfeld \& Ramsey(2016)Kummerfeld and Ramsey]{kummerfeld2016causal}
Kummerfeld, E. and Ramsey, J.
\newblock Causal clustering for 1-factor measurement models.
\newblock In \emph{Proceedings of the 22nd ACM SIGKDD international conference
  on knowledge discovery and data mining}, pp.\  1655--1664, 2016.

\bibitem[Kummerfeld et~al.(2022)Kummerfeld, Lim, and Shi]{kummerfeld2022data}
Kummerfeld, E., Lim, J., and Shi, X.
\newblock Data-driven automated negative control estimation (dance): Search
  for, validation of, and causal inference with negative controls.
\newblock \emph{arXiv preprint arXiv:2210.00528}, 2022.

\bibitem[Kuroki \& Cai(2005)Kuroki and Cai]{kuroki2005instrumental}
Kuroki, M. and Cai, Z.
\newblock Instrumental variable tests for directed acyclic graph models.
\newblock In \emph{International Workshop on Artificial Intelligence and
  Statistics}, pp.\  190--197. PMLR, 2005.

\bibitem[Kuroki \& Pearl(2014)Kuroki and Pearl]{kuroki2014measurement}
Kuroki, M. and Pearl, J.
\newblock Measurement bias and effect restoration in causal inference.
\newblock \emph{Biometrika}, 101:\penalty0 423--437, 2014.

\bibitem[Langley(2000)]{langley00}
Langley, P.
\newblock Crafting papers on machine learning.
\newblock In Langley, P. (ed.), \emph{Proceedings of the 17th International
  Conference on Machine Learning (ICML 2000)}, pp.\  1207--1216, Stanford, CA,
  2000. Morgan Kaufmann.

\bibitem[Lin et~al.(2015)Lin, Feng, and Li]{lin2015regularization}
Lin, W., Feng, R., and Li, H.
\newblock Regularization methods for high-dimensional instrumental variables
  regression with an application to genetical genomics.
\newblock \emph{Journal of the American Statistical Association}, 110\penalty0
  (509):\penalty0 270--288, 2015.

\bibitem[Lipsitch et~al.(2010)Lipsitch, Tchetgen, and
  Cohen]{lipsitch2010negative}
Lipsitch, M., Tchetgen, E.~T., and Cohen, T.
\newblock Negative controls: a tool for detecting confounding and bias in
  observational studies.
\newblock \emph{Epidemiology (Cambridge, Mass.)}, 21\penalty0 (3):\penalty0
  383, 2010.

\bibitem[Mastouri et~al.(2021)Mastouri, Zhu, Gultchin, Korba, Silva, Kusner,
  Gretton, and Muandet]{mastouri2021proximal}
Mastouri, A., Zhu, Y., Gultchin, L., Korba, A., Silva, R., Kusner, M., Gretton,
  A., and Muandet, K.
\newblock Proximal causal learning with kernels: Two-stage estimation and
  moment restriction.
\newblock In \emph{International Conference on Machine Learning}, pp.\
  7512--7523. PMLR, 2021.

\bibitem[Miao et~al.(2016)Miao, Ding, and Geng]{miao2016identifiability}
Miao, W., Ding, P., and Geng, Z.
\newblock Identifiability of normal and normal mixture models with nonignorable
  missing data.
\newblock \emph{Journal of the American Statistical Association}, 111:\penalty0
  1673--1683, 2016.

\bibitem[Miao et~al.(2018{\natexlab{a}})Miao, Geng, and
  Tchetgen~Tchetgen]{miao2018proxy}
Miao, W., Geng, Z., and Tchetgen~Tchetgen, E.
\newblock Identifying causal effects with proxy variables of an unmeasured
  confounder.
\newblock \emph{Biometrika}, 105:\penalty0 987--993., 2018{\natexlab{a}}.

\bibitem[Miao et~al.(2018{\natexlab{b}})Miao, Shi, and
  Tchetgen]{miao2018confounding}
Miao, W., Shi, X., and Tchetgen, E.~T.
\newblock A confounding bridge approach for double negative control inference
  on causal effects.
\newblock \emph{arXiv preprint arXiv:1808.04945}, 2018{\natexlab{b}}.

\bibitem[Miao et~al.(2022)Miao, Hu, Ogburn, and Zhou]{miao2022identifying}
Miao, W., Hu, W., Ogburn, E.~L., and Zhou, X.-H.
\newblock Identifying effects of multiple treatments in the presence of
  unmeasured confounding.
\newblock \emph{Journal of the American Statistical Association}, pp.\  1--15,
  2022.

\bibitem[Ogburn et~al.(2019)Ogburn, Shpitser, and Tchetgen]{ogburn2019comment}
Ogburn, E.~L., Shpitser, I., and Tchetgen, E. J.~T.
\newblock Comment on ``\protect{\uppercase {b}}lessings of multiple causes".
\newblock \emph{Journal of the American Statistical Association}, 114:\penalty0
  1611--1615, 2019.

\bibitem[Pearl(1988)]{pearl1988probabilistic}
Pearl, J.
\newblock \emph{Probabilistic reasoning in intelligent systems: networks of
  plausible inference}.
\newblock Morgan kaufmann, 1988.

\bibitem[Pearl(1995)]{pearl1995testability}
Pearl, J.
\newblock On the testability of causal models with latent and instrumental
  variables.
\newblock In \emph{Proceedings of the Eleventh conference on Uncertainty in
  artificial intelligence}, pp.\  435--443, 1995.

\bibitem[Pearl(2009)]{pearl2009causality}
Pearl, J.
\newblock \emph{Causality: Models, Reasoning, and Inference}.
\newblock Cambridge University Press, New York, 2nd edition, 2009.

\bibitem[Peters et~al.(2014)Peters, Mooij, D., and
  Sch{\"o}lkopf]{peters2014causal}
Peters, J., Mooij, J.~M., D., J., and Sch{\"o}lkopf, B.
\newblock Minimal nonlinear distortion principle for nonlinear independent
  component analysis.
\newblock \emph{Journal of Machine Learning Research}, 15:\penalty0 2009--2053,
  2014.

\bibitem[Peters et~al.(2017)Peters, Janzing, and
  Sch{\"o}lkopf]{peters2017elements}
Peters, J., Janzing, D., and Sch{\"o}lkopf, B.
\newblock \emph{\textit{Elements of Causal Inference}}.
\newblock MIT Press, 2017.

\bibitem[Rotnitzky \& Smucler(2020)Rotnitzky and
  Smucler]{rotnitzky2020efficient}
Rotnitzky, A. and Smucler, E.
\newblock Efficient adjustment sets for population average causal treatment
  effect estimation in graphical models.
\newblock \emph{The Journal of Machine Learning Research}, 21\penalty0
  (1):\penalty0 7642--7727, 2020.

\bibitem[Salehkaleybar et~al.(2020)Salehkaleybar, Ghassami, Kiyavash, and
  Zhang]{salehkaleybar2020learning}
Salehkaleybar, S., Ghassami, A., Kiyavash, N., and Zhang, K.
\newblock Learning linear non-gaussian causal models in the presence of latent
  variables.
\newblock \emph{The Journal of Machine Learning Research}, 21\penalty0
  (1):\penalty0 1436--1459, 2020.

\bibitem[Schneeberger(2019)]{schneeberger2019irx3}
Schneeberger, M.
\newblock Irx3, a new leader on obesity genetics.
\newblock \emph{EBioMedicine}, 39:\penalty0 19--20, 2019.

\bibitem[Shi et~al.(2020{\natexlab{a}})Shi, Miao, Nelson, and
  Tchetgen~Tchetgen]{shi2020multiply}
Shi, X., Miao, W., Nelson, J.~C., and Tchetgen~Tchetgen, E.~J.
\newblock Multiply robust causal inference with double-negative control
  adjustment for categorical unmeasured confounding.
\newblock \emph{Journal of the Royal Statistical Society: Series B (Statistical
  Methodology)}, 82\penalty0 (2):\penalty0 521--540, 2020{\natexlab{a}}.

\bibitem[Shi et~al.(2020{\natexlab{b}})Shi, Miao, and
  Tchetgen]{shi2020selective}
Shi, X., Miao, W., and Tchetgen, E.~T.
\newblock A selective review of negative control methods in epidemiology.
\newblock \emph{Current epidemiology reports}, 7:\penalty0 190--202,
  2020{\natexlab{b}}.

\bibitem[Shimizu(2022)]{shimizu2022statistical}
Shimizu, S.
\newblock \emph{Statistical Causal Discovery: LiNGAM Approach}.
\newblock Springer, 2022.

\bibitem[Shimizu et~al.(2006)Shimizu, Hoyer, Hyv{\"a}rinen, and
  Kerminen]{shimizu2006linear}
Shimizu, S., Hoyer, P.~O., Hyv{\"a}rinen, A., and Kerminen, A.
\newblock A linear non-{G}aussian acyclic model for causal discovery.
\newblock \emph{{Journal of Machine Learning Research}}, 7\penalty0
  (Oct):\penalty0 2003--2030, 2006.

\bibitem[Shimizu et~al.(2011)Shimizu, Inazumi, Sogawa, Hyv{\"a}rinen, Kawahara,
  Washio, Hoyer, and Bollen]{shimizu2011directlingam}
Shimizu, S., Inazumi, T., Sogawa, Y., Hyv{\"a}rinen, A., Kawahara, Y., Washio,
  T., Hoyer, P.~O., and Bollen, K.
\newblock Direct{L}i{NGAM}: A direct method for learning a linear
  non-{G}aussian structural equation model.
\newblock \emph{Journal of Machine Learning Research}, 12\penalty0
  (Apr):\penalty0 1225--1248, 2011.

\bibitem[Shpitser et~al.(2023)Shpitser, Wood-Doughty, and
  Tchetgen]{shpitser2023proximal}
Shpitser, I., Wood-Doughty, Z., and Tchetgen, E. J.~T.
\newblock The proximal id algorithm.
\newblock \emph{Journal of Machine Learning Research}, 23:\penalty0 1--46,
  2023.

\bibitem[Silva \& Shimizu(2017)Silva and Shimizu]{silva2017learning}
Silva, R. and Shimizu, S.
\newblock Learning instrumental variables with structural and non-gaussianity
  assumptions.
\newblock \emph{Journal of Machine Learning Research}, 18\penalty0
  (120):\penalty0 1--49, 2017.

\bibitem[Singh(2020)]{singh2020kernel}
Singh, R.
\newblock Kernel methods for unobserved confounding: Negative controls,
  proxies, and instruments.
\newblock \emph{arXiv preprint arXiv:2012.10315}, 2020.

\bibitem[Skitovitch(1953)]{skitovitch1953property}
Skitovitch, V.~P.
\newblock On a property of the normal distribution.
\newblock \emph{DAN SSSR}, 89:\penalty0 217--219, 1953.

\bibitem[Sofer et~al.(2016)Sofer, Richardson, Colicino, Schwartz, and
  Tchetgen]{sofer2016negative}
Sofer, T., Richardson, D.~B., Colicino, E., Schwartz, J., and Tchetgen, E.
  J.~T.
\newblock On negative outcome control of unobserved confounding as a
  generalization of difference-in-differences.
\newblock \emph{Statistical science: a review journal of the Institute of
  Mathematical Statistics}, 31\penalty0 (3):\penalty0 348, 2016.

\bibitem[Spearman(1928)]{spearman1928pearson}
Spearman, C.
\newblock Pearson's contribution to the theory of two factors.
\newblock \emph{British Journal of Psychology. General Section}, 19\penalty0
  (1):\penalty0 95--101, 1928.

\bibitem[Spirtes(2010)]{spirtes2010introduction}
Spirtes, P.
\newblock Introduction to causal inference.
\newblock \emph{Journal of Machine Learning Research}, 11\penalty0 (5), 2010.

\bibitem[Spirtes(2013)]{spirtes2013calculation-t-separation}
Spirtes, P.
\newblock Calculation of entailed rank constraints in partially non-linear and
  cyclic models.
\newblock In \emph{Proceedings of the Twenty-Ninth Conference on Uncertainty in
  Artificial Intelligence}, pp.\  606--615. AUAI Press, 2013.

\bibitem[Spirtes \& Zhang(2016)Spirtes and Zhang]{spirtes2016causal}
Spirtes, P. and Zhang, K.
\newblock Causal discovery and inference: concepts and recent methodological
  advances.
\newblock \emph{Applied Informatics}, 3\penalty0 (1):\penalty0 1--28, 2016.

\bibitem[Spirtes et~al.(1995)Spirtes, Meek, and Richardson]{spirtes1995causal}
Spirtes, P., Meek, C., and Richardson, T.
\newblock Causal inference in the presence of latent variables and selection
  bias.
\newblock In \emph{Proceedings of the Eleventh conference on Uncertainty in
  artificial intelligence (UAI)}, pp.\  499--506. Morgan Kaufmann Publishers
  Inc., 1995.

\bibitem[Spirtes et~al.(2000)Spirtes, Glymour, and
  Scheines]{spirtes2000causation}
Spirtes, P., Glymour, C.~N., and Scheines, R.
\newblock \emph{Causation, Prediction, and Search}.
\newblock MIT press, 2000.

\bibitem[Sullivant et~al.(2010)Sullivant, Talaska, and
  Draisma]{Sullivant-T-separation}
Sullivant, S., Talaska, K., and Draisma, J.
\newblock Trek separation for gaussian graphical models.
\newblock \emph{The Annals of Statistics}, 38\penalty0 (3):\penalty0
  1665--1685, 2010.

\bibitem[Tchetgen et~al.(2020)Tchetgen, Ying, Cui, Shi, and
  Miao]{tchetgen2020introduction}
Tchetgen, E. J.~T., Ying, A., Cui, Y., Shi, X., and Miao, W.
\newblock An introduction to proximal causal learning.
\newblock \emph{arXiv preprint arXiv:2009.10982}, 2020.

\bibitem[Van~der Vaart(2000)]{van2000asymptotic}
Van~der Vaart, A.~W.
\newblock \emph{Asymptotic statistics}, volume~3.
\newblock Cambridge university press, 2000.

\bibitem[Van Der~Zander et~al.(2019)Van Der~Zander, Li{\'s}kiewicz, and
  Textor]{van2019separators}
Van Der~Zander, B., Li{\'s}kiewicz, M., and Textor, J.
\newblock Separators and adjustment sets in causal graphs: Complete criteria
  and an algorithmic framework.
\newblock \emph{Artificial Intelligence}, 270:\penalty0 1--40, 2019.

\bibitem[Wang et~al.(2017)Wang, Zhao, Hastie, and Owen]{wang2017confounder}
Wang, J., Zhao, Q., Hastie, T., and Owen, A.~B.
\newblock Confounder adjustment in multiple hypothesis testing.
\newblock \emph{Annals of statistics}, 45\penalty0 (5):\penalty0 1863, 2017.

\bibitem[Wang et~al.(2006)Wang, Yehya, Schadt, Wang, Drake, and
  Lusis]{wang2006genetic}
Wang, S., Yehya, N., Schadt, E.~E., Wang, H., Drake, T.~A., and Lusis, A.~J.
\newblock Genetic and genomic analysis of a fat mass trait with complex
  inheritance reveals marked sex specificity.
\newblock \emph{PLoS Genetics}, 2:\penalty0 e15, 2006.

\bibitem[Wang \& Blei(2021)Wang and Blei]{wang2021proxy}
Wang, Y. and Blei, D.
\newblock A proxy variable view of shared confounding.
\newblock In \emph{International Conference on Machine Learning}, pp.\
  10697--10707. PMLR, 2021.

\bibitem[{Wang} \& {Blei}(2019){Wang} and {Blei}]{wang2019jasa}
{Wang}, Y. and {Blei}, D.~M.
\newblock {The Blessings of Multiple Causes}.
\newblock \emph{Journal of the American Statistical Association}, 114:\penalty0
  1574--1596, 2019.

\bibitem[Wang \& Drton(2020)Wang and Drton]{wang2020high}
Wang, Y.~S. and Drton, M.
\newblock High-dimensional causal discovery under non-gaussianity.
\newblock \emph{Biometrika}, 107\penalty0 (1):\penalty0 41--59, 2020.

\bibitem[Wheatcroft et~al.(2007)Wheatcroft, Kearney, Shah, Ezzat, Miell, Modo,
  Williams, Cawthorn, Medina-Gomez, Vidal-Puig, Sethi, and
  Crossey]{wheatcroft2007igf}
Wheatcroft, S.~B., Kearney, M.~T., Shah, A.~M., Ezzat, V.~A., Miell, J.~R.,
  Modo, M., Williams, S.~C., Cawthorn, W.~P., Medina-Gomez, G., Vidal-Puig, A.,
  Sethi, J.~K., and Crossey, P.~A.
\newblock Igf-binding protein-2 protects against the development of obesity and
  insulin resistance.
\newblock \emph{Diabetes}, 56:\penalty0 285--294, 2007.

\bibitem[Wright(1928)]{wright1928tariff}
Wright, P.~G.
\newblock \emph{Tariff on Animal and Vegetable Oils}.
\newblock Macmillan, New York, 1928.

\bibitem[Xie et~al.(2020)Xie, Cai, Huang, Glymour, Hao, and Zhang]{xie2020GIN}
Xie, F., Cai, R., Huang, B., Glymour, C., Hao, Z., and Zhang, K.
\newblock Generalized independent noise conditionfor estimating latent variable
  causal graphs.
\newblock In \emph{Advances in Neural Information Processing Systems}, pp.\
  14891--14902, 2020.

\bibitem[Xie et~al.(2022{\natexlab{a}})Xie, He, Geng, Chen, Hou, and
  Zhang]{xie2022testability}
Xie, F., He, Y., Geng, Z., Chen, Z., Hou, R., and Zhang, K.
\newblock Testability of instrumental variables in linear non-gaussian acyclic
  causal models.
\newblock \emph{Entropy}, 24\penalty0 (4):\penalty0 512, 2022{\natexlab{a}}.

\bibitem[Xie et~al.(2022{\natexlab{b}})Xie, Huang, Chen, He, Geng, and
  Zhang]{xie2022identification}
Xie, F., Huang, B., Chen, Z., He, Y., Geng, Z., and Zhang, K.
\newblock Identification of linear non-gaussian latent hierarchical structure.
\newblock In \emph{International Conference on Machine Learning}, pp.\
  24370--24387. PMLR, 2022{\natexlab{b}}.

\bibitem[Xie et~al.(2023)Xie, Huang, Chen, Cai, Glymour, Geng, and
  Zhang]{xie2023generalized}
Xie, F., Huang, B., Chen, Z., Cai, R., Glymour, C., Geng, Z., and Zhang, K.
\newblock Generalized independent noise condition for estimating causal
  structure with latent variables.
\newblock \emph{arXiv preprint arXiv:2308.06718}, 2023.

\bibitem[Xu et~al.(2021)Xu, Kanagawa, and Gretton]{xu2021deep}
Xu, L., Kanagawa, H., and Gretton, A.
\newblock Deep proxy causal learning and its application to confounded bandit
  policy evaluation.
\newblock \emph{Advances in Neural Information Processing Systems},
  34:\penalty0 26264--26275, 2021.

\bibitem[Zhang(2008)]{zhang2008completeness}
Zhang, J.
\newblock On the completeness of orientation rules for causal discovery in the
  presence of latent confounders and selection bias.
\newblock \emph{Artificial Intelligence}, 172\penalty0 (16-17):\penalty0
  1873--1896, 2008.

\bibitem[Zhang \& Hyv{\"a}rinen(2009)Zhang and Hyv{\"a}rinen]{zhang2009PNL}
Zhang, K. and Hyv{\"a}rinen, A.
\newblock On the identifiability of the post-nonlinear causal model.
\newblock In \emph{UAI}, pp.\  647--655. AUAI Press, 2009.

\bibitem[Zhang et~al.(2018)Zhang, Filippi, Gretton, and
  Sejdinovic]{zhang2018large}
Zhang, Q., Filippi, S., Gretton, A., and Sejdinovic, D.
\newblock Large-scale kernel methods for independence testing.
\newblock \emph{Statistics and Computing}, 28\penalty0 (1):\penalty0 113--130,
  2018.

\bibitem[Zhao et~al.(2022)Zhao, He, and Wang]{zhao2022learning}
Zhao, R., He, X., and Wang, J.
\newblock Learning linear non-gaussian directed acyclic graph with diverging
  number of nodes.
\newblock \emph{The Journal of Machine Learning Research}, 23\penalty0
  (1):\penalty0 12314--12347, 2022.

\end{thebibliography}
\bibliographystyle{icml2024}

\newpage
\appendix
\onecolumn

\section{More Details on Extended Proxy Variables Estimator}\label{Appendix-Extended-Proxy}
We provide simulation results to demonstrate that when the dimension of NCE and NCO is lower than the number of unmeasured confounders, the estimated causal effect becomes biased when using a standard proxy variable estimator (Proposition \ref{Pro-Proxy-Estimator}). In particular, we examine the causal relationship $X_k \to Y$ depicted in Figure \ref{Fig-more-simulation-graph}. We consider different quantities of unmeasured confounders, denoted as $q=1,2,3,4$. We employ the following estimators:
\begin{itemize}
    \item The \emph{Traditional Proxy Variables Estimator} described in Proposition \ref{Pro-Proxy-Estimator}, which is given by:
    \begin{align}\label{Eq-single-proximal-inference}
\beta_{X_k \to Y} & = \frac{\sigma_{X_kY}\sigma_{W_1Z_1}-\sigma_{X_kW_1}\sigma_{YZ_1}}{\sigma_{X_kX_k}\sigma_{W_1Z_1}-\sigma_{X_kW_1}\sigma_{X_kZ_1}}.
\end{align}
    \item  The \emph{Extended Proxy Variables Estimator} described in Proposition \ref{Pro-Multi-Proxy-Estimator}, which is expressed as:
    \begin{align}\label{Appendix-Eq-multiple-proximal-inference}
\beta_{X_k \to Y} & = \frac{\mathrm{det}(\boldsymbol{\Sigma}_{\{X_k \cup \mathbf{Z}\}, \{Y \cup \mathbf{W}\}})}{\mathrm{det}(\boldsymbol{\Sigma}_{\{X_k \cup \mathbf{Z}\}, \{X_k \cup \mathbf{W}\}})},
\end{align}
where $\mathbf{Z}=\{Z_1, \ldots, Z_q\}$ and $\mathbf{W}=\{W_1, \ldots, W_q\}$.
\end{itemize}

\begin{figure}[htp]
	\begin{center}
		\begin{tikzpicture}[scale=1.4, line width=0.5pt, inner sep=0.2mm, shorten >=.1pt, shorten <=.1pt]
        \draw [fill=blue!50,thick, fill opacity=0.5, draw=none] (1.0,0.8) ellipse [x radius=0.8cm, y radius=0.35cm];
        \draw [fill=red!50,thick, fill opacity=0.5, draw=none] (-0.25,0.0) ellipse [x radius=0.6cm, y radius=0.3cm];
        \draw [fill=red!50,thick, fill opacity=0.5, draw=none] (2.25,0.0) ellipse [x radius=0.6cm, y radius=0.3cm];
        \draw (1.0, 1.0) node(R21B) [] {{\footnotesize\,$\mathbf{U}$\,}};
		\draw (0.6, 0.8) node(U1) [circle, fill=gray!60, minimum size=0.5cm,draw] {{\footnotesize\,$U_1$\,}};
        \draw (1.0, 0.8) node(con1) [] {{\footnotesize\, ... \,}};
        \draw (1.4, 0.8) node(U2) [circle, fill=gray!60, minimum size=0.5cm,draw] {{\footnotesize\,$U_q$\,}};
		\draw (-0.0, 0.0) node(Z1) [] {{\footnotesize\,${Z}_{q}$\,}};
		\draw (-0.25, 0.0) node(con1) [] {{\footnotesize\, ... \,}};
		\draw (0.5, 0.0) node(Xk) [] {{\footnotesize\,${X}_k$\,}};
		\draw (2.25, 0.0) node(con1) [] {{\footnotesize\, ... \,}};
		\draw (1.5, 0.0) node(Y) [] {{\footnotesize\,$Y$\,}};
		\draw (2.0, 0.0) node(W1) [] {{\footnotesize\,$W_1$\,}};
        \draw (-0.5, 0.0) node(Z2) [] {{\footnotesize\,$Z_1$\,}};
        \draw (2.5, 0.0) node(W2) [] {{\footnotesize\,$W_q$\,}};
		\draw[-arcsq] (U1) -- (Xk) node[pos=0.5,sloped,above] {};
		\draw[-arcsq] (U1) -- (Y) node[pos=0.5,sloped,above] {};
		\draw[-arcsq] (U1) -- (Z1) node[pos=0.5,sloped,above] {};
		\draw[-arcsq] (U1) -- (W1) node[pos=0.5,sloped,above] {};
            \draw[-arcsq] (U1) -- (Z2) node[pos=0.5,sloped,above] {};
		\draw[-arcsq] (U1) -- (W2) node[pos=0.5,sloped,above] {};
        \draw[-arcsq] (U2) -- (Xk) node[pos=0.5,sloped,above] {};
		\draw[-arcsq] (U2) -- (Y) node[pos=0.5,sloped,above] {};
		\draw[-arcsq] (U2) -- (Z1) node[pos=0.5,sloped,above] {};
		\draw[-arcsq] (U2) -- (W1) node[pos=0.5,sloped,above] {};
         \draw[-arcsq] (U2) -- (Z2) node[pos=0.5,sloped,above] {};
		\draw[-arcsq] (U2) -- (W2) node[pos=0.5,sloped,above] {};
        \draw[-arcsq] (Xk) -- (Y) node[pos=0.5,sloped,above] {{\scriptsize\,$\beta$\,}}; 
		\draw (1.2, -0.3) node(con1) [] {{\footnotesize\,(c) \,}};
		\end{tikzpicture}
		
		\vspace{-0.1cm}
		\caption{Causal Diagram used in our simulation studies, where $q=1,2,3,4$.}
		\label{Fig-more-simulation-graph} 
	\end{center}
	\vspace{-0.3cm}  
\end{figure}
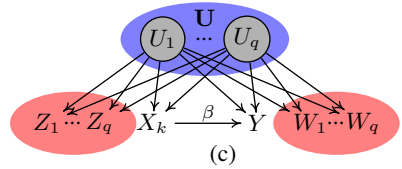

Figures \ref{Fig-simulation-mul-factor} summarizes the bias of the estimators of each parameter. As expected, our proposed \emph{Extended Proxy Variables Estimator} outperforms \emph{Traditional Proxy Variables Estimator} (with little bias for the causal effect of $X_k$ on $Y$) when the number of unmeasured confounders is greater or equal to 2.

\begin{figure}[htp]
  \setlength{\abovecaptionskip}{0pt}
	\setlength{\belowcaptionskip}{-6pt}
	\vspace{-0.1cm}
	\begin{center}
        \includegraphics[height=0.32\textwidth]{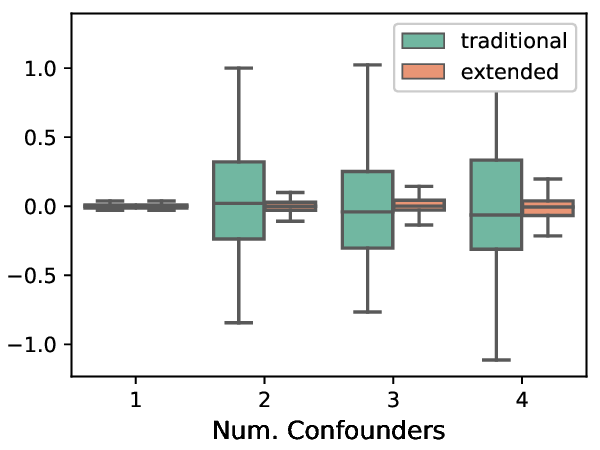}
		\caption{Performance of extended estimator and traditional estimator with varying numbers of unmeasured confounders.}
		\vspace{-0.3cm}
		\label{Fig-simulation-mul-factor} 
	\end{center}
\end{figure}

\section{More Details on the Motivating Example for Second-Order Statistics (in Section \ref{Subsection-Second-Order-Statistics})}

In this section, we will give the details of the motivating example described in Section \ref{Subsection-Second-Order-Statistics}.

\begin{figure}[H]
    \centering
	\begin{tikzpicture}[scale=0.9, line width=0.5pt, inner sep=0.2mm, shorten >=.1pt, shorten <=.1pt]
			\draw (3, 1.5) node(U) [circle, fill=gray!60, minimum size=0.5cm,draw] {{\footnotesize\,${U}$\,}};
			\draw (1, 0) node(X1) [] {{\footnotesize\,$X_1$\,}};
			\draw (2, 0) node(X2) [] {{\footnotesize\,$X_2$\,}};
			\draw (3, 0) node(X3) [] {{\footnotesize\,{$X_3$}\,}};
			\draw (4, 0) node(X4) [] {{\footnotesize\,{$X_4$}\,}};
			\draw (5, 0) node(X5) [] {{\footnotesize\,{$X_5$}\,}};
			\draw (6, 0) node(X6) [] {{\footnotesize\,{$X_6$}\,}};
			\draw (7, 0) node(Y) [circle, draw]  {{\footnotesize\,{$Y$}\,}};
			\draw[-arcsq] (U) -- (X1) node[pos=0.5,sloped,above] {{\scriptsize\,$c_{1}$\,}};  
			\draw[-arcsq] (U) -- (X2)node[pos=0.5,sloped,above] {{\scriptsize\,$c_{2}$\,}}; 
			\draw[-arcsq] (U) -- (X3) node[pos=0.5,sloped,above] {{\scriptsize\,$c_{3}$\,}};
			\draw[-arcsq] (U) -- (X4) node[pos=0.5,sloped,above] {{\scriptsize\,$c_{4}$\,}};
			\draw[-arcsq] (U) -- (X5) node[pos=0.5,sloped,above] {{\scriptsize\,$c_{5}$\,}};
			\draw[-arcsq] (U) -- (X6) node[pos=0.5,sloped,above] {{\scriptsize\,$c_{6}$\,}};
			\draw[-arcsq] (X1) -- (X2) node[pos=0.5,sloped,above] {{\scriptsize\,$b_{21}$\,}}; 
			\draw[-arcsq] (X4) -- (X5) node[pos=0.5,sloped,above] {{\scriptsize\,$b_{54}$\,}};
			%
			\draw[-arcsq] (U) -- (Y) node[pos=0.5,sloped,above] {{\scriptsize\,$\delta$\,}}; 
			\draw [-arcsq] (X2) edge[bend right=20] (Y);
                \draw (4.2,-0.4) node(label-ii) [] {{\scriptsize\,$\beta_2$\,}};
			\draw[-arcsq] (X5) -- (X6) node[pos=0.5,sloped,above] {{\scriptsize\,$b_{65}$\,}};
			\draw[-arcsq] (X6) -- (Y) node[pos=0.5,sloped,above] {{\scriptsize\,$\beta_{6}$\,}};
            \draw (4,-1) node(label-ii) [] {{\footnotesize\,(a)\,}};
		\end{tikzpicture}~~~~
    \begin{tikzpicture}[scale=0.9, line width=0.5pt, inner sep=0.2mm, shorten >=.1pt, shorten <=.1pt]
			\draw (3, 1.5) node(U) [circle, fill=gray!60, minimum size=0.5cm,draw] {{\footnotesize\,${U}$\,}};
			\draw (1, 0) node(X1) [] {{\footnotesize\,$X_1$\,}};
			\draw (2, 0) node(X2) [] {{\footnotesize\,$X_2$\,}};
			\draw (3, 0) node(X3) [] {{\footnotesize\,{$X_3$}\,}};
			\draw (4, 0) node(X4) [] {{\footnotesize\,{$X_4$}\,}};
			\draw (5, 0) node(X5) [] {{\footnotesize\,{$X_5$}\,}};
			\draw (6, 0) node(X6) [] {{\footnotesize\,{$X_6$}\,}};
			\draw (7, 0) node(Y) [circle, draw]  {{\footnotesize\,{$Y$}\,}};
			\draw[-arcsq] (U) -- (X1) node[pos=0.5,sloped,above] {{\scriptsize\,$c_{1}$\,}};  
			\draw[-arcsq] (U) -- (X2)node[pos=0.5,sloped,above] {{\scriptsize\,$c_{2}$\,}}; 
			\draw[-arcsq] (U) -- (X3) node[pos=0.5,sloped,above] {{\scriptsize\,$c_{3}$\,}};
			\draw[-arcsq] (U) -- (X4) node[pos=0.5,sloped,above] {{\scriptsize\,$c_{4}$\,}};
			\draw[-arcsq] (U) -- (X5) node[pos=0.5,sloped,above] {{\scriptsize\,$c_{5}$\,}};
			\draw[-arcsq] (U) -- (X6) node[pos=0.5,sloped,above] {{\scriptsize\,$c_{6}$\,}};
			\draw[-arcsq] (X1) -- (X2) node[pos=0.5,sloped,above] {};
			\draw[-arcsq] (X4) -- (X5) node[pos=0.5,sloped,above] {{\scriptsize\,$b_{54}$\,}};
			%
			\draw[-arcsq] (U) -- (Y) node[pos=0.5,sloped,above] {{\scriptsize\,$\delta$\,}}; 
			\draw [-arcsq] (X2) edge[bend right=20] (Y);
                \draw (4.2,-0.4) node(label-ii) [] {{\scriptsize\,$\beta_2$\,}};
			\draw[-arcsq] (X5) -- (X6) node[pos=0.5,sloped,above] {{\scriptsize\,$b_{65}$\,}};
			\draw[-arcsq] (X6) -- (Y) node[pos=0.5,sloped,above] {{\scriptsize\,$\beta_6$\,}};
			\draw [-arcsq,color=red,line width=0.8pt] (X1) edge[bend right=25] (Y);
                \draw (2.4,-0.4) node(label-ii) [] {{\scriptsize\,$\beta_1$\,}};
                \draw (4,-1) node(label-ii) [] {{\footnotesize\,(b)\,}};
		\end{tikzpicture}
		\caption{A simple causal graph involving 6 potential treatments and one outcome.}
		\vspace{-0.4cm}
		\label{Appendix-Fig-model-example}
\end{figure}
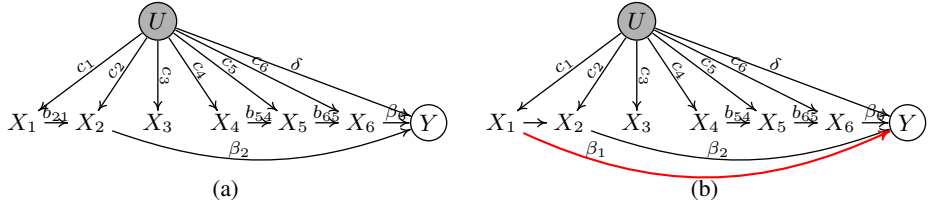

We assume that all random variables have mean zero for simplicity.

\textbf{\emph{Structure (a)}}: Since the variables strictly follow the linear causal model, we obtain
\begin{align}\nonumber
{U} & = \varepsilon _{U},\label{Eq1-latent} \\ \nonumber
{X_1} & = c_1 {U} + \varepsilon_{X_1} = c_1 \varepsilon_{U} + \varepsilon_{X_1},\\ \nonumber
{X_2} & = c_2 {U} + b_{21} X_1 + \varepsilon_{X_2} = (c_1 b_{21}+ c_2) \varepsilon_{U} + b_{21} \varepsilon_{X_1} + \varepsilon_{X_2},\\ \nonumber
{X_3} & = c_3 {U} + \varepsilon_{X_3} = c_3 \varepsilon_{U} + \varepsilon_{X_3},\\ \nonumber
{X_4} & = c_4 {U} + \varepsilon_{X_4} = c_4 \varepsilon_{U} + \varepsilon_{X_4},\\ \nonumber
{X_5} & = c_5 {U} + b_{54} X_4 + \varepsilon_{X_5} = (c_5 + c_4 b_{54}) \varepsilon_{U} + b_{54} \varepsilon_{X_4} + \varepsilon_{X_5},\\ \nonumber
{X_6} & = c_6 {U} + b_{65} X_5 + \varepsilon_{X_6} \nonumber \\ &= (c_6 + b_{65} c_5+ b_{65} c_4 b_{54}) \varepsilon_{U} + b_{65} b_{54} \varepsilon_{X_4} + b_{65} \varepsilon_{X_5} + \varepsilon_{X_6},\\ \nonumber
Y & = \delta {U} + \beta_6 X_6 + \beta_2 X_2 + \varepsilon_{Y} \nonumber\\ 
&= \left[ \delta + \beta_6 (c_6 + b_{65} c_5+ b_{65} c_4 b_{54}) + \beta_2 c_1 b_{21} \right] \varepsilon_{U} + \beta_2 b_{21} \varepsilon_{X_1} + \beta_2 \varepsilon_{X_2} \nonumber \\& + \beta_6 b_{65} b_{54} \varepsilon_{X_4} + \beta_6 b_{65} \varepsilon_{X_5} +  \beta_6 \varepsilon_{X_6} + \varepsilon_{Y}.
\end{align}

We now consider the causal relationship $X_2 \to Y$ in Figure \ref{Appendix-Fig-model-example}(a). $X_1$ and $X_6$ are the valid NCE and NCO for the causal relationship $X_2 \to Y$, respectively.
We have the vanishing determinants
on the cross-covariance matrix $\boldsymbol{\Sigma}_{\{X_2, X_3, X_1\},  \{X_2,Y, X_6\}}$, i.e.,
\begin{align}\label{Appendix-Eq_Rank1-intuition}
\mathrm{det}(\boldsymbol{\Sigma}_{\{X_2, X_3, X_1\},  \{X_2,Y, X_6\}})  &= \mathrm{det}(\begin{bmatrix*}
\sigma_{X_2X_2} & \sigma_{X_2Y} & \sigma_{X_2X_6} \\
\sigma_{X_3X_2} & \sigma_{X_3Y} & \sigma_{X_3X_6}\\
\sigma_{X_1X_2} & \sigma_{X_1Y} & \sigma_{X_1X_6}
\end{bmatrix*}
)=0
\end{align}
By Eq.\ref{Appendix-Eq_Rank1-intuition}, we quickly know that $\mathrm{rk}(\boldsymbol{\Sigma}_{\{X_2, X_3, X_1\},  \{X_2,Y, X_6\}}) \leq 2$.
We next add an edge $X_1 \to Y$ to Figure \ref{Appendix-Fig-model-example}(a) such that $X_1$ and $X_6$ are the invalid NCE and NCO relative to the causal relationship $X_2 \to Y$ (as shown in Figure \ref{Appendix-Fig-model-example}(b)). 
Now, the vanishing determinant on the cross-covariance matrix $\boldsymbol{\Sigma}_{\{X_2, X_3, X_1\},  \{X_2,Y, X_6\}}$ will fail, i.e.,
\begin{align}\label{Eq_Rank2-intuition}
\mathrm{det}(\boldsymbol{\Sigma}_{\{X_2, X_3, X_1\},  \{X_2,Y, X_6\}}) &= \mathrm{det}(\begin{bmatrix*}
\sigma_{X_2X_2} & \sigma_{X_2Y} & \sigma_{X_2X_6} \\
\sigma_{X_3X_2} & \sigma_{X_3Y} & \sigma_{X_3X_6}\\
\sigma_{X_1X_2} & \sigma_{X_1Y} & \sigma_{X_1X_6}
\end{bmatrix*}
)\nonumber\\
&= \sigma _{U}^{2} \sigma _{X_{1}}^{2} \sigma _{X_{2}}^{2} c_{3} \beta _{1}( c_{4} b_{54} b_{65} +c_{5} b_{65} +c_{6}) \neq 0
\end{align}
That is to say, $\boldsymbol{\Sigma}_{\{X_2, X_3, X_1\},  \{X_2,Y, X_6\}}$ is full rank.
Assuming the distribution is rank-faithful to the graph, the above facts show that lack of edge $X_1 \to Y$, i.e., the variable of NCE does not causally affect the primary outcome, has a testable implication.

\section{More Details on the Motivating Example for Higher-Order Statistics (in Section \ref{Subsection-Higher-Order-Statistics})}\label{Appendix-Higher-Order}

In this section, we will give the details of another motivating example described in Section \ref{Subsection-Higher-Order-Statistics}.

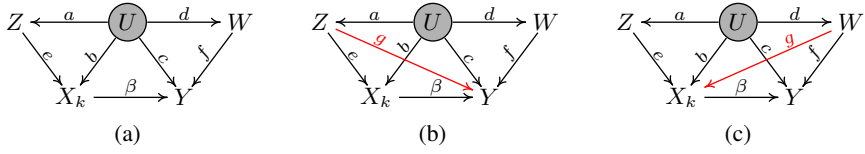
\begin{figure}[htp]
	\centering
	\begin{tikzpicture}[scale=1.0, line width=0.5pt, inner sep=0.2mm, shorten >=.1pt, shorten <=.1pt]
		\draw (3, 1.0) node(U) [circle, fill=gray!60, minimum size=0.5cm,draw] {{\footnotesize\,$U$\,}};
		\draw (1.5, 1.0) node(Z) [] {{\footnotesize\,$Z$\,}};
		\draw (4.5, 1.0) node(W) [] {{\footnotesize\,$W$\,}};
		\draw (2.25, 0) node(X) [] {{\footnotesize\,{$X_k$}\,}};
		\draw (3.75, 0) node(Y) [] {{\footnotesize\,{$Y$}\,}};
		%
		%
		\draw[-arcsq] (U) -- (Z) node[pos=0.5,sloped,above] {{\scriptsize\,$a$\,}}; 
		\draw[-arcsq] (U) -- (W)node[pos=0.5,sloped,above] {{\scriptsize\,$d$\,}}; 
		\draw[-arcsq] (U) -- (X) node[pos=0.5,sloped,above] {{\scriptsize\,$b$\,}};
		\draw[-arcsq] (U) -- (Y) node[pos=0.5,sloped,above] {{\scriptsize\,$c$\,}};
		\draw[-arcsq] (X) -- (Y) node[pos=0.5,sloped,above] {{\scriptsize\,$\beta$\,}};
    	\draw[-arcsq] (Z) -- (X) node[pos=0.5,sloped,above] {{\scriptsize\,$e$\,}};
		\draw[-arcsq] (W) -- (Y) node[pos=0.5,sloped,above] {{\scriptsize\,$f$\,}};
		\draw (3,-0.5) node(label-ii) [] {{\footnotesize\,(a)\,}};
	\end{tikzpicture}~~~~~~
	\begin{tikzpicture}[scale=1.0, line width=0.5pt, inner sep=0.2mm, shorten >=.1pt, shorten <=.1pt]
	\draw (3, 1.0) node(U) [circle, fill=gray!60, minimum size=0.5cm,draw] {{\footnotesize\,$U$\,}};
	\draw (1.5, 1.0) node(Z) [] {{\footnotesize\,$Z$\,}};
	\draw (4.5, 1.0) node(W) [] {{\footnotesize\,$W$\,}};
	\draw (2.25, 0) node(X) [] {{\footnotesize\,{$X_k$}\,}};
	\draw (3.75, 0) node(Y) [] {{\footnotesize\,{$Y$}\,}};
	%
	%
	\draw[-arcsq] (U) -- (Z) node[pos=0.5,sloped,above] {{\scriptsize\,$a$\,}}; 
	\draw[-arcsq] (U) -- (W)node[pos=0.5,sloped,above] {{\scriptsize\,$d$\,}}; 
	\draw[-arcsq] (U) -- (X) node[pos=0.3,sloped,above] {{\scriptsize\,$b$\,}};
	\draw[-arcsq] (U) -- (Y) node[pos=0.5,sloped,above] {{\scriptsize\,$c$\,}};
	\draw[-arcsq] (X) -- (Y) node[pos=0.5,sloped,above] {{\scriptsize\,$\beta$\,}}; 
        \draw[-arcsq] (Z) -- (X) node[pos=0.5,sloped,above] {{\scriptsize\,$e$\,}}; 
	\draw[-arcsq, color=red] (Z) -- (Y) node[pos=0.3,sloped,above] {{\scriptsize\,$g$\,}}; 
	\draw[-arcsq] (W) -- (Y) node[pos=0.5,sloped,above] {{\scriptsize\,$f$\,}}; 
	\draw (3,-0.5) node(label-ii) [] {{\footnotesize\,(b)\,}};
	\end{tikzpicture}~~~~~~
    \begin{tikzpicture}[scale=1.0, line width=0.5pt, inner sep=0.2mm, shorten >=.1pt, shorten <=.1pt]
	\draw (3, 1.0) node(U) [circle, fill=gray!60, minimum size=0.5cm,draw] {{\footnotesize\,$U$\,}};
	\draw (1.5, 1.0) node(Z) [] {{\footnotesize\,$Z$\,}};
	\draw (4.5, 1.0) node(W) [] {{\footnotesize\,$W$\,}};
	\draw (2.25, 0) node(X) [] {{\footnotesize\,{$X_k$}\,}};
	\draw (3.75, 0) node(Y) [] {{\footnotesize\,{$Y$}\,}};
	%
	%
	\draw[-arcsq] (U) -- (Z) node[pos=0.5,sloped,above] {{\scriptsize\,$a$\,}}; 
	\draw[-arcsq] (U) -- (W)node[pos=0.5,sloped,above] {{\scriptsize\,$d$\,}}; 
	\draw[-arcsq] (U) -- (X) node[pos=0.5,sloped,above] {{\scriptsize\,$b$\,}};
	\draw[-arcsq] (U) -- (Y) node[pos=0.3,sloped,above] {{\scriptsize\,$c$\,}};
	\draw[-arcsq] (X) -- (Y) node[pos=0.5,sloped,above] {{\scriptsize\,$\beta$\,}}; 
	\draw[-arcsq] (Z) -- (X) node[pos=0.5,sloped,above] {{\scriptsize\,$e$\,}}; 
        \draw[-arcsq] (W) -- (Y) node[pos=0.5,sloped,above] {{\scriptsize\,$f$\,}}; 
	\draw[-arcsq, color=red] (W) -- (X) node[pos=0.3,sloped,above] {{\scriptsize\,$g$\,}}; 
	\draw (3,-0.5) node(label-ii) [] {{\footnotesize\,(c)\,}};
\end{tikzpicture}
	\caption{A linear causal model with any of the graphical structures above entails the same rank constraints in the marginal covariance matrix of $\{X_k, Y, Z, W\}$, but it entails different GIN constraints. }
	\label{Appendix-Fig-equivalent-proximal-example}
\end{figure}

We assume that all random variables have mean zero for simplicity.
We consider the linear causal model with non-Gaussian noise terms. 


\textbf{\emph{Structure (a)}}: Since the variables strictly follow the linear causal model, we obtain
\begin{align}\nonumber
{U} & = \varepsilon _{U},\label{Eq1-latent} \\ \nonumber
{Z} & = a {U} + \varepsilon_{Z} = a \varepsilon_{U} + \varepsilon_{Z},\\ \nonumber
{W} & = d {U} + \varepsilon_{W} = d \varepsilon_{U} + \varepsilon_{W},\\ \nonumber
{X_k} & = b {U} + e Z + \varepsilon_{X_k} = (ae+b)\varepsilon _{U} + e\varepsilon_{Z} + \varepsilon _{X_{k}},\\ 
Y & = c {U} + \beta X_k + f W + \varepsilon_{Y} \nonumber \\
&= (ae+b+c+df)\varepsilon _{U} + e\beta\varepsilon_{Z} + \beta\varepsilon_{X_k} + f\varepsilon_{W} + \varepsilon _{X_{k}}.
\end{align}

1). By the above equations, $\{X_k,Y,W\}$ and $\{X_k,Z\}$ can then be represented as
\begin{align}
\underbrace{\left[\begin{matrix}
{{X_k}}\\
{{Y}}\\
{{W}}
\end{matrix}\right]}_{\mathcal{Y}} & =  {\left[\begin{matrix}
1 & {0}\\
\beta & {c+fd}\\
0 & {d}
\end{matrix}\right]}\left[\begin{matrix}
{{{{X_k}}}}\\
{{{{U}}}}
\end{matrix}\right] + \underbrace{\left[\begin{matrix}
{{0}}\\
{{f\varepsilon_{W}+\varepsilon_{{Y}}}}\\
{{\varepsilon_{{W}}}}
\end{matrix}\right]}_{\mathcal{E}_{\mathcal{Y}}},\\
\underbrace{\left[\begin{matrix}
{{X_k}}\\
{{Z}}
\end{matrix}\right]}_{\mathcal{Z}} &=  {\left[\begin{matrix}
1 &  0\\
0 &  a
\end{matrix}\right]}\left[\begin{matrix}
{{X_k}}\\
{{U}}
\end{matrix}\right] + \underbrace{\left[\begin{matrix}
{{0}}\\
{\varepsilon_{{Z}}}
\end{matrix}\right]}_{\mathcal{E}_{\mathcal{Z}}}, \label{equ-y1-y2}
\end{align}
According to the above equations,
$\omega^\intercal \mathbb{E}[\mathcal{Y}\mathcal{Z}^{\intercal}] = \mathbf{0}$
$\Rightarrow$ $\omega = [d\beta, -d, c + df]^{\intercal}$. Then we can see $\omega^\intercal \mathcal{Y}= \omega^\intercal \mathcal{E}_{\mathcal{Y}}= c\varepsilon_{W} - d\varepsilon_{Y}$. By Darmois–Skitovich theorem, $\omega^\intercal \mathcal{Y}$ is independent of $\mathcal{Z}$ because there is no common non-Gaussian noise terms between $c\varepsilon_{W} - d\varepsilon_{Y}$ and $\mathcal{Z}$ (including $\varepsilon_{U}$ and $\varepsilon_{Z}$).
That is to say, $(\{X_k, Z\}, \{X_k, Y, W\})$ \emph{follows the GIN constraint}.

2). $\{X_k,Z\}$ and $W$ can then be represented as
\begin{align}
\underbrace{\left[\begin{matrix}
{{X_k}}\\
{Z}
\end{matrix}\right]}_{\mathcal{Y}} & =  {\left[\begin{matrix}
{ae+b}\\
{a}
\end{matrix}\right]}\left[\begin{matrix}
{{{{U}}}}
\end{matrix}\right] + \underbrace{\left[\begin{matrix}
{{e\varepsilon_{Z}+\varepsilon_{X_k}}}\\
{{\varepsilon_{{Z}}}}
\end{matrix}\right]}_{\mathcal{E}_{\mathcal{Y}}},\\
\underbrace{\left[\begin{matrix}
{W}
\end{matrix}\right]}_{\mathcal{Z}} &=  {\left[\begin{matrix}
d
\end{matrix}\right]}\left[\begin{matrix}
{{U}}
\end{matrix}\right] + \underbrace{\left[\begin{matrix}
{\varepsilon_{{W}}}
\end{matrix}\right]}_{\mathcal{E}_{\mathcal{Z}}}, \label{equ-y1-y2}
\end{align}
According to the above equations, 
$\omega^\intercal \mathbb{E}[\mathcal{Y}\mathcal{Z}^{\intercal}] = \mathbf{0}$
$\Rightarrow$ $\omega = [-a, e + b]^{\intercal}$. Then we can see $\omega^\intercal \mathcal{Y}= \omega^\intercal {\mathcal{E}_{\mathcal{Y}}} = b\varepsilon_{Z} - a\varepsilon_{X_k}$. By Darmois–Skitovich theorem, $\omega^\intercal \mathcal{Y}$ is independent of $[W]$ because there is no common non-Gaussian noise terms between $b\varepsilon_{Z} - a\varepsilon_{X_k}$ and $[W]$.
That is to say, {$(\{W\}, \{X_k, Z\})$ \emph{follows the GIN constraint}}.

\textbf{\emph{Structure (b)}}: The data generation process is as follows:
\begin{align}\nonumber
{U} & = \varepsilon _{U},\label{Eq1-latent} \\ \nonumber
{Z} & = a {U} + \varepsilon_{Z} = a \varepsilon_{U} + \varepsilon_{Z},\\ \nonumber
{W} & = d {U} + \varepsilon_{W} = d \varepsilon_{U} + \varepsilon_{W},\\ \nonumber
{X_k} & = b {U} + e Z + \varepsilon_{X_k} = (ae+b)\varepsilon _{U} + e\varepsilon_{Z} + \varepsilon _{X_{k}},\\ \nonumber
Y & = c {U} + \beta X_k + f W + g Z + \varepsilon_{Y} \\&=(ae+b+c+df +ag)\varepsilon _{U} + (e\beta+g)\varepsilon_{Z} + \beta\varepsilon_{X_k} + f\varepsilon_{W} + \varepsilon _{X_{k}}.
\end{align}

1). By the above equations, $\{X_k,Y,W\}$ and $\{X_k,Z\}$ can then be represented as
\begin{align}
\underbrace{\left[\begin{matrix}
{{X_k}}\\
{{Y}}\\
{{W}}
\end{matrix}\right]}_{\mathcal{Y}} & =  {\left[\begin{matrix}
1 & {0}\\
\beta & {c+fd + ag}\\
0 & {d}
\end{matrix}\right]}\left[\begin{matrix}
{{{{X_k}}}}\\
{{{{U}}}}
\end{matrix}\right] + \underbrace{\left[\begin{matrix}
{{0}}\\
{{f\varepsilon_{W}+ a\varepsilon_{Z}+ \varepsilon_{{Y}}}}\\
{{\varepsilon_{{W}}}}
\end{matrix}\right]}_{\mathcal{E}_{\mathcal{Y}}},\\
\underbrace{\left[\begin{matrix}
{{X_k}}\\
{{Z}}
\end{matrix}\right]}_{\mathcal{Z}} &=  {\left[\begin{matrix}
1 &  0\\
0 &  a
\end{matrix}\right]}\left[\begin{matrix}
{{X_k}}\\
{{U}}
\end{matrix}\right] + \underbrace{\left[\begin{matrix}
{{0}}\\
{\varepsilon_{{Z}}}
\end{matrix}\right]}_{\mathcal{E}_{\mathcal{Z}}}, \label{equ-y1-y2}
\end{align}

We have $\omega^\intercal \mathcal{Y}$ is dependent of $[X_k,Z]$ because there exists common non-Gaussian noise terms $\varepsilon_Z$ between $\mathcal{Y}$ and $\mathcal{Z}$, no matter $\omega^\intercal \mathbb{E}[\mathcal{Y}\mathcal{Z}^{\intercal}] = \mathbf{0}$ or not.
That is to say, {$(\{X_k, Z\}, \{X_k, Y, W\})$ \emph{violates the GIN constraint}}.

2). $\{X_k,Z\}$ and $W$ can then be represented as
\begin{align}
\underbrace{\left[\begin{matrix}
{{X_k}}\\
{Z}
\end{matrix}\right]}_{\mathcal{Y}} & =  {\left[\begin{matrix}
{ae+b}\\
{a}
\end{matrix}\right]}\left[\begin{matrix}
{{{{U}}}}
\end{matrix}\right] + \underbrace{\left[\begin{matrix}
{{e\varepsilon_{Z}+\varepsilon_{X_k}}}\\
{{\varepsilon_{{Z}}}}
\end{matrix}\right]}_{\mathcal{E}_{\mathcal{Y}}},\\
\underbrace{\left[\begin{matrix}
{W}
\end{matrix}\right]}_{\mathcal{Z}} &=  {\left[\begin{matrix}
d
\end{matrix}\right]}\left[\begin{matrix}
{{U}}
\end{matrix}\right] + \underbrace{\left[\begin{matrix}
{\varepsilon_{{W}}}
\end{matrix}\right]}_{\mathcal{E}_{\mathcal{Z}}}, \label{equ-y1-y2}
\end{align}

According to the above equations, 
$\omega^\intercal \mathbb{E}[\mathcal{Y}\mathcal{Z}^{\intercal}] = \mathbf{0}$
$\Rightarrow$ $\omega = [-a, e + b]^{\intercal}$. Then we can see $\omega^\intercal \mathcal{Y}= \omega^\intercal {\mathcal{E}_{\mathcal{Y}}} = b\varepsilon_{Z} - a\varepsilon_{X_k}$. By Darmois–Skitovich theorem, $\omega^\intercal \mathcal{Y}$ is independent of $[W]$ because there is no common non-Gaussian noise terms between $b\varepsilon_{Z} - a\varepsilon_{X_k}$ and $[W]$.
That is to say, {$(\{W\}, \{X_k, Z\})$ \emph{follows the GIN constraint}}.

\textbf{\emph{Structure (c)}}: The data generation process is as follows:
\begin{align}\nonumber
{U} & = \varepsilon _{U},\label{Eq1-latent} \\ \nonumber
{Z} & = a {U} + \varepsilon_{Z} = a \varepsilon_{U} + \varepsilon_{Z},\\ \nonumber
{W} & = d {U} + \varepsilon_{W} = d \varepsilon_{U} + \varepsilon_{W},\\ \nonumber
{X_k} & = b {U} + e Z + g W + \varepsilon_{X_k} = (ae+b+gd)\varepsilon _{U} + e\varepsilon_{Z} + g\varepsilon_{W} + \varepsilon _{X_{k}},\\ \nonumber
Y & = c {U} + \beta X_k + f W + \varepsilon_{Y} \\&= (ae+b+gd+c+df)\varepsilon _{U} + e\beta\varepsilon_{Z} + \beta\varepsilon_{X_k} + (\beta g+ f)\varepsilon_{W} + \varepsilon _{X_{k}}.
\end{align}

1). By the above equations, $\{X_k,Y,W\}$ and $\{X_k,Z\}$ can then be represented as
\begin{align}
\underbrace{\left[\begin{matrix}
{{X_k}}\\
{{Y}}\\
{{W}}
\end{matrix}\right]}_{\mathcal{Y}} & =  {\left[\begin{matrix}
1 & {0}\\
\beta & {c+fd}\\
0 & {d}
\end{matrix}\right]}\left[\begin{matrix}
{{{{X_k}}}}\\
{{{{U}}}}
\end{matrix}\right] + \underbrace{\left[\begin{matrix}
{{0}}\\
{{f\varepsilon_{W} + \varepsilon_{{Y}}}}\\
{{\varepsilon_{{W}}}}
\end{matrix}\right]}_{\mathcal{E}_{\mathcal{Y}}},\\
\underbrace{\left[\begin{matrix}
{{X_k}}\\
{{Z}}
\end{matrix}\right]}_{\mathcal{Z}} &=  {\left[\begin{matrix}
1 &  0\\
0 &  a
\end{matrix}\right]}\left[\begin{matrix}
{{X_k}}\\
{{U}}
\end{matrix}\right] + \underbrace{\left[\begin{matrix}
{{0}}\\
{\varepsilon_{{Z}}}
\end{matrix}\right]}_{\mathcal{E}_{\mathcal{Z}}}, \label{equ-y1-y2}
\end{align}

According to the above equations, 
$\omega^\intercal \mathbb{E}[\mathcal{Y}\mathcal{Z}^{\intercal}] = \mathbf{0}$
$\Rightarrow$ $\omega = [d\beta, -d, c + df]^{\intercal}$. Then we can see $\omega^\intercal \mathcal{Y}= \omega^\intercal {\mathcal{E}_{\mathcal{Z}}} = c\varepsilon_{W} - d\varepsilon_{Y}$. By Darmois–Skitovich theorem, $\omega^\intercal \mathcal{Y}$ is independent of $\mathcal{Z}$ because there is no common non-Gaussian noise terms between $c\varepsilon_{W} - d\varepsilon_{Y}$ and $\mathcal{Z}$.
That is to say, {$(\{X_k, Z\}, \{X_k, Y, W\})$ \emph{follows the GIN constraint} (This result is the same as the result in Structure (a))}.

2). {However, we have that $(\{W\}, \{X_k, Z\})$ \emph{violates the GIN constraint}, as explained below}. $\{X_k,Z\}$ and $W$ can then be represented as

\begin{align}
\underbrace{\left[\begin{matrix}
{{X_k}}\\
{Z}
\end{matrix}\right]}_{\mathcal{Y}} & =  {\left[\begin{matrix}
{ae+b+gd}\\
{a}
\end{matrix}\right]}\left[\begin{matrix}
{{{{U}}}}
\end{matrix}\right] + \underbrace{\left[\begin{matrix}
{{e\varepsilon_{Z}+g\varepsilon_{W}+\varepsilon_{X_k}}}\\
{{\varepsilon_{{Z}}}}
\end{matrix}\right]}_{\mathcal{E}_{\mathcal{Y}}},\\
\underbrace{\left[\begin{matrix}
{W}
\end{matrix}\right]}_{\mathcal{Z}} &=  {\left[\begin{matrix}
d
\end{matrix}\right]}\left[\begin{matrix}
{{U}}
\end{matrix}\right] + \underbrace{\left[\begin{matrix}
{\varepsilon_{{W}}}
\end{matrix}\right]}_{\mathcal{E}_{\mathcal{Z}}}, \label{equ-y1-y2}
\end{align}

We have $\omega^\intercal \mathcal{Y}$ is dependent of $\mathcal{Z}$ because there exists common non-Gaussian noise terms $\varepsilon_W$ between $\mathcal{Y}$ and $\mathcal{Z}$, no matter $\omega^\intercal \mathbb{E}[\mathcal{Y}\mathcal{Z}^{\intercal}] = \mathbf{0}$ or not.

\textbf{\emph{Conclusion}}: Assuming the distribution is rank-faithful to the graph, the above facts show that lack of edges $Z \to Y$ (i.e., the variable of NCE does not causally affect the primary outcome) or $W \to X_k$ (i.e., the variable of NCO does not causally affect the primary treatment) has a testable implication.

\section{More Details about Depending on Assumption 4 in Lemma \ref{Lemma-rule3}}\label{Appendix-Section-Assumption-Lemma3}

We here would like to mention that the result of Lemma \ref{Lemma-rule3} does not strictly require adherence to Assumption 3—that is, not all noise variables need to follow a non-Gaussian distribution. For instance, consider the causal graphs shown in Figure 4, the identification of valid NCO and NCE in subfigure (a) solely depends on the non-Gaussian distribution of the noise components associated with variables Z and W, making them valid, whereas the other subfigures demonstrate invalid cases. Specifically, as elaborated in Appendix C, for subfigure (b), we find that $\omega^\intercal \mathcal{Y}$ is dependent on $[X_k,Z]$ due to the presence of common non-Gaussian noise terms $\varepsilon_Z$ between $\mathcal{Y}$ and $\mathcal{Z}$. This dependence aligns with the Darmois-Skitovitch theorem, which necessitates that $\varepsilon_Z$ must be non-Gaussian. Similarly, for subfigure (c), we observe that $\omega^\intercal \mathcal{Y}$ is dependent on $\mathcal{Z}$ due to the presence of common non-Gaussian noise terms $\varepsilon_W$ between $\mathcal{Y}$ and $\mathcal{Z}$, again requiring $\varepsilon_W$ to be non-Gaussian as per the Darmois-Skitovitch theorem.

\section{More Details on Proxy-Rank Algorithm (in Section \ref{Section-Algorithm-proxy-rank})}\label{Appendix-Algorithm-rank}

The specific details of the Proxy-Rank algorithm are provided in the following,

\begin{algorithm}[htp]
	\caption{Proxy-Rank}
	\label{Alg-Proxy-Rank}
	\begin{algorithmic}[1]
        \REQUIRE
	A dataset of treatments $\mathbf{X} = \{ {X_1},...,{X_p}\}$, outcome $Y$, and the number of unmeasured confounders $q$.\\
		\STATE Initialize sets $\mathcal{C}=\emptyset$, $\mathcal{NCE}=\emptyset$, and $\mathcal{NCO}=\emptyset$
            \FOR{every variable $X_k$ in $\mathbf{X}$}
		\REPEAT
		\STATE Select two subsets $\mathbf{A}$, $\mathbf{B}$  and one variable $Q$ from $\mathbf{X} \setminus {X_k}$ such that  $\mathbf{A} \cap \mathbf{B} \cap \{Q\} = \emptyset$, and that $|\mathbf{A}| = q$, $|\mathbf{B}| = q$.
		\IF{$\mathbf{A}$, $\mathbf{B}$,  and $Q$ satisfy $\mathcal{R}_1$ of Lemma \ref{Lemma-rule1}}
		\STATE $\mathcal{NCE}_k  \leftarrow \mathbf{A}$, $\mathcal{NCO}_k  \leftarrow \mathbf{B}$;
		\STATE $\mathcal{C}_k =  \frac{\mathrm{det}(\boldsymbol{\Sigma}_{\{X_k \cup \mathbf{A}\}, \{Y \cup \mathbf{B}\}})}{\mathrm{det}(\boldsymbol{\Sigma}_{\{X_k \cup \mathbf{A}\}, \{X_k \cup \mathbf{B}\}})}$;
		\STATE Break the loop of line 3;
		\ENDIF
		\UNTIL{all possible disjoint subsets $\mathbf{A}$ with $|\mathbf{A}|=q$, $\mathbf{B}$ with $|\mathbf{B}|=q$, and variable $Q$ in $\mathbf{X} \setminus X_k$ are selected.}
        \IF{$\mathcal{NCO}_{k} \neq \emptyset$}
		\STATE Continue;
		\ENDIF
            \REPEAT
		\STATE Select two subsets $\mathbf{A}$ and $\mathbf{B}$ from $\mathbf{X} \setminus {X_k}$ such that  $\mathbf{A} \cap \mathbf{B} = \emptyset$, and that $|\mathbf{A}| = q+1$, $|\mathbf{B}| = q+1$.
		\IF{$\mathbf{A}$ and $\mathbf{B}$ satisfy $\mathcal{R}_2$ of Lemma \ref{Lemma-rule2}}
		\STATE $\mathcal{NCE}_k  \leftarrow \mathbf{A}$, $\mathcal{NCO}_k  \leftarrow \mathbf{B}$;
		\STATE $\mathcal{C}_k =  \frac{\mathrm{det}(\boldsymbol{\Sigma}_{\{X_k \cup \mathbf{A}\}, \{Y \cup \mathbf{B}\}})}{\mathrm{det}(\boldsymbol{\Sigma}_{\{X_k \cup \mathbf{A}\}, \{X_k \cup \mathbf{B}\}})}$;
		\STATE Break the for loop of line 14
		\ENDIF
		\UNTIL{all possible disjoint subsets $\mathbf{A}$ with $|\mathbf{A}|=q+1$, and $\mathbf{B}$ with $|\mathbf{B}|=q+1$ in $\mathbf{X} \setminus X_k$ are selected.}
            \ENDFOR
            \IF{$\mathcal{NCO}_{k} \neq \emptyset$}
		\STATE Continue;
            \ELSE
            \STATE $\mathcal{C}_k= \mathrm{NA}$. ~// indicating the lack of knowledge to obtain the unbiased causal effect.
		\ENDIF
	\ENSURE
	$\mathcal{C}$, a set that collects the total causal effects of $X_k \in \mathbf{X}$ on $Y$.
		\end{algorithmic}
	\vspace{-0.1cm}
\end{algorithm}

\section{More Details on Proxy-GIN Algorithm (in Section \ref{Section-Algorithm-GIN})}\label{Appendix-Algorithm-GIN}
The specific details of the Proxy-GIN algorithm are provided in the following,

\begin{algorithm}[htb]
        \caption{Proxy-GIN}
	\label{Alg-Proxy-GIN}
        \begin{algorithmic}[1]
        \REQUIRE
	A dataset of treatments $\mathbf{X} = \{ {X_1},...,{X_p}\}$, outcome $Y$, and the number of unmeasured confounders $q$.\\
		\STATE Initialize  sets $\mathcal{C}=\emptyset$, $\mathcal{NCE}=\emptyset$, and $\mathcal{NCO}=\emptyset$
            \FOR{every variable $X_k$ in $\mathbf{X}$}
		\REPEAT
		\STATE Select two disjoint subsets $\mathbf{A}$ and $\mathbf{B}$   from $\mathbf{X} \setminus {X_k}$ such that  $|\mathbf{A}| = q$, $|\mathbf{B}| = q$.
		\IF{$\mathbf{A}$ and $\mathbf{B}$ satisfy $\mathcal{R}3$ of Lemma \ref{Lemma-rule3}}
		\STATE $\mathcal{NCE}_k  \leftarrow \mathbf{A}$, $\mathcal{NCO}_k  \leftarrow \mathbf{B}$;
		\STATE $\mathcal{C}_k =  \frac{\mathrm{det}(\boldsymbol{\Sigma}_{\{X_k \cup \mathbf{A}\}, \{Y \cup \mathbf{B}\}})}{\mathrm{det}(\boldsymbol{\Sigma}_{\{X_k \cup \mathbf{A}\}, \{X_k \cup \mathbf{B}\}})}$;
		\STATE Break the loop of line 3;
		\ENDIF
		\UNTIL{all possible disjoint subsets $\mathbf{A}$ with $|\mathbf{A}|=q$ and $\mathbf{B}$ with $|\mathbf{B}|=q$ in $\mathbf{X} \setminus X_k$ are selected.}
        \IF{$\mathcal{NCO}_{k} \neq \emptyset$}
		\STATE Continue;
		\ENDIF
         \ENDFOR
            \IF{$\mathcal{NCO}_{k} \neq \emptyset$}
		\STATE Continue;
            \ELSE
            \STATE $\mathcal{C}_k= \mathrm{NA}$. ~// indicating the lack of knowledge to obtain the unbiased causal effect.
		\ENDIF
	\ENSURE
	$\mathcal{C}$, a set that collects the total causal effects of $X_k \in \mathbf{X}$ on $Y$.
		\end{algorithmic}
	\vspace{-0.1cm}
\end{algorithm}

\section{Discussion on the Consistency Result and Convergence Rate of Theorem \ref{Theorem-Corr-Rank}}\label{Appendix-Consistency-The2}
\subsection{Discussion on the Consistency Result}
\textbf{General description}. Theorem 2 shows that the true causal effect obtained by the Proxy-Rank is correct in the sense that the proxy variables relative to the causal relationship of interest are valid and the causal effect in the output $\mathcal{C}$ is the true value.
In fact, the consistency results of our estimation depend on two processes: first, appropriately selecting proximal variables; next, based on the selected proxy variables, the nonparametric estimator $\hat{\beta}_{X_k\to Y}$ is obtained using Eq. (3). 
For Theorem 2, the formal statement can be expressed as follows: 
let $\mathrm{cond1}=(\mathrm{rk} (\hat{\boldsymbol{\Sigma}}_{\{X_k, Q, \mathbf{A}_1\}, \{X_k, Y, \mathbf{B}_1\}}) \leq q+1, \mathrm{rk}(\hat {\boldsymbol{\Sigma}}_{\{X_k, \mathbf{A}_1\}, \{Q,\mathbf{B}_1\}} \leq q) $, where $|\mathbf{A}_1|=q$ and $|\mathbf{B}_1|=q$ (Lemma 1),
and $\mathrm{cond2}=(\mathrm{rk} (\hat{\boldsymbol{\Sigma}}_{\{X_k, \mathbf{A}_2\}, \{X_k, Y,\mathbf{B}_2\}}) \leq q+1, \mathrm{rk}(\hat {\boldsymbol{\Sigma}}_{\{X_k, \mathbf{A}_2\}, \{\mathbf{B}_2\}})) \leq q)$, where $|\mathbf{A}_2|=q+1$ and $|\mathbf{B}_2|=q+1$ (Lemma 2).
\begin{align}
    \lim_{n\to\infty} \mathrm{pr} (|\hat{\beta}_{X_k\to Y}-{\beta}_{X_k\to Y}| >\epsilon \mid 
 \mathrm{cond1}~or~\mathrm{cond2} )=0
\end{align}
 for all $\epsilon>0$, which implies that the estimated causal effect $\hat{\beta}_{X_k\to Y}$ obtained by the Proxy-Rank algorithm is consistent. 

\textbf{Proof details.} Before presenting the detailed proof, we first introduce two lemmas that will aid in our demonstration. Let

 $$\Delta_1=\mathrm{rk} (\hat{\boldsymbol{\Sigma}}_{\{X_k, \mathbf{A}\},  \mathbf{B}})-\mathrm{rk} ({\boldsymbol{\Sigma}}_{\{X_k, \mathbf{A}\},  \mathbf{B}}),~\Delta_2=\mathrm{rk}(\hat {\boldsymbol{\Sigma}}_{\{X_k, \mathbf{A}\}, \{X_k,Y,\mathbf{B}\}}) -\mathrm{rk}(\boldsymbol{\Sigma}_{\{X_k, \mathbf{A}\}, \{X_k,Y,\mathbf{B}\}}) ,$$

\begin{Lemma}
\label{lem:consistency-Delta}
$ \Delta_1=o_p(1)$ and $ \Delta_2=o_p(1)$, or equivalently, 
 $  \lim_{n\to\infty} \mathrm{pr}( \vert\mathrm{rk} (\hat{\boldsymbol{\Sigma}}_{\{X_k, \mathbf{A}\},  \mathbf{B}})-\mathrm{rk} ({\boldsymbol{\Sigma}}_{\{X_k, \mathbf{A}\},  \mathbf{B}})\vert >\epsilon  ) =0$ and $\lim_{n\to\infty} \mathrm{pr}(  \vert \mathrm{rk}(\hat {\boldsymbol{\Sigma}}_{\{X_k, \mathbf{A}\}, \{X_k,Y,\mathbf{B}\}}) -\mathrm{rk}(\boldsymbol{\Sigma}_{\{X_k, \mathbf{A}\}, \{X_k,Y,\mathbf{B}\}}) \vert >\epsilon  ) =0$ for all $\epsilon>0$. 
\end{Lemma}
\begin{proof}
    The rank of $\mathrm{rk} ({\boldsymbol{\Sigma}}_{\{X_k, \mathbf{A}\},  \mathbf{B}})$ is equal to $r$ if and only if there exists an invertible $m \times m$ matrix $C$ and an invertible $n \times n$ matrix $D$ such that
    $$C{\boldsymbol{\Sigma}}_{\{X_k, \mathbf{A}\},  \mathbf{B}}D=\begin{pmatrix}
        I_r&0\\0&0
    \end{pmatrix}.$$
    By the law of large numbers, we have that $\hat{\boldsymbol{\Sigma}}_{\{X_k, \mathbf{C}\},  \mathbf{B}}={\boldsymbol{\Sigma}}_{\{X_k, \mathbf{C}\},  \mathbf{B}}+o_p(1)$  as well as $$C\hat{\boldsymbol{\Sigma}}_{\{X_k, \mathbf{C}\},  \mathbf{B}}D=C{\boldsymbol{\Sigma}}_{\{X_k, \mathbf{A}\},  \mathbf{B}}D+o_p(1)=\begin{pmatrix}
        I_r&0\\0&0
    \end{pmatrix}+o_p(1).$$
    As a result, we have $\mathrm{rk} (\hat{\boldsymbol{\Sigma}}_{\{X_k, \mathbf{A}\},  \mathbf{B}})=r+o_p(1)=\mathrm{rk} ({\boldsymbol{\Sigma}}_{\{X_k, \mathbf{A}\},  \mathbf{B}})+o_p(1)$. The proof for $ \mathrm{rk}(\hat {\boldsymbol{\Sigma}}_{\{X_k, \mathbf{A}\}, \{X_k,Y,\mathbf{B}\}})$ is similar, we thus omit for simplicity.
\end{proof}
\begin{Lemma}
\label{lem:consistency-beta}
    $ \lim_{n\to\infty} \mathrm{pr}(  \vert \hat{\beta}_{X_k\to Y}-{\beta}_{X_k\to Y}\vert >\epsilon  \mid \mathrm{rk} ({\boldsymbol{\Sigma}}_{\{X_k, \mathbf{A}\},  \mathbf{B}}) \leq q ,   \mathrm{rk}(\boldsymbol{\Sigma}_{\{X_k, \mathbf{A}\}, \{X_k,Y,\mathbf{B}\}}) \leq q+1)=0$ for all $\epsilon>0$. 
\end{Lemma}
\begin{proof}
  According to Lemma 2 in that main text, we know that the event $\mathrm{rk} ({\boldsymbol{\Sigma}}_{\{X_k, \mathbf{A}\},  \mathbf{B}}) \leq q ,   \mathrm{rk}(\boldsymbol{\Sigma}_{\{X_k, \mathbf{A}\}, \{X_k,Y,\mathbf{B}\}}) \leq q+1 $ implies $\mathbf{A}$ and  $\mathbf{B}$ are valid NCE and NCO relative to $X_k\to Y$ respectively.

  Once we have selected suitable proxies $\mathbf{A}$ and  $\mathbf{B}$, we can utilize   Eq.(3) to establish nonparametric estimator  $\hat \beta_{X_k\to Y}$, and according to standard M-estimation theory, specifically Theorem 5.41 in \citet{van2000asymptotic}, we can demonstrate the consistency of $\hat \beta_{X_k\to Y}$.
\end{proof}
\begin{Proposition}
$ 
    \lim_{n\to\infty}        \mathrm{pr} (\vert \hat{\beta}_{X_k\to Y}-{\beta}_{X_k\to Y}\vert >\epsilon\mid \mathrm{rk} (\hat{\boldsymbol{\Sigma}}_{\{X_k, \mathbf{A}\},  \mathbf{B}}) \leq q,\mathrm{rk}(\hat {\boldsymbol{\Sigma}}_{\{X_k, \mathbf{A}\}, \{X_k,Y,\mathbf{B}\}}) \leq q+1)=0 $ for all $\epsilon>0$. 
\end{Proposition}
\begin{proof}
Without loss of generality, we assume that $0<\epsilon<1$. We thus have,
    \begin{align*}
    \mathrm{pr}&(\mathrm{rk} (\hat{\boldsymbol{\Sigma}}_{\{X_k, \mathbf{A}\},  \mathbf{B}}) \leq q,\mathrm{rk}(\hat {\boldsymbol{\Sigma}}_{\{X_k, \mathbf{A}\}, \{X_k,Y,\mathbf{B}\}}) \leq q+1,\vert \hat{\beta}_{X_k\to Y}-{\beta}_{X_k\to Y}\vert >\epsilon)
   \\& =\mathrm{pr}\begin{pmatrix}
       \mathrm{rk} (\hat{\boldsymbol{\Sigma}}_{\{X_k, \mathbf{A}\},  \mathbf{B}})-\mathrm{rk} ({\boldsymbol{\Sigma}}_{\{X_k, \mathbf{A}\},  \mathbf{B}})+\mathrm{rk} ({\boldsymbol{\Sigma}}_{\{X_k, \mathbf{A}\},  \mathbf{B}}) \leq q,\\\mathrm{rk}(\hat {\boldsymbol{\Sigma}}_{\{X_k, \mathbf{A}\}, \{X_k,Y,\mathbf{B}\}}) -\mathrm{rk}(\boldsymbol{\Sigma}_{\{X_k, \mathbf{A}\}, \{X_k,Y,\mathbf{B}\}}) +\mathrm{rk}(\boldsymbol{\Sigma}_{\{X_k, \mathbf{A}\}, \{X_k,Y,\mathbf{B}\}}) \leq q+1,\\\vert \hat{\beta}_{X_k\to Y}-{\beta}_{X_k\to Y}\vert >\epsilon
   \end{pmatrix}\\& =\mathrm{pr}(\Delta_1+\mathrm{rk} ({\boldsymbol{\Sigma}}_{\{X_k, \mathbf{A}\},  \mathbf{B}}) \leq q,\Delta_2   +\mathrm{rk}(\boldsymbol{\Sigma}_{\{X_k, \mathbf{A}\}, \{X_k,Y,\mathbf{B}\}}) \leq q+1,\vert \hat{\beta}_{X_k\to Y}-{\beta}_{X_k\to Y}\vert >\epsilon, \vert\Delta_1\vert >\epsilon, \vert\Delta_2\vert >\epsilon  ) \\&~~~~+\mathrm{pr}(\Delta_1+\mathrm{rk} ({\boldsymbol{\Sigma}}_{\{X_k, \mathbf{A}\},  \mathbf{B}}) \leq q,\Delta_2   +\mathrm{rk}(\boldsymbol{\Sigma}_{\{X_k, \mathbf{A}\}, \{X_k,Y,\mathbf{B}\}}) \leq q+1,\vert \hat{\beta}_{X_k\to Y}-{\beta}_{X_k\to Y}\vert >\epsilon, \vert\Delta_1\vert >\epsilon, \vert\Delta_2\vert \leq\epsilon  )\\&~~~~+\mathrm{pr}(\Delta_1+\mathrm{rk} ({\boldsymbol{\Sigma}}_{\{X_k, \mathbf{A}\},  \mathbf{B}}) \leq q,\Delta_2   +\mathrm{rk}(\boldsymbol{\Sigma}_{\{X_k, \mathbf{A}\}, \{X_k,Y,\mathbf{B}\}}) \leq q+1,\vert \hat{\beta}_{X_k\to Y}-{\beta}_{X_k\to Y}\vert >\epsilon, \vert\Delta_1\vert \leq \epsilon, \vert\Delta_2\vert >\epsilon  ) 
  \\&~~~~+\mathrm{pr}(\Delta_1+\mathrm{rk} ({\boldsymbol{\Sigma}}_{\{X_k, \mathbf{A}\},  \mathbf{B}}) \leq q,\Delta_2   +\mathrm{rk}(\boldsymbol{\Sigma}_{\{X_k, \mathbf{A}\}, \{X_k,Y,\mathbf{B}\}}) \leq q+1,\vert \hat{\beta}_{X_k\to Y}-{\beta}_{X_k\to Y}\vert >\epsilon, \vert\Delta_1\vert \leq\epsilon, \vert\Delta_2\vert \leq\epsilon  )   \\&\leq  \mathrm{pr}( \vert\Delta_1\vert >\epsilon  ) +\mathrm{pr}( \vert\Delta_1\vert >\epsilon  ) +\mathrm{pr}( \vert\Delta_2\vert >\epsilon  ) \\&~~~~+ \mathrm{pr}\begin{pmatrix}
       \Delta_1+\mathrm{rk} ({\boldsymbol{\Sigma}}_{\{X_k, \mathbf{A}\},  \mathbf{B}}) \leq q,  \Delta_2 +\mathrm{rk}(\boldsymbol{\Sigma}_{\{X_k, \mathbf{A}\}, \{X_k,Y,\mathbf{B}\}}) \leq q+1, \vert \hat{\beta}_{X_k\to Y}-{\beta}_{X_k\to Y}\vert >\epsilon,\\
       -\epsilon\leq-\Delta_1 \leq\epsilon ,-\epsilon\leq-\Delta_2 \leq\epsilon 
   \end{pmatrix}  \\&\leq  \mathrm{pr}( \vert\Delta_1\vert >\epsilon  ) +\mathrm{pr}( \vert\Delta_1\vert >\epsilon  ) +\mathrm{pr}( \vert\Delta_2\vert >\epsilon  ) \\&~~~~+ \mathrm{pr}( \mathrm{rk} ({\boldsymbol{\Sigma}}_{\{X_k, \mathbf{A}\},  \mathbf{B}}) \leq q+\epsilon,   \mathrm{rk}(\boldsymbol{\Sigma}_{\{X_k, \mathbf{A}\}, \{X_k,Y,\mathbf{B}\}}) \leq q+1+\epsilon, \vert \hat{\beta}_{X_k\to Y}-{\beta}_{X_k\to Y}\vert >\epsilon )
   \\&\leq o_p(1) +\mathrm{pr}( \mathrm{rk} ({\boldsymbol{\Sigma}}_{\{X_k, \mathbf{A}\},  \mathbf{B}}) \leq q ,   \mathrm{rk}(\boldsymbol{\Sigma}_{\{X_k, \mathbf{A}\}, \{X_k,Y,\mathbf{B}\}}) \leq q+1 , \vert \hat{\beta}_{X_k\to Y}-{\beta}_{X_k\to Y}\vert >\epsilon ) 
                \\&= o_p(1)+\mathrm{pr}(  \vert \hat{\beta}_{X_k\to Y}-{\beta}_{X_k\to Y}\vert >\epsilon  \mid \mathrm{rk} ({\boldsymbol{\Sigma}}_{\{X_k, \mathbf{A}\},  \mathbf{B}}) \leq q ,   \mathrm{rk}(\boldsymbol{\Sigma}_{\{X_k, \mathbf{A}\}, \{X_k,Y,\mathbf{B}\}}) \leq q+1)\\&~~~~\times\mathrm{pr}( \mathrm{rk} ({\boldsymbol{\Sigma}}_{\{X_k, \mathbf{A}\},  \mathbf{B}}) \leq q ,   \mathrm{rk}(\boldsymbol{\Sigma}_{\{X_k, \mathbf{A}\}, \{X_k,Y,\mathbf{B}\}}) \leq q+1 ) 
                \\&\leq o_p(1)+o_p(1)O_p(1)
                \\&=o_p(1),
\end{align*} 
where the two-to-last inequality holds because of (i) Lemma \ref{lem:consistency-Delta}  and (ii) the event $\{\mathrm{rk} ({\boldsymbol{\Sigma}}_{\{X_k, \mathbf{A}\},  \mathbf{B}}) \leq q+\epsilon\}$ is exactly  equivalent to  $\{\mathrm{rk} ({\boldsymbol{\Sigma}}_{\{X_k, \mathbf{A}\},  \mathbf{B}}) \leq q \}$ for $0<\epsilon<1$, and the last inequality holds because of  Lemma \ref{lem:consistency-beta}.

Therefore, we have 
\begin{align*}
    \lim_{n\to\infty} &\mathrm{pr} (|\hat{\beta}_{X_k\to Y}-{\beta}_{X_k\to Y}| >\epsilon\mid \mathrm{rk} (\hat{\boldsymbol{\Sigma}}_{\{X_k, \mathbf{A}\}, \mathbf{B}}) \leq q, \mathrm{rk}(\hat {\boldsymbol{\Sigma}}_{\{X_k, \mathbf{A}\}, \{X_k,Y,\mathbf{B}\}}) \leq q+1)
    \\&=  \lim_{n\to\infty}\dfrac{\mathrm{pr} (|\hat{\beta}_{X_k\to Y}-{\beta}_{X_k\to Y}| >\epsilon, \mathrm{rk} (\hat{\boldsymbol{\Sigma}}_{\{X_k, \mathbf{A}\}, \mathbf{B}}) \leq q, \mathrm{rk}(\hat {\boldsymbol{\Sigma}}_{\{X_k, \mathbf{A}\}, \{X_k,Y,\mathbf{B}\}}) \leq q+1)}{\mathrm{pr} ( \mathrm{rk} (\hat{\boldsymbol{\Sigma}}_{\{X_k, \mathbf{A}\}, \mathbf{B}}) \leq q, \mathrm{rk}(\hat {\boldsymbol{\Sigma}}_{\{X_k, \mathbf{A}\}, \{X_k,Y,\mathbf{B}\}}) \leq q+1)}\\&= \lim_{n\to\infty}\dfrac{\mathrm{pr} (|\hat{\beta}_{X_k\to Y}-{\beta}_{X_k\to Y}| >\epsilon, \mathrm{rk} (\hat{\boldsymbol{\Sigma}}_{\{X_k, \mathbf{A}\}, \mathbf{B}}) \leq q, \mathrm{rk}(\hat {\boldsymbol{\Sigma}}_{\{X_k, \mathbf{A}\}, \{X_k,Y,\mathbf{B}\}}) \leq q+1)}{\mathrm{pr} ( \mathrm{rk} ( {\boldsymbol{\Sigma}}_{\{X_k, \mathbf{A}\}, \mathbf{B}}) \leq q, \mathrm{rk}( {\boldsymbol{\Sigma}}_{\{X_k, \mathbf{A}\}, \{X_k,Y,\mathbf{B}\}}) \leq q+1)}+o_p(1)\\&
    =0.
\end{align*}
\end{proof}

\subsection{Discussion on the Convergence Rate}
Proxy-Rank algorithm involves two states: first, correctly selecting appropriate proxy variables; next, based on the selected proxies, obtaining the nonparametric estimator $\hat{\beta}_{X_k\to Y}$ using Eq. (3). Once we have chosen suitable proxy variables, the nonparametric estimator $\hat \beta_{X_k\to Y}$ of the causal effects based on Eq. (3) will exhibit a $\sqrt n$ rate of convergence, meaning that the random variables $\sqrt{n}(\hat{\beta}_{X_k\to Y} - \beta_{X_k\to Y})$ will converge to a normal distribution. These asymptotic results can still be obtained using standard M-estimation theory, particularly Theorem 5.41 of \citet{van2000asymptotic}.

\section{Discussion and Further Work}\label{Appendix-Discussion}

The preceding sections have provided two different sets of precise
identifiability conditions for the selection of proxy variables of unmeasured confounders, along with their corresponding search algorithms. It's worth noting that these methods theoretically assume that the unobserved confounding variable $\mathbf{U}$ affects both the treatments $\mathbf{X}$ and $Y$, meaning that all entries of matrix $\mathbf{C}$ are non-zero.
This assumption inherently leads to the following two conditions: 1) all variables are mutually dependent on each other, and 2) the two necessary conditions of the extended proxy variables estimator, that is, $\boldsymbol{\Sigma}_{\{X_k, \mathbf{Z}\},  \{Y,\mathbf{W}\}}$ and $\boldsymbol{\Sigma}_{\{X_k, \mathbf{Z}\},  \{X_k,\mathbf{W}\}}$ both are full rank. In some real-world applications, unobserved confounding might not affect all potential treatment variables. Thus, given a causal relationship $X_k \to Y$, before applying the proposed rules $\mathcal{R}1\sim \mathcal{R}3$, one can perform an initial screening as follows: 
\begin{itemize}
    \item [1.] identify the maximal clique set containing both $X_k$ and $Y$, ensuring the removal of variables that are statistically independent of $T$ and $Y$ given the subset in $\mathbf{X}$, and
    \item [2.] identify the candidate sets $\mathbf{Z}$ and $\mathbf{W}$ that satisfy the following two conditions: $\mathrm{rk}(\boldsymbol{\Sigma}_{\{X_k, \mathbf{Z}\},  \{Y,\mathbf{W}\}}) = q+1$ and
$\mathrm{rk}(\boldsymbol{\Sigma}_{\{X_k, \mathbf{Z}\},  \{X_k,\mathbf{W}\}}) = q+1$.
\end{itemize}
The above operation ensures that even if not all entries of matrix $\mathbf{C}$ are non-zero, the identification methods we propose can still be applied.

In this paper, we restrict our attention to linear causal models, which are common in the social sciences and ought to be more common in economics and elsewhere \citep{bollen1989structural,spirtes2000causation}.
One of the future research directions is to address the discrete model or the non-linear causal model, existing techniques, e.g., extended trek separation in \citet{spirtes2013calculation-t-separation} or additive noise model in \citet{hoyer2009ANM}, \citet{zhang2009PNL} and \citet{peters2014causal} may help to address this issue.
Another direction of future work is to extend our results to multiple outcomes setting \citep{wang2017confounder}.

\section{More Results on Experimental Results}\label{Appendix-Section-Simulation}

We here evaluate the performance of the proposed method in an additional \textbf{Mixture case} setting. The data are generated according to the causal graph in Figure \ref{Fig-model-example}, with the noise terms being randomly selected from standard normal distributions and standard exponential distributions.

\begin{figure*}[htp]
  \setlength{\abovecaptionskip}{0pt}
	\setlength{\belowcaptionskip}{-6pt}
	\vspace{-0.1cm}
	\begin{center}
        \includegraphics[height=0.20\textwidth,width=0.9\textwidth]{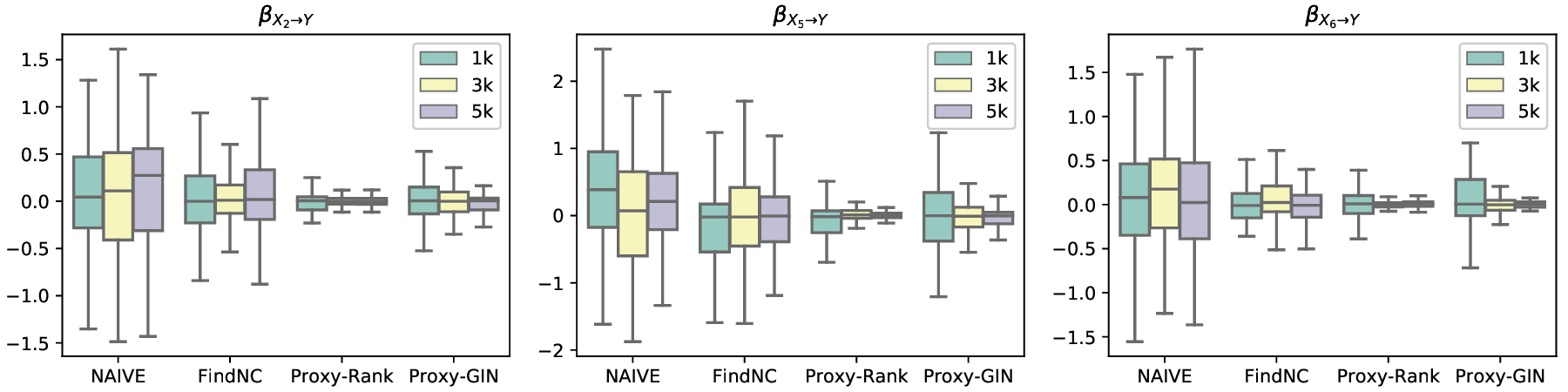}
		\caption{Performance of NAIVE, FindNC, Proxy-Rank, and Proxy-GIN on the Mixture case.}
		\vspace{-0.1cm}
		\label{Fig-simulation-Gaussian-results3} 
	\end{center}
\end{figure*}

Figure \ref{Fig-simulation-Gaussian-results3} summarizes the bias of the estimators of each parameter. As expected, our proposed Proxy-Rank algorithm almost outperforms other methods
(with little bias for all causal effects) with all sample sizes, which indicates that the Proxy-Rank algorithm is a distribution-free method.
An interesting conclusion is that in the Mixture case, the Proxy-GIN algorithm performs equally well, even though Assumption \ref{Ass-non-Gaussianity} is not fully satisfied.
We further noticed that the Proxy-GIN algorithm does not have the same level of stability in the small sample size as the Proxy-Rank algorithm, e.g., 1k for the causal effect of $X_5$ on $Y$ (See Remark \ref{Appendix-Simulation} for more details).
One possible reason is that reliable estimation of higher-order statistics requires much more samples than that of second-order statistics~\citep{hyvarinen2004independent}.


\begin{Remark}\label{Appendix-Simulation}
    Figure \ref{Appendix-Fig-simulation-Gaussian-results} presents a comparative graph of the results obtained from the Proxy-Rank algorithm and the Proxy-GIN algorithm, considering two different distributions: the normal distribution and the exponential distribution, respectively. Our findings reveal that the Proxy-GIN algorithm does not exhibit the same level of stability in cases of small sample sizes, for instance, when the sample size is equal to 1k. One possible explanation for this behavior is that reliable estimation of higher-order statistics typically requires a substantially larger number of samples compared to second-order statistics~\citep{hyvarinen2004independent}.
\end{Remark}

\begin{figure}[htp]
  \setlength{\abovecaptionskip}{0pt}
	\setlength{\belowcaptionskip}{-6pt}
	\vspace{-0.1cm}
	\begin{center}
        \includegraphics[height=0.22\textwidth]{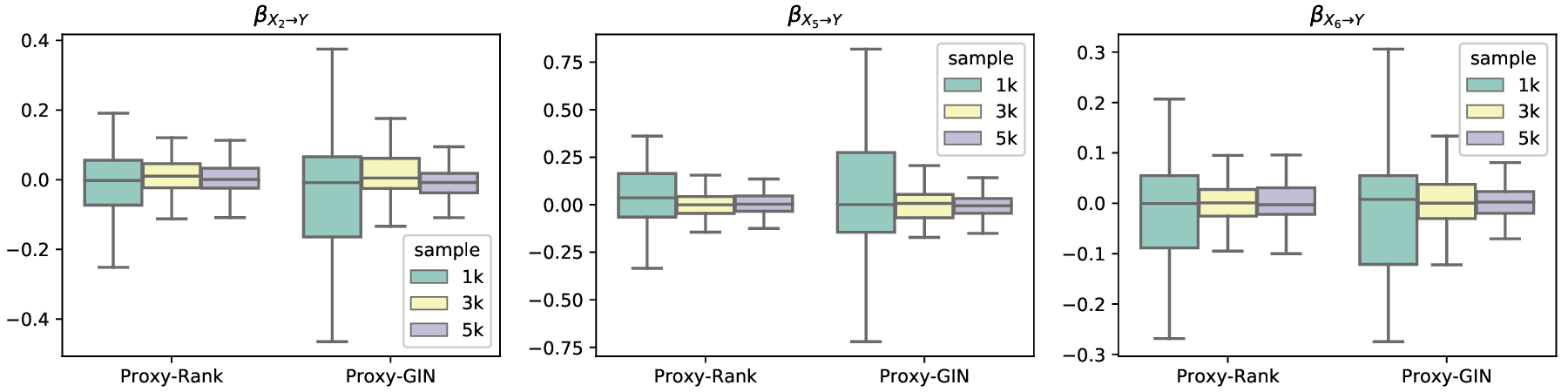}
		\caption{Performance of Proxy-Rank in Gaussian data and Proxy-GIN in Non-Gaussian.}
		\vspace{-0.3cm}
		\label{Appendix-Fig-simulation-Gaussian-results} 
	\end{center}
\end{figure}

\section{More Details of Real-World Application}\label{Appendix-Application}

In this section, we apply the proposed methods to analyze the causal effects of gene expressions on the body weight of F2 mice using the mouse obesity dataset as described by \citet{wang2006genetic}. The dataset we used comprises 17 gene expressions that are known to potentially influence mouse weight, as reported by \citet{lin2015regularization}. Additionally, it includes body weight as the outcome variable and data collected from 227 mice.
As discussed in \citet{miao2022identifying}, gene expression studies like this one may encounter unmeasured confounding issues stemming from batch effects or unobserved phenotypes. Diverging from the approach taken by \citet{miao2022identifying}, we intentionally refrained from incorporating five additional potential instrumental variables from the raw data as prior knowledge in our analysis. This choice was made to underscore the superiority of the proposed algorithm.

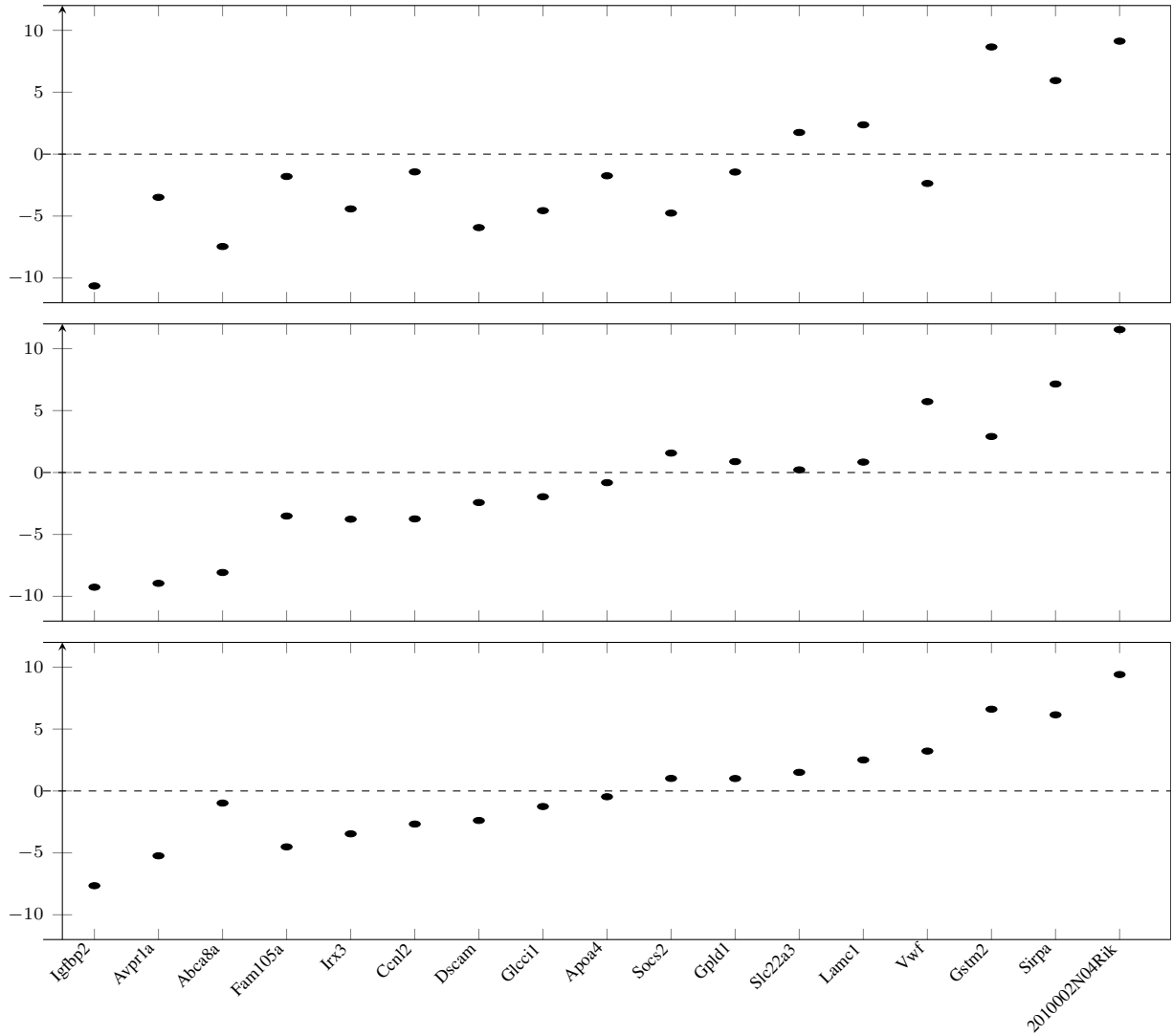
\begin{figure*}[htp!]
  \centering
  \begin{tikzpicture}
    \begin{axis}[
      xscale=1.8,
      height=0.25\textheight, 
      width=0.6\textwidth, 
      grid=both, 
      xtick={1,3,5,7,9,11,13,15,17,19,21,23,25,27,29,31,33}, 
      xticklabels={}, 
      ytick={-10,-5,0,5,10}, 
      ymin=-12, 
      ymax=12,  
      minor grid style={draw=none}, 
      major grid style={draw=none}, 
      axis y line=center, 
      x tick label style={
        font=\scriptsize, 
        rotate=45, 
        anchor=east, 
      },
      y tick label style={
        font=\scriptsize, 
      },
      enlarge x limits=0.05, 
    ]

    \addplot[only marks,
    mark size=1.2pt, 
    ] coordinates {
      (1, -10.64351225)
      (3, -3.488313263)
      (5, -7.461154244)
      (7, -1.80354693)
      (9, -4.42358821)
      (11, -1.433659826)
      (13, -5.934098951)
      (15, -4.563296529)
      (17, -1.7489)
      (19, -4.762170165)
      (21, -1.44844)
      (23, 1.75019)
      (25, 2.36627)
      (27, -2.368145)
      (29, 8.654)
      (31, 5.943)
      (33, 9.1238731)
    };

    \draw [dashed] (axis cs:\pgfkeysvalueof{/pgfplots/xmin},0) -- (axis cs:\pgfkeysvalueof{/pgfplots/xmax},0);

    \draw (axis cs:\pgfkeysvalueof{/pgfplots/xmax},\pgfkeysvalueof{/pgfplots/ymin}) -- (axis cs:\pgfkeysvalueof{/pgfplots/xmax},\pgfkeysvalueof{/pgfplots/ymax});
    \end{axis}
  \end{tikzpicture}\\
  \begin{tikzpicture}
    \begin{axis}[
      xscale=1.8,
      height=0.25\textheight, 
      width=0.6\textwidth, 
      grid=both, 
      xtick={1,3,5,7,9,11,13,15,17,19,21,23,25,27,29,31,33}, 
      xticklabels={}, 
      ytick={-10,-5,0,5,10}, 
      ymin=-12, 
      ymax=12,  
      minor grid style={draw=none}, 
      major grid style={draw=none}, 
      axis y line=center, 
      x tick label style={
        font=\scriptsize, 
        rotate=45, 
        anchor=east, 
      },
      y tick label style={
        font=\scriptsize, 
      },
      enlarge x limits=0.05, 
    ]

    \addplot[only marks,
    mark size=1.2pt, 
    ] coordinates {
      (1, -9.26265)
  (3, -8.95)
  (5, -8.07538)
  (7, -3.51623534)
  (9, -3.7700917)
  (11, -3.746280013)
  (13, -2.424537553)
  (15, -1.95757)
  (17, -0.8204804)
  (19, 1.568622)
  (21, 0.87884)
  (23, 0.2136752)
  (25, 0.83977)
  (27, 5.72183943)
  (29, 2.908324)
  (31, 7.145861599)
  (33, 11.54038538)
    };

    \draw [dashed] (axis cs:\pgfkeysvalueof{/pgfplots/xmin},0) -- (axis cs:\pgfkeysvalueof{/pgfplots/xmax},0);

    \draw (axis cs:\pgfkeysvalueof{/pgfplots/xmax},\pgfkeysvalueof{/pgfplots/ymin}) -- (axis cs:\pgfkeysvalueof{/pgfplots/xmax},\pgfkeysvalueof{/pgfplots/ymax});
    \end{axis}
  \end{tikzpicture}\\
  \begin{tikzpicture}
    \begin{axis}[
      xscale=1.8,
      height=0.25\textheight, 
      width=0.6\textwidth, 
      grid=both, 
      xtick={1,3,5,7,9,11,13,15,17,19,21,23,25,27,29,31,33}, 
      xticklabels={
        Igfbp2, Avpr1a, Abca8a, Fam105a, Irx3, Ccnl2, Dscam, Glcci1, Apoa4, Socs2, Gpld1, Slc22a3, Lamc1, Vwf, Gstm2, Sirpa, 2010002N04Rik
      }, 
      ytick={-10,-5,0,5,10}, 
      ymin=-12, 
      ymax=12,  
      minor grid style={draw=none}, 
      major grid style={draw=none}, 
      axis y line=center, 
      x tick label style={
        font=\scriptsize, 
        rotate=45, 
        anchor=east, 
      },
      y tick label style={
        font=\scriptsize, 
      },
      enlarge x limits=0.05, 
    ]

    \addplot[only marks,
    mark size=1.2pt, 
    ] coordinates {
(1, -7.651501919)
  (3, -5.236321651)
  (5, -0.977587633)
  (7, -4.514230122)
  (9, -3.46189616)
  (11, -2.672427166)
  (13, -2.382747491)
  (15, -1.255231874)
  (17, -0.46671752)
  (19, 1.013338162)
  (21, 1.004281978)
  (23, 1.5022134)
  (25, 2.5042853)
  (27, 3.220533)
  (29, 6.6023354)
  (31, 6.15034105)
  (33, 9.403059593)
    };

    \draw [dashed] (axis cs:\pgfkeysvalueof{/pgfplots/xmin},0) -- (axis cs:\pgfkeysvalueof{/pgfplots/xmax},0);

    \draw (axis cs:\pgfkeysvalueof{/pgfplots/xmax},\pgfkeysvalueof{/pgfplots/ymin}) -- (axis cs:\pgfkeysvalueof{/pgfplots/xmax},\pgfkeysvalueof{/pgfplots/ymax});

    \end{axis}
  \end{tikzpicture}
  \caption{Causal Effect estimates for 17 gene expressions on body weight. The first panel depicts Proxy-Rank estimation, the second displays Proxy-GIN estimation, and the last one showcases NAIVE estimation.}
  \label{figure-realdata}
\end{figure*}

Following the analysis conducted by Miao et al., we assume that there is only one latent variable underlying the common influence, and the data generation mechanism adheres to a linear causal model. Figure \ref{figure-realdata} presents the causal effects of the 17 genes on mouse weight with our proposed methods and the NAIVE method. 
We observed that the majority of our findings align with those presented by \citet{miao2022identifying}. For instance, the gene expressions \emph{Gstm2}, \emph{Sirpa}, and \emph{2010002N04Rik} exhibit positive and significant effects on body weight, whereas the gene expression \emph{Dscam} demonstrates a negative impact on body weight. Furthermore, some of our conclusions coincide with prior research results. In particular, \emph{Igfbp2} (Insulin-like growth factor binding protein 2) displays negative and significant effects on body weight, attributable to its role in mitigating the development of obesity, as supported by \citet{wheatcroft2007igf}. Similarly,\emph{Irx3} (Iroquois homebox gene 3) exhibits negative and significant effects on body weight, which can be attributed to its association with lifestyle changes and its pivotal role in weight regulation through energy balance, as elucidated in \citet{schneeberger2019irx3}.

\section{Proofs}\label{Appendix-Section-Proofs}

Before we proceed with presenting the proofs of our results, we require a few additional theorems and definitions.

\begin{Definition}[Trek]
    A \textit{trek} in $\mathcal{G}$ from $i$ to $j$ is an ordered pair of directed paths $(P_1, P_2)$ where $P_1$ has sink $i$, $P_2$ has sink $j$, and both $P_1$ and $P_2$ have the same source $k$. The common source $k$ is called the \textit{top} of the trek, denoted $\operatorname{top}(P_1,P_2)$.
    Note that one or both of $P_1$ and $P_2$ may consist of a single vertex, that is, a path with no edges. 
\end{Definition}
\begin{Definition}[{trek-separation (t-separation)}]
Let $\mathbf{A},\mathbf{B},\mathbf{C_{A}}$, and $\mathbf{C_{B}}$ be four variable subsets. We say the order pair $(\mathbf{C_{A}}, \mathbf{C_{B}})$ {t-separates} $\mathbf{A}$ from $\mathbf{B}$ if for every trek $(P_1;P_2)$ from a vertex in $\mathbf{A}$ to a vertex in $\mathbf{B}$, either $P_1$ contains a vertex in $\mathbf{C_{A}}$ or $P_2$ contains a vertex in $\mathbf{C_{B}}$.
\end{Definition}
Note that the notion of t-separation is a more general separation criterion than d-separation in a graph (See Theorem 2.11 in \cite{Sullivant-T-separation}). \citet{Sullivant-T-separation} characterized the vanishing determinants of a cross-covariance matrix by using the notion of t-separation.

\begin{Theorem}[Graphical Representation of Rank Constraints]Let $\mathcal{G}$ be a linear directed graphical model and $\mathbf{A}$ and $\mathbf{B}$ be two subsets of the variables in $\mathcal{G}$. The $\mathrm{rk}(\boldsymbol{\Sigma}_{\mathbf{A}, \mathbf{B}})$ less than or equal to $r$ for all covariance matrices consistent with the graph $\mathcal{G}$ if and only if there exist subsets $\mathbf{C}_\mathbf{A},\mathbf{C}_\mathbf{B}$ with $|\mathbf{C}_\mathbf{A}|+|\mathbf{C}_\mathbf{B}| \leq r$ such that $(\mathbf{C_{A}}, \mathbf{C_{B}})$ t-separates $\mathbf{A}$ from $\mathbf{B}$. 
\end{Theorem}
We now quote Darmois-Skitovitch theorem \citep{darmois1953analyse,skitovitch1953property}.
\begin{Theorem}[Darmois-Skitovitch Theorem]
\label{Theorem-DS}
Define two random variables ${V_1}$ and ${V_2}$ as linear combinations of independent random variables ${\varepsilon_i}(i = 1,...,n)$:
\begin{flalign}
V_1 = \sum_{i = 1}^n {{\alpha _i}} {\varepsilon_i}, \quad {V_2} = \sum_{i = 1}^n {{\beta _i}} {\varepsilon_i}.
\end{flalign}
Then, if ${V_1}$ and ${V_2}$ are statistically independent, all variables ${\varepsilon_j}$ for which ${\alpha _j}{\beta _j} \ne 0$ are Gaussian. In other words, if there exists a non-Gaussian ${\varepsilon_j}$ for which ${\alpha _j}{\beta _j} \ne 0$, ${V_1}$ and ${V_2}$ are dependent.
\end{Theorem}

We next introduce the graphical representation of GIN constraints as presented in \citep{xie2023generalized}.
\begin{Theorem}[Graphical Representation of GIN Constraints Theorem]
    Let $\mathcal{G}$ be a linear directed graphical model. Let $\mathcal{Y}$, $\mathcal{Z}$ be two sets of observed variables in $\mathcal{G}$. 
$(\mathcal{Y}$, $\mathcal{Z})$ satisfies the GIN condition (while with the same $\mathcal{Z}$, no proper subset of $\mathcal{Y}$ does) if and only if there exists a $\mathcal{S}$ with $0\leq |\mathcal{S}| \leq \textrm{min}(|\mathcal{Y}|-1, |\mathcal{Z}|)$ such that 1) the order pair $(\emptyset, \mathcal{S})$ t-separates $\mathcal{Z}$ and $\mathcal{Y}$, and that 2) the covariance matrix of $\mathcal{S}$ and $\mathcal{Z}$ has rank $s$, and so does that of $\mathcal{S}$ and $\mathcal{Y}$.
\end{Theorem}

\subsection{Proof of Proposition \ref{Pro-Proxy-Estimator}}
\begin{proof}
The proof can be found in \citet{kuroki2014measurement} or \citet{miao2018proxy} when an unmeasured confounder exists in a linear causal model.
\end{proof}

\subsection{Proof of Proposition \ref{Pro-Multi-Proxy-Estimator}}
Here, we offer two methods of proof. The first utilizes the back door criterion and the conditional instrumental variable approach. The second utilizes the properties of the Trek rules \citep{Sullivant-T-separation}. The details are as follows.

\begin{proof}
We initially define $\boldsymbol{\Sigma}_{(\mathbf{A},\mathbf{B}) \cdot \mathbf{C}} = \boldsymbol{\Sigma}_{\mathbf{A},\mathbf{B}} - \boldsymbol{\Sigma}_{\mathbf{A},\mathbf{C}} \boldsymbol{\Sigma}_{\mathbf{C},\mathbf{C}}^{-1} 
\boldsymbol{\Sigma}_{\mathbf{C},\mathbf{B}}$. In this context, $\boldsymbol{\Sigma}_{(\mathbf{A},\mathbf{B}) \cdot \mathbf{C}}$ can be interpreted as the conditional covariance matrices of $\mathbf{A}$ and $\mathbf{B}$ given $\mathbf{C}$, and $\boldsymbol{\Sigma}_{\mathbf{C},\mathbf{C}}^{-1}$ represents the inverse of $\boldsymbol{\Sigma}_{\mathbf{C},\mathbf{C}}$. We then apply the \emph{back door criterion} and obtain:
\begin{equation}\label{Eq-pro2-1}
    \begin{aligned}
    \beta_{X_k \to Y} = \frac{\boldsymbol{\Sigma}_{(X_{k},Y)\cdot\mathbf{U}}}{\boldsymbol{\Sigma}_{(X_{k},X_{k})\cdot\mathbf{U}}} 
\end{aligned}
\end{equation}
Hence, we have 
\begin{equation}
    \begin{aligned}\label{Eq-pro2-1}
    \left(\boldsymbol{\Sigma}_{X_k,X_k} - \boldsymbol{\Sigma}_{ X_k, \mathbf{U}}\boldsymbol{\Sigma}_{ \mathbf{U}, \mathbf{U}}^{-1}\boldsymbol{\Sigma}_{ \mathbf{U}, X_k}\right)\beta_{X_k \to Y} = \boldsymbol{\Sigma}_{X_k,Y} - \boldsymbol{\Sigma}_{ X_k, \mathbf{U}}\boldsymbol{\Sigma}_{ \mathbf{U}, \mathbf{U}}^{-1}\boldsymbol{\Sigma}_{ \mathbf{U}, Y}.
    \end{aligned}
\end{equation}
According to the proximal criteria, i.e., $\mathbf{W} \CI (X_k, \mathbf{Z}) | \mathbf{U}$, the following two conditions hold: $\mathbf{Z} \CI \mathbf{W} | \mathbf{U}$ and $\mathbf{W} \CI X_k | \mathbf{U}$. Consequently, this will imply that $\boldsymbol{\Sigma}_{(\mathbf{Z},\mathbf{W}) \cdot \mathbf{U}}=\mathbf{0}$ and $\boldsymbol{\Sigma}_{(\mathbf{W}, X_k) \cdot \mathbf{U}}=\mathbf{0}$. Let's expand the above three equations to obtain:
\begin{align}
    \boldsymbol{\Sigma}_{\mathbf{Z},\mathbf{W}} =\boldsymbol{\Sigma}_{\mathbf{Z},\mathbf{U}} \boldsymbol{\Sigma}_{\mathbf{U},\mathbf{U}}^{-1} 
\boldsymbol{\Sigma}_{\mathbf{U},\mathbf{W}},\label{Eq-pro2-2}\\
    \boldsymbol{\Sigma}_{\mathbf{W},X_k} = \boldsymbol{\Sigma}_{\mathbf{W},\mathbf{U}} \boldsymbol{\Sigma}_{\mathbf{U},\mathbf{U}}^{-1} 
\boldsymbol{\Sigma}_{\mathbf{U},X_k} \label{Eq-pro2-3}
\end{align}
By solving the above Equations \ref{Eq-pro2-2}$\sim$\ref{Eq-pro2-3}, we obtain
\begin{align}\label{Eq-pro2-4}
    \boldsymbol{\Sigma}_{ X_k, \mathbf{W}}\boldsymbol{\Sigma}_{ \mathbf{Z}, \mathbf{W}}^{-1} = \boldsymbol{\Sigma}_{ X_k, \mathbf{U}}\boldsymbol{\Sigma}_{ \mathbf{Z}, \mathbf{U}}^{-1}
\end{align}
To verify the conclusion, we next consider the following two scenarios. In Scenario 1, we assume independence between $\mathbf{Z}$ and $X_k$ given $\mathbf{U}$, while in Scenario 2, we assume dependence of $\mathbf{Z}$ on $X_k$ given $\mathbf{U}$, as illustrated in Figure \ref{Fig-proof-proposition-example} below. It is worth noting that, although from the graphical representation, Scenario 2 encompasses the situation in Scenario 1, they are not inclusive from a proof perspective. Specifically, Scenario 1 relies on the condition $\mathbf{Z} \CI X_k | \mathbf{U}$, while the proof in Scenario 2 capitalizes on the property of conditional instrumental variables. In other words, in Scenario 2, given the condition $\mathbf{U}$, $\mathbf{Z}$ can serve as the instrumental variable set for the causal relationship $X_k \to Y$, whereas $\mathbf{Z}$ is not the valid instrumental variable set in Scenario 1. In summary, although the proof strategies for both scenarios are independent, it is intriguing that they share the same expression for the causal effect $X_k \to Y$.
\begin{figure}[htp]
	\begin{center}
	\begin{tikzpicture}[scale=1.2, line width=0.5pt, inner sep=0.2mm, shorten >=.1pt, shorten <=.1pt]
		\draw (3, 1.0) node(U) [circle, fill=gray!60, minimum size=0.5cm,draw] {{\footnotesize\,$\mathbf{U}$\,}};
		\draw (1.5, 1.0) node(Z) [] {{\footnotesize\,$\mathbf{Z}$\,}};
		\draw (4.5, 1.0) node(W) [] {{\footnotesize\,$\mathbf{W}$\,}};
		\draw (2.25, 0) node(X) [] {{\footnotesize\,{$X_k$}\,}};
		\draw (3.75, 0) node(Y) [] {{\footnotesize\,{$Y$}\,}};
		\draw[-arcsq] (U) -- (Z) node[pos=0.5,sloped,above] {}; 
		\draw[-arcsq] (U) -- (W)node[pos=0.5,sloped,above] {}; 
		\draw[-arcsq] (U) -- (X) node[pos=0.5,sloped,above] {};
		\draw[-arcsq] (U) -- (Y) node[pos=0.5,sloped,above] {};
            \draw[-arcsq] (W) -- (Y) node[pos=0.5,sloped,above] {};
		\draw[-arcsq] (X) -- (Y) node[pos=0.5,sloped,above] {};
            \draw (3,-0.5) node(label-ii) [] {{\footnotesize\,(a)\,}};
	\end{tikzpicture}~~~
        \begin{tikzpicture}[scale=1.2, line width=0.5pt, inner sep=0.2mm, shorten >=.1pt, shorten <=.1pt]
		\draw (3, 1.0) node(U) [circle, fill=gray!60, minimum size=0.5cm,draw] {{\footnotesize\,$\mathbf{U}$\,}};
		\draw (1.5, 1.0) node(Z) [] {{\footnotesize\,$\mathbf{Z}$\,}};
		\draw (4.5, 1.0) node(W) [] {{\footnotesize\,$\mathbf{W}$\,}};
		\draw (2.25, 0) node(X) [] {{\footnotesize\,{$X_k$}\,}};
		\draw (3.75, 0) node(Y) [] {{\footnotesize\,{$Y$}\,}};
		\draw[-arcsq] (U) -- (Z) node[pos=0.5,sloped,above] {}; 
		\draw[-arcsq] (U) -- (W)node[pos=0.5,sloped,above] {}; 
		\draw[-arcsq] (U) -- (X) node[pos=0.5,sloped,above] {};
		\draw[-arcsq] (U) -- (Y) node[pos=0.5,sloped,above] {};
            \draw[-arcsq] (Z) -- (X) node[pos=0.5,sloped,above] {};
            \draw[-arcsq] (W) -- (Y) node[pos=0.5,sloped,above] {};
		\draw[-arcsq] (X) -- (Y) node[pos=0.5,sloped,above] {};
            \draw (3,-0.5) node(label-ii) [] {{\footnotesize\,(b)\,}};
	\end{tikzpicture}
	\caption{(a) Scenario 1: $\mathbf{Z} \CI X_k | \mathbf{U}$, and (b) Scenario 2: $\mathbf{Z} \nCI X_k | \mathbf{U}$}
	\vspace{-0.4cm}
	\label{Fig-proof-proposition-example}
	\end{center}
\end{figure}
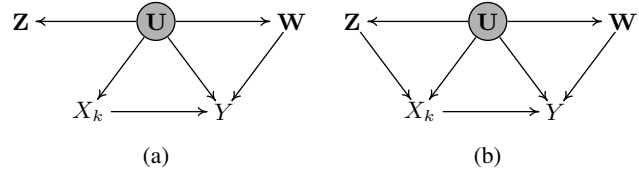

\noindent\textbf{{Scenario 1:}} 
Because $\mathbf{Z} \CI X_k | \mathbf{U}$, and according to the proximal criteria, i.e., $\mathbf{Z} \CI {Y} | (\mathbf{U},X_k)$, we further have $\mathbf{Z} \CI Y | \mathbf{U}$. Consequently, this will imply that
$\boldsymbol{\Sigma}_{(\mathbf{Z}, X_k) \cdot \mathbf{U}}=\mathbf{0}$ and $\boldsymbol{\Sigma}_{(\mathbf{Z},Y) \cdot \mathbf{U}}=\mathbf{0}$.
Thus, we obtain 
\begin{align}
    \boldsymbol{\Sigma}_{\mathbf{Z},X_k} = \boldsymbol{\Sigma}_{\mathbf{Z},\mathbf{U}} \boldsymbol{\Sigma}_{\mathbf{U},\mathbf{U}}^{-1} 
\boldsymbol{\Sigma}_{\mathbf{U},X_k} \label{Eq-pro2-5}\\
    \boldsymbol{\Sigma}_{\mathbf{Z},Y} = \boldsymbol{\Sigma}_{\mathbf{Z},\mathbf{U}} \boldsymbol{\Sigma}_{\mathbf{U},\mathbf{U}}^{-1} 
\boldsymbol{\Sigma}_{\mathbf{U},Y},\label{Eq-pro2-6}
\end{align}
By solving the above Equations \ref{Eq-pro2-4}$\sim$\ref{Eq-pro2-6}, we obtain
\begin{align}\label{Eq-pro2-7}
    \boldsymbol{\Sigma}_{X_k, \mathbf{W}}\boldsymbol{\Sigma}_{ \mathbf{Z}, \mathbf{W}}^{-1}\boldsymbol{\Sigma}_{ \mathbf{Z}, X_k} = \boldsymbol{\Sigma}_{ X_k, \mathbf{U}}\boldsymbol{\Sigma}_{ \mathbf{U}, \mathbf{U}}^{-1}\boldsymbol{\Sigma}_{ \mathbf{U}, X_k}, \\
    \boldsymbol{\Sigma}_{ X_k, \mathbf{W}}\boldsymbol{\Sigma}_{ \mathbf{Z}, \mathbf{W}}^{-1}\boldsymbol{\Sigma}_{ \mathbf{Z}, Y} = \boldsymbol{\Sigma}_{ X_k, \mathbf{U}}\boldsymbol{\Sigma}_{ \mathbf{U}, \mathbf{U}}^{-1}\boldsymbol{\Sigma}_{ \mathbf{U}, Y} \label{Eq-pro2-8}
\end{align}
By combing Equations \ref{Eq-pro2-1}, \ref{Eq-pro2-7}, and \ref{Eq-pro2-8}, we have
\begin{equation}
    \begin{aligned}
    \left(\boldsymbol{\Sigma}_{X_k,X_k} - \boldsymbol{\Sigma}_{X_k, \mathbf{W}}\boldsymbol{\Sigma}_{ \mathbf{Z}, \mathbf{W}}^{-1}\boldsymbol{\Sigma}_{ \mathbf{Z}, X_k}\right)\beta_{X_k \to Y} = \boldsymbol{\Sigma}_{X_k,Y} - \boldsymbol{\Sigma}_{ X_k, \mathbf{W}}\boldsymbol{\Sigma}_{ \mathbf{Z}, \mathbf{W}}^{-1}\boldsymbol{\Sigma}_{ \mathbf{Z}, Y}.
    \end{aligned}
\end{equation}
Hence, we finally have
\begin{equation}
    \begin{aligned}
  \left( {\mathrm{det}(\boldsymbol{\Sigma}_{\{X_k \cup \mathbf{Z}\}, \{X_k \cup \mathbf{W}\}})} \right) \beta_{X_k \to Y} & = {\mathrm{det}(\boldsymbol{\Sigma}_{\{X_k \cup \mathbf{Z}\}, \{Y \cup \mathbf{W}\}})},
    \end{aligned}
\end{equation}
which is consistent with the Equation \ref{Eq-multiple-proximal-inference}.

\noindent\textbf{{Scenario 2:}} We here apply the \emph{conditional instrumental variable approach} and obtain:
\begin{equation}
\begin{aligned}
    &\beta_{X_k \to Y} = \frac{\boldsymbol{\Sigma}_{Z_{i}Y\cdot\mathbf{U}}}{\boldsymbol{\Sigma}_{Z_{i}X_{k}\cdot\mathbf{U}}}, \quad Z_{i} \in \mathbf{Z}.
\end{aligned}
\end{equation}
That is,
\begin{equation}\label{Eq-Scenario2-1}
    \begin{aligned}
    \left(\boldsymbol{\Sigma}_{Z_i,X_k} - \boldsymbol{\Sigma}_{Z_i, \mathbf{U}}\boldsymbol{\Sigma}_{ \mathbf{U}, \mathbf{U}}^{-1}\boldsymbol{\Sigma}_{ \mathbf{U}, X_k}\right)\beta_{X_k \to Y} = \boldsymbol{\Sigma}_{Z_k,Y} - \boldsymbol{\Sigma}_{ Z_k, \mathbf{U}}\boldsymbol{\Sigma}_{ \mathbf{U}, \mathbf{U}}^{-1}\boldsymbol{\Sigma}_{ \mathbf{U}, Y}.
    \end{aligned}
\end{equation}
Based on the proximal criteria, i.e., $\mathbf{Z} \CI {Y} | (\mathbf{U},X_k)$, we conclude that, for $Z_{i} \in \mathbf{Z}$, $Z_i$ serves as a valid instrumental variable for the causal relationship $X_k \to Y$ given $\mathbf{U}$. Hence, the equation above can be expressed in vector form as:
\begin{equation}
    \begin{aligned}\label{Eq-Scenario2-2}
        \left(\boldsymbol{\Sigma}_{\mathbf{Z}, X_k} - \boldsymbol{\Sigma}_{\mathbf{Z}, \mathbf{U}}\boldsymbol{\Sigma}_{ \mathbf{U}, \mathbf{U}}^{-1}\boldsymbol{\Sigma}_{ \mathbf{U}, X_k}\right)\beta_{X_k \to Y} = \boldsymbol{\Sigma}_{\mathbf{Z}, Y} - \boldsymbol{\Sigma}_{\mathbf{Z}, \mathbf{U}}\boldsymbol{\Sigma}_{ \mathbf{U}, \mathbf{U}}^{-1}\boldsymbol{\Sigma}_{ \mathbf{U}, Y}.
    \end{aligned}
\end{equation}
By solving Equations \ref{Eq-pro2-1}, \ref{Eq-pro2-4}, and \ref{Eq-Scenario2-2} for $\beta_{X_k \to Y}$, we obtain
\begin{equation}
    \begin{aligned}
    \left(\boldsymbol{\Sigma}_{X_k,X_k} - \boldsymbol{\Sigma}_{X_k, \mathbf{W}}\boldsymbol{\Sigma}_{ \mathbf{Z}, \mathbf{W}}^{-1}\boldsymbol{\Sigma}_{ \mathbf{Z}, X_k}\right)\beta_{X_k \to Y} = \boldsymbol{\Sigma}_{X_k,Y} - \boldsymbol{\Sigma}_{ X_k, \mathbf{W}}\boldsymbol{\Sigma}_{ \mathbf{Z}, \mathbf{W}}^{-1}\boldsymbol{\Sigma}_{ \mathbf{Z}, Y}.
    \end{aligned}
\end{equation}
This will imply that 
\begin{equation}
    \begin{aligned}
  \left( {\mathrm{det}(\boldsymbol{\Sigma}_{\{X_k \cup \mathbf{Z}\}, \{X_k \cup \mathbf{W}\}})} \right) \beta_{X_k \to Y} & = {\mathrm{det}(\boldsymbol{\Sigma}_{\{X_k \cup \mathbf{Z}\}, \{Y \cup \mathbf{W}\}})},
    \end{aligned}
\end{equation}
which is consistent with the Equation \ref{Eq-multiple-proximal-inference}.
\end{proof}

We hereby present an alternative proof strategy employing Trek rules. Before delving into the proof of this result, we initially introduce two definitions and a theorem that play a crucial role in our argument.

\begin{Definition}[Trek System \& Intersection]
    Let $A$ and $B$ be two subsets of the vertex set of a DAG $\mathcal{G}$,
    with $|A| = |B|$.  A \emph{system of treks} from $A$ to $B$ is a set of treks that each are between a vertex in $A$ and a vertex in $B$.  Let $\mathcal{T}$ be such a system.  Then $\mathcal{T}$ has \emph{no sided intersection} if any two distinct treks in $\mathcal{T}$ have disjoint left sides and
    disjoint right sides.
\end{Definition}
\begin{Definition}[Trek Rule \citep{Sullivant-T-separation}]
    Let $\Lambda=(\lambda_{ij})\in\mathbb{R}^\mathcal{D}$ be the cofficients matrix and
$\Omega=(\sigma_{ij})$ be the variance of noise term.  To any trek $\tau$, associate a {\em trek monomial}
\begin{equation}\label{eq:trek-monomial}
  \sigma(\tau) \;=\; \sigma_{i_1j_1}\prod_{k=1}^{\ell-1}
  \lambda_{i_ki_{k+1}}\prod_{k=1}^{r-1} \lambda_{j_kj_{k+1}}. 
\end{equation}
\end{Definition}

\begin{Theorem}[\citep{drton2020nested,draisma2013positivity}] \label{the:trekrule}Suppose the underlying
  graph $\mathcal{G}$ is acyclic.  Then the determinant of $\boldsymbol{\Sigma}_{A, B}$ equals $\mathcal P_{A, B}$, and $\mathcal P_{A, B}$ is defined as 
  \begin{equation}
  \label{eq:PAC}
  \mathcal P_{A, B} \;=\; \sum  (-1)^\mathcal{T}
  \prod_{\tau\in\mathcal{T}} \sigma(\tau)
\end{equation}
with the summation being over all systems of treks $\mathcal{T}$ from
$A$ to $B$ with no sided intersection.
\end{Theorem}

Now, we prove the Proposition \ref{Pro-Multi-Proxy-Estimator} based on the above theorem.
\begin{proof}
    Let us first consider the trek system between $\mathbf{Z}$ and $\mathbf{W}$. Based on the model definition, we know that $\mathbf{Z}$ and $\mathbf{W}$ are the child set of confounders $\mathbf{U}$ with $|\mathbf{Z}| = |\mathbf{W}| = |\mathbf{U}|=q$. Since $\mathbf{Z} \CI \mathbf{W} | \mathbf{U}$, there are no direct path from $\mathbf{Z}$ to $\mathbf{W}$ but $Z_i \leftarrow U_i \to W_i$ is only trek between $Z_i$ and $W_i$. Thus, the trek system without sided intersection between $\mathbf{Z}$ and $\mathbf{W}$ must have the source $\mathbf{U}$ according to \textit{Pigeonhole principle}. That is any two treks $\mathcal{\tau}_i$ and $\mathcal{\tau}_j$ has source $U_i$ and $U_j$ respectively, $U_i \neq U_j$ in all trek system (note that if $U_i = U_j$ in a trek system, this trek system has sided intersection in source $U_i$).

    Now, consider the numerator term of Eq(3), i.e., the determinant of covariance matrice between $A = \{X_k, \mathbf{Z}\}$ and $B = \{Y, \mathbf{W}\}$. To do so, by Theorem \ref{the:trekrule}, it must discuss the trek system between $A$ and $B$ which is no side intersection. There are two cases for $A$ and $B$: Case I: $\mathbf{Z} \CI X_k |\mathbf{U}$; and Case II: $\mathbf{Z} \not\CI X_k |\mathbf{U}$.

    \textbf{Case I:} $\mathbf{Z} \CI X_k |\mathbf{U}$, i.e., there are no edges between $\mathbf{Z}$ and $X_k$. For the trek system $\mathcal{T}$ between $A$ and $B$ with $|A|=|B|=q+1$, there are $q+1$ trek in $\mathcal{T}$. if $\mathcal{T}$ is no side intersection, the trek with the sink node $X$ on the left side and sink node $Y$ on the right side, denote as $\mathcal{\tau}_i$, $Top(\mathcal{\tau}_i) \notin \mathbf{U}$ (i.e., the source of this trek can no be $U_i$). Otherwise, there exist two treks in the trek system that have a common source that violates the no-sided intersection (according to the above analysis in $\mathbf{Z}$ and $\mathbf{W}$). For instance, the trek between $X_k$ and $Y$, $(U_i \to X_k, U_i \to Y)$, is intersecting with one of the trek between $\mathbf{Z}$ and $\mathbf{W}$. \textit{An illustrative example is given in Example 1}. Therefore, the trek between $\{X_k, Y\}$ must be $(X_k; X_k \rightarrow Y)$ in the no side-intersection trek system between $A$ and $B$, and meanwhile, other $q$ trek between $\mathbf{Z}$ and $\mathbf{W}$ has source $\mathbf{U}$. According to Theorem \ref{the:trekrule}, denote the trek between $X_k$ and $Y$ as $\tau_{\{X_k,Y\}}$, the determinant equals
    \begin{align}
        \mathcal P_{A, B} = \sum  (-1)^\mathcal{T}
  \prod_{\tau\in\mathcal{T}\setminus \tau_{\{X_k,Y\}}} \sigma(\tau) \cdot \sigma(\tau_{\{X_k,Y\}}).
    \end{align}
    Since $\sigma(\tau_{\{X_k,Y\}})= \sigma(X_i)\beta_{X_k\rightarrow Y}$ (by \textit{Trek Rule}), the above equation can be rewritten as
    \begin{align}
        \mathcal P_{A, B} = \beta_{X_k\rightarrow Y} \sigma(X_i)\sum  (-1)^\mathcal{T}
  \prod_{\tau\in\mathcal{T}\setminus \tau_{\{X_k,Y\}}} \sigma(\tau).
    \end{align}

  To show $\beta_{X_k\rightarrow Y}$ can be unbiasedly estimated, now we consider the denominator term of Eq. (3). Similarly, for two vectors $C = \{X_k, \mathbf{Z}\}$ and $D = \{X_k, \mathbf{W}\}$, the determinant can be formalized as 
  \begin{align}
      \mathcal P_{C, D} = \sigma(X_i) \sum (-1)^\mathcal{T'}\prod_{\tau\in\mathcal{T}\setminus \tau_{\{X_k,Y\}}} \sigma(\tau),
  \end{align}
  where the covariance of the trek between $X_k$ and $X_k$ equals $\sigma(X_i)$. Based on the above analysis, we can get the unbiased estimation of $\beta_{X_k\rightarrow Y}$ by the ratio of two determinants, i.e., $\mathcal P_{A, B} / \mathcal P_{C, D}$.

  \textbf{Case II:} $\mathbf{Z} \not\CI X_k |\mathbf{U}$, i.e., there exits a edges between $Z_i$ and $X_k$ for some $ Z_i \in \mathbf{Z}$. According to the proximal criteria, e.g., $\{X_k,\mathbf{Z}\} \CI \mathbf{W} | \mathbf{U}$, there are no other trek between $Z_i$ and $W_i$ except for the $(U_i\to Z_i, U_i \to W_i)$. For this case, the key difference to Case I is that there may exist a trek between $Z_i$ and $Y$ in which the source $Top(Z_i, Y) \not\in \mathbf{U}$. 
  
  There are two cases for the trek between $Z_i$ and $Y$: (i). $\tau_1 = (Z_i, Z_i\rightarrow X_k \rightarrow Y)$ or (ii). $\tau_2 = (X_k\rightarrow Z_i, X_k \rightarrow Y)$. By \textit{Trek Rule}, we have $\sigma(\tau_1) = \sigma(Z_i)\beta_{Z_i \to X_k}\beta_{X_k \to Y}$ for case (i) while $\sigma(\tau_2) = \sigma(X_k)\beta_{X_k \to Z_i}\beta_{X_k \to Y}$ for case (ii). Furthermore, if there are more than one $Z_i$, for example, $Z_i$ and $Z_j$ have the trek of the above cases, then this trek system has a side intersection in the node $X_k$ (as $\sigma(\tau_1)$ and $\sigma(\tau_2)$ has side intersection in the node $X_k$). Thus, a trek system without sided intersection between $\{X_k, \mathbf{Z}\}$ and $\{Y, \mathbf{W}\}$ can only be one trek following the above cases, i.e., only a $Z_i$ follows the above cases. Thus, the set of no side-intersection trek system is $q$ treks with source $\mathbf{U}$ between $\{\mathbf{Z}\setminus{Z_i}, X_k\}$ and $\{\mathbf{W}\}$ and plus a trek between $Z_i$ and $Y$. Denote the trek between $Z_i$ and $Y$ as $\tau_{\{Z_i, Y\}}$, $\tau_{\{Z_i, Y\}}$ follows one of the above cases.

Therefore, by Theorem \ref{the:trekrule}, the determinant can be formalized as 
\begin{equation}
    \begin{aligned}
        \beta_{X_k\rightarrow Y} \sum (-1)^\mathcal{T'}\prod_{\tau' \in\mathcal{T'}} \sigma(\tau') +\ \sum _{{1}}^{q} \tau _{Z_{i}- X_k} \beta _{X_k\rightarrow Y}\sum ( -1)^{\mathcal{T}^{''} }\prod_{\tau\in\mathcal{T^{''}}}  \sigma ( \tau),
    \end{aligned}
\end{equation}
where the first term represents the summation of no sided intersection system including the trek with source $\mathbf{U}$ between $\mathbf{Z}$ and $\mathbf{W}$ and a trek from $X_k$ to $Y$; the second term presents the summation of the no sided intersection system inlcuding the trek with source $\mathbf{U}$ between $\{\mathbf{Z} \setminus Z_i, X_k\}$ and $\mathbf{W}$ and a trek between $Z_i$ and $Y$. Similarly, for two vector $\{X_k, \mathbf{Z}\}$ and $\{X_k, \mathbf{W}\}$, the determinant can be formalized as 
\begin{equation}
    \begin{aligned}
        \sum (-1)^\mathcal{T'}\prod_{\tau' \in\mathcal{T'}} \sigma(\tau') +\ \sum _{{1}}^{q} \tau _{Z_{i} - X_k} \sum ( -1)^{\mathcal{T}^{''} }\prod_{\tau\in\mathcal{T^{''}}}  \sigma ( \tau).
    \end{aligned}
\end{equation}
In the end, we can obtain the unbiased estimation of $\beta_{X_k \rightarrow Y}$ by the ratio of two determinants.
\end{proof}    

\begin{figure}[h]
	\begin{center}
		\begin{tikzpicture}[scale=1.4, line width=0.5pt, inner sep=0.2mm, shorten >=.1pt, shorten <=.1pt]
        \draw [fill=blue!50,thick, fill opacity=0.5, draw=none] (1.0,0.8) ellipse [x radius=0.8cm, y radius=0.35cm];
        \draw [fill=red!50,thick, fill opacity=0.5, draw=none] (-0.25,0.0) ellipse [x radius=0.6cm, y radius=0.3cm];
        \draw [fill=red!50,thick, fill opacity=0.5, draw=none] (2.25,0.0) ellipse [x radius=0.6cm, y radius=0.3cm];
        \draw (1.0, 1.0) node(R21B) [] {{\footnotesize\,$\mathbf{U}$\,}};
		\draw (0.6, 0.8) node(U1) [circle, fill=gray!60, minimum size=0.5cm,draw] {{\footnotesize\,$U_1$\,}};

        \draw (1.4, 0.8) node(U2) [circle, fill=gray!60, minimum size=0.5cm,draw] {{\footnotesize\,$U_2$\,}};
		\draw (-0.0, 0.0) node(Z1) [] {{\footnotesize\,${Z}_{2}$\,}};

		\draw (0.5, 0.0) node(Xk) [] {{\footnotesize\,${X}_k$\,}};

		\draw (1.5, 0.0) node(Y) [] {{\footnotesize\,$Y$\,}};
		\draw (2.0, 0.0) node(W1) [] {{\footnotesize\,$W_1$\,}};
        \draw (-0.5, 0.0) node(Z2) [] {{\footnotesize\,$Z_1$\,}};
        \draw (2.5, 0.0) node(W2) [] {{\footnotesize\,$W_2$\,}};
		\draw[-arcsq] (U1) -- (Xk) node[pos=0.5,sloped,above] {};
		\draw[-arcsq] (U1) -- (Y) node[pos=0.5,sloped,above] {};
		\draw[-arcsq] (U1) -- (Z1) node[pos=0.5,sloped,above] {};
		\draw[-arcsq] (U1) -- (W1) node[pos=0.5,sloped,above] {};
            \draw[-arcsq] (U1) -- (Z2) node[pos=0.5,sloped,above] {};
		\draw[-arcsq] (U1) -- (W2) node[pos=0.5,sloped,above] {};
        \draw[-arcsq] (U2) -- (Xk) node[pos=0.5,sloped,above] {};
		\draw[-arcsq] (U2) -- (Y) node[pos=0.5,sloped,above] {};
		\draw[-arcsq] (U2) -- (Z1) node[pos=0.5,sloped,above] {};
		\draw[-arcsq] (U2) -- (W1) node[pos=0.5,sloped,above] {};
         \draw[-arcsq] (U2) -- (Z2) node[pos=0.5,sloped,above] {};
		\draw[-arcsq] (U2) -- (W2) node[pos=0.5,sloped,above] {};
        \draw[-arcsq] (Xk) -- (Y) node[pos=0.5,sloped,above] {{\scriptsize\,$\beta$\,}}; 

		\end{tikzpicture}

		\caption{Illustration of Proof of Proposition 2.}
		\label{Fig-exp-pro2} 
	\end{center}
	
\end{figure}
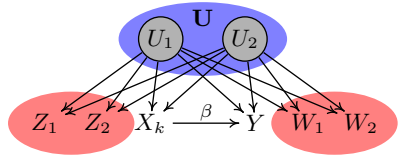

\begin{Example-set}[Illustration of Proof of Proposition \ref{Pro-Multi-Proxy-Estimator}]
Consider a graph in Fig. \ref{Fig-exp-pro2}. There are two confounders $U_1$ and $U_2$ that affect the $X_k$ and $Y$.
Let $\mathbf{Z} = \{Z_1, Z_2\}$ and $\mathbf{W} = \{W_1, W_2\}$.
For two vector $\{X_k, Z_1, Z_2\}$ and $\{Y, W_1, W_2\}$, all no intersection trek system between $\mathbf{Z}$ and $\mathbf{W}$ must has the source $\{U_1, U_2\}$ due to $\mathbf{Z} \CI X_k |\{U_1,U_2\}$, that is, 

\begin{itemize}
    \item ($U_1 \to Z_1$, $U_1 \to W_1$) and ($U_2 \to Z_2$, $U_2 \to W_2$)
    \item ($U_2 \to Z_1$, $U_2 \to W_1$) and ($U_1 \to Z_2$, $U_1 \to W_2$) 
    \item ($U_1 \to Z_2$, $U_1 \to W_1$) and ($U_2 \to Z_1$, $U_2 \to W_2$) 
    \item ($U_2 \to Z_2$, $U_2 \to W_1$) and ($U_1 \to Z_1$, $U_1 \to W_2$) 
\end{itemize}

For the trek between $X_k$ and $Y$, it must be ($X_k, X_k \rightarrow Y$), otherwise, it will be an intersection with the trek between $\mathbf{Z}$ and $\mathbf{W}$ in the source node $\mathbf{U}$. For example, a trek between $X_k$ and $Y$ with source $U_1$, $(U_1 \to X_k, U_1 \to Y)$, is intersecting with ($U_1 \to Z_1$, $U_1 \to W_1$) in the trek system: ($U_1 \to Z_1$, $U_1 \to W_1$), ($U_2 \to Z_2$, $U_2 \to W_2$) and $(U_1 \to X_k, U_1 \to Y)$. Therefore, according to Theorem 1, the determinant equals $\beta_{X_k\rightarrow Y}[ \sigma _{U_{1}}^{2} \sigma _{U_{2}}^{2} \sigma _{X_{k}}^{2}( a_{1} a_{3} b_{2} b_{4} -a_{1} a_{4} b_{2} b_{3} -a_{2} a_{3} b_{1} b_{4} +a_{2} a_{4} b_{1} b_{3})]$, where $a_i$ represent the effect from $U_1$ to $\mathbf{Z} \cup \mathbf{ W}$ while $b_i$ represent the effect from $U_2$ to $\mathbf{Z} \cup \mathbf{ W}$.

For two vectors $\{X_k, Z_1, Z_2\}$ and $\{X_k, W_1, W_2\}$, based on Theorem 1, we also obtain the determinant as $ \sigma _{U_{1}}^{2} \sigma _{U_{2}}^{2} \sigma _{X_{k}}^{2}( a_{1} a_{3} b_{2} b_{4} -a_{1} a_{4} b_{2} b_{3} -a_{2} a_{3} b_{1} b_{4} +a_{2} a_{4} b_{1} b_{3})$. Therefore, one may obtain an unbiased estimation of $\beta_{X_k \to Y}$ by the ratio of two determinants.

\end{Example-set}

\subsection{Proof of Lemma \ref{Lemma-rule1}}

\begin{proof}
Firstly, in accordance with Equation 2, where all entries of the matrix $\mathbf{C}$ are non-zero, and under the faithfulness assumption, we can directly infer that both$\boldsymbol{\Sigma}_{\{X_k, \mathbf{A}\}, \{Y, \mathbf{B}\}}$ and $\boldsymbol{\Sigma}_{\{X_k, \mathbf{A}\}, \{X_k, \mathbf{B}\}}$ are always full rank. Therefore, these necessary conditions are satisfied for any sets $\mathbf{Z}$ and $\mathbf{W}$ (otherwise, it would be required to test these conditions).

Secondly, due to condition (1), namely $\mathrm{rk}(\boldsymbol{\Sigma}_{\{X_k, Q, \mathbf{A}\}, \{X_k, Y, \mathbf{B}}\}) \leq q+1$, and the faithfulness assumption, and based on the "Graphical Representation of Rank Constraints" Theorem, we can assert that there exist subsets $\mathbf{C}_\mathbf{A}$ and $\mathbf{C}_\mathbf{B}$ with $|\mathbf{C}_\mathbf{A}|+|\mathbf{C}_\mathbf{B}| \leq q+1$ such that $(\mathbf{C}_\mathbf{A}, \mathbf{C}_\mathbf{B})$ t-separates $\{X_k, Q, \mathbf{A}\}$ from $\{X_k, Y, \mathbf{B}\}$. According to the data generation process, the treks between $\{Q, A\}$ and $\{Y, B\}$ must go through unmeasured confounder $\mathbf{U}$ (except for $X_k$). Therefore, $\mathbf{C}_\mathbf{A} \cup \mathbf{C}_\mathbf{B} = \{X_k, \mathbf{U}\}$. Since $|{X_k, \mathbf{U}}| = q+1$, we can conclude that all treks between $\{X_k, Q, \mathbf{A}\}$ and $ \{X_k,Y,\mathbf{B}\}$ must go through a node in $ \{X_k, \mathbf{U}\}$. This will imply that $\mathbf{A} \CI {Y} | (\mathbf{U},X_k)$, i.e., condition 1 of proximal criteria holds, and $\mathbf{A} \CI \mathbf{B} | (\mathbf{U},X_k)$.

Furthermore, because of condition 2), i.e.,  $\mathrm{rk}(\boldsymbol{\Sigma}_{\{X_k, \mathbf{A}\},  \{Q, \mathbf{B}\}}) \leq q$, and because according to the "Graphical Representation of Rank Constraints" Theorem, we know that there exist subsets $\mathbf{C}_\mathbf{A},\mathbf{C}_\mathbf{B}$ with $|\mathbf{C}_\mathbf{A}|+|\mathbf{C}_\mathbf{B}| \leq q$ such that $(\mathbf{C_{A}}, \mathbf{C_{B}})$ t-separates $\{X_k, \mathbf{A}\}$ from $ \{Q,\mathbf{B}\}$. 
According to the generation of data (Equation 2), all treks between $\{X_k, \mathbf{A}\}$ and $ \{Q,\mathbf{B}\}$ must go through unmeasured confounders $\mathbf{U}$. Hence, $\mathbf{C}_\mathbf{A} \cup \mathbf{C}_\mathbf{B} = \mathbf{U}$. This will imply that $\mathbf{B} \CI {X_k} | \mathbf{U}$. Because $\mathbf{A} \CI \mathbf{B} | (\mathbf{U},X_k)$, we have  $\mathbf{B} \CI (X_k, \mathbf{A}) | \mathbf{U}$, i.e., condition 2 of proximal criteria holds. 

Based on the above analysis, $\mathbf{A}$ and $\mathbf{B}$ are valid NCE and NCO with respect to $X_k \to Y$, respectively. Furthermore, due to Assumption 2, we know that such sets $\mathbf{A}$ and $\mathbf{B}$ must exist in the system.
\end{proof}

%

\subsection{Proof of Lemma \ref{Lemma-rule2}}

\begin{proof}
The proof strategy for this theorem is similar to the proof strategy for Lemma \ref{Lemma-rule1}.

Firstly, according to Equation 2 (where all entries of matrix $\mathbf{C}$ are non-zero) and the faithfulness assumption, we can directly infer that $\boldsymbol{\Sigma}_{\{X_k, \mathbf{A}\}, \{Y, \mathbf{B}\}}$ and $\boldsymbol{\Sigma}_{\{X_k, \mathbf{A}\}, \{X_k, \mathbf{B}\}}$ are both full rank. Therefore, these necessary conditions always hold for any sets $\mathbf{Z}$ and $\mathbf{W}$ (otherwise, it would be required to test these conditions).

Secondly, due to condition (1), namely $\mathrm{rk}(\boldsymbol{\Sigma}_{\{X_k, \mathbf{A}\},  \{X_k, Y,\mathbf{B}\}}) \leq q+1$, and the faithfulness assumption, and based on the "Graphical Representation of Rank Constraints" Theorem, we can assert that there exist subsets $\mathbf{C}_\mathbf{A}$ and $\mathbf{C}_\mathbf{B}$ with $|\mathbf{C}_\mathbf{A}|+|\mathbf{C}_\mathbf{B}| \leq q+1$ such that $(\mathbf{C_{A}}, \mathbf{C_{B}})$ t-separates $\{X_k, \mathbf{A}\}$ from $ \{X_k, Y,\mathbf{B}\}$. According to the data generation process, the treks between $\{Q, A\}$ and $\{Y, B\}$ must go through unmeasured confounder $\mathbf{U}$ (except for $X_k$). Therefore, $\mathbf{C}_\mathbf{A} \cup \mathbf{C}_\mathbf{B} = \{X_k, \mathbf{U}\}$. Since $|\{X_k, \mathbf{U}\}| = q+1$, we can conclude that all treks between $\{X_k, \mathbf{A}\}$ and $ \{X_k,Y,\mathbf{B}\}$ must go through a node in $ \{X_k, \mathbf{U}\}$. This will imply that $\mathbf{A} \CI {Y} | (\mathbf{U},X_k)$, i.e., condition 1 of proximal criteria holds, and $\mathbf{A} \CI \mathbf{B} | (\mathbf{U},X_k)$.

Furthermore, because of condition 2), i.e.,  $\mathrm{rk}(\boldsymbol{\Sigma}_{\{X_k, \mathbf{A}\},  \mathbf{B}}) \leq q$, and because according to the "Graphical Representation of Rank Constraints" Theorem, we know that there exist subsets $\mathbf{C}_\mathbf{A},\mathbf{C}_\mathbf{B}$ with $|\mathbf{C}_\mathbf{A}|+|\mathbf{C}_\mathbf{B}| \leq q$ such that $(\mathbf{C_{A}}, \mathbf{C_{B}})$ t-separates $\{X_k, \mathbf{A}\}$ from $ \{\mathbf{B}\}$.
According to the generation of data (Equation 2), all treks between $\{X_k, \mathbf{A}\}$ and $ \{\mathbf{B}\}$ must go through unmeasured confounders $\mathbf{U}$. Hence, $\mathbf{C}_\mathbf{A} \cup \mathbf{C}_\mathbf{B} = \mathbf{U}$. This will imply that $\mathbf{B} \CI {X_k} | \mathbf{U}$. Because $\mathbf{A} \CI \mathbf{B} | (\mathbf{U},X_k)$, we have  $\mathbf{B} \CI (X_k, \mathbf{A}) | \mathbf{U}$, i.e., condition 2 of proximal criteria holds. 

Based on the above analysis, $\mathbf{A}$ and $\mathbf{B}$ are valid NCE and NCO with respect to $X_k \to Y$, respectively. Furthermore, due to Assumption 3, we know that such sets $\mathbf{A}$ and $\mathbf{B}$ must exist in the system.
\end{proof}

\subsection{Proof of Theorem \ref{Theo-NCE-Rank}}

\begin{proof}
Assuming Assumption 2 holds, then according to Lemma \ref{Lemma-rule1}, for a given causal relationship $X_k \to Y$ in the system, the underlying NCE and NCO relative to the causal relationship $X_k \to Y$ can be identified using $\mathcal{R}1$.

Similarly, assuming Assumption 3 holds, then according to Lemma \ref{Lemma-rule2}, for a given causal relationship $X_k \to Y$ in the system, the underlying NCE and NCO relative to the causal relationship $X_k \to Y$ can be identified using $\mathcal{R}2$. 
\end{proof}

\subsection{Proof of Theorem \ref{Theorem-Corr-Rank}}

\begin{proof}
The correctness of Proxy-Rank originates from the following observations:
\begin{itemize}
    \item Firstly, for a given causal relationship $X_k \to Y$ in the system, by Lemma \ref{Lemma-rule1} and Proposition \ref{Pro-Multi-Proxy-Estimator}, valid set of NCE and NCO in $\mathbf{X} \setminus X_k$ have been exactly discovered, and the unbiased causal effect $\mathcal{C}_k$ is obtained if Assumption 2 satisfies (Lines 3$\sim$13 of Algorithm \ref{Alg-Proxy-Rank}).
    \item Secondly, for a given causal relationship $X_k \to Y$ in the system, by Lemma \ref{Lemma-rule2} and Proposition \ref{Pro-Multi-Proxy-Estimator}, valid set of NCE and NCO in $\mathbf{X} \setminus X_k$ have been exactly discovered, and the unbiased causal effect $\mathcal{C}_k$ is obtained if Assumption 2 violates but Assumption 3 satisfies (Lines 14$\sim$21 Algorithm \ref{Alg-Proxy-Rank}).
    \item Lastly, value (NA) is obtained, which indicates the lack of knowledge to obtain the unbiased causal effect (Lines 23$\sim$27 Algorithm \ref{Alg-Proxy-Rank}). 
\end{itemize}
\end{proof}

\subsection{Proof of Lemma \ref{Lemma-rule3}}
\begin{proof}
Firstly, according to Equation 2 (where all entries of matrix $\mathbf{C}$ are non-zero) and the faithfulness assumption, we can directly infer that $\boldsymbol{\Sigma}_{\{X_k, \mathbf{A}\}, \{Y, \mathbf{B}\}}$ and $\boldsymbol{\Sigma}_{\{X_k, \mathbf{A}\}, \{X_k, \mathbf{B}\}}$ are both full rank. Therefore, these necessary conditions always hold for any sets $\mathbf{Z}$ and $\mathbf{W}$ (otherwise, it would be required to test these conditions).

Secondly, due to condition (1), namely $(\{X_k, \mathbf{A}\}, \{X_k, Y, \mathbf{B}\})$ follows the GIN constraint, the faithfulness assumption, Assumption 4 (Non-Gaussianity), and based on the "Graphical Representation of GIN Constraints" Theorem, we can assert that there exist $\mathcal{S}$ with $0\leq |\mathcal{S}| \leq \textrm{min}(| \{X_k, Y, \mathbf{B}\}|-1, |\{X_k, \mathbf{A}\}|)=q+1$ such that 1) the order pair $(\emptyset, \mathcal{S})$ t-separates $\mathcal{Z}$ and $\mathcal{Y}$, and that 2) the covariance matrix of $\mathcal{S}$ and $\mathcal{Z}$ has rank $s$, and so does that of $\mathcal{S}$ and $\mathcal{Y}$. 
According to the data generation process, the treks between $\{\mathbf{A}\}$ and $\{Y, \mathbf{B}\}$ must go through unmeasured confounder $\mathbf{U}$ (except for $X_k$). Therefore, $\mathcal{S} = \{X_k, \mathbf{U}\}$. Since $|\{X_k, \mathbf{U}\}| = q+1$, we can conclude that all treks between $\{X_k, Q, \mathbf{A}\}$ and $ \{X_k,Y,\mathbf{B}\}$ must go through a node in $ \{X_k, \mathbf{U}\}$. This will imply that $\mathbf{A} \CI {Y} | (\mathbf{U},X_k)$, i.e., condition 1 of proximal criteria holds.

Furthermore, because of condition 2), i.e.,  $(\mathbf{B}, \{X_k, \mathbf{A}\})$ follows the GIN constraint, because of Assumption 4, and because according to the "Graphical Representation of Rank Constraints" Theorem, we know that there exist $\mathcal{S}$ with $0\leq |\mathcal{S}| \leq \textrm{min}(|\{X_k, \mathbf{A}\}|-1, |\mathbf{B}|)=q$ such that the order pair $(\emptyset, \mathcal{S})$ t-separates $\mathbf{B}$ and $\{X_k, \mathbf{A}\}$.
According to the generation of data (Equation 2), all treks between $ \mathbf{B}$ and $ \{X_k,\mathbf{A}\}$ must go through unmeasured confounders $\mathbf{U}$. Hence, $\mathcal{S} = \mathbf{U}$. This will imply that $|\mathbf{U}| = q$. Thus, we have $\mathbf{B} \CI (X_k, \mathbf{A}) | \mathbf{U}$, i.e., condition 2 of proximal criteria holds. 

Based on the above analysis, $\mathbf{A}$ and $\mathbf{B}$ are valid NCE and NCO with respect to $X_k \to Y$, respectively. Furthermore, due to Assumptions 1 and 4, we know that such sets $\mathbf{A}$ and $\mathbf{B}$ must exist in the system.
\end{proof}

\subsection{Proof of Theorem \ref{Theo-NCE-GIN}}
\begin{proof}

Assuming Assumptions 1 and 4 hold, then according to Lemma \ref{Lemma-rule3}, for a given causal relationship $X_k \to Y$ in the system, the underlying NCE and NCO relative to the causal relationship $X_k \to Y$ can be identified using $\mathcal{R}3$.
\end{proof}

\subsection{Proof of Theorem \ref{Theorem-Corre-GIN}}

\begin{proof}

The correctness of Proxy-Rank originates from the following observations:
\begin{itemize}
    \item Firstly, for a given causal relationship $X_k \to Y$, by Lemma \ref{Lemma-rule3} and Proposition \ref{Pro-Multi-Proxy-Estimator}, valid set of NCE and NCO in $\mathbf{X} \setminus X_k$ have been exactly discovered, and the unbiased causal effect $\mathcal{C}_k$ is obtained if Assumption 1 satisfies (Lines 2$\sim$10 of Algorithm \ref{Alg-Proxy-GIN}).
    \item Then, value (NA) is obtained, which the lack of valid NCE and NCO for this causal relationship $X_k \to Y$ to obtain the unbiased causal effect (Lines 15$\sim$19 Algorithm \ref{Alg-Proxy-GIN}). 
\end{itemize}

\end{proof}

\end{document}